%% file: main.tex
\documentclass[table]{article}

\usepackage[preprint,nonatbib]{neurips_2024}
\usepackage[round,sort&compress]{natbib}
\PassOptionsToPackage{hyphens}{url}

\usepackage[utf8]{inputenc}
\usepackage[T1]{fontenc}

\usepackage{tikz}
\usepackage{pgfplots}
\pgfplotsset{compat=1.18}
\usepgfplotslibrary{groupplots}
\usetikzlibrary{patterns, arrows.meta, positioning, calc, decorations.pathreplacing, shapes.geometric, fit, backgrounds, shadows}
\usepackage{xcolor}

\newlength{\colstep}
\newlength{\rowgap}
\pgfplotsset{
  perfaxis/.style={
    width=4.3cm, height=4.2cm,
    xmin=-0.05, xmax=1.15, xticklabels={},
    ylabel style={font=\small, xshift=3pt},
    tick label style={font=\scriptsize},
    axis lines=left,
    axis line style={gray!60, line width=0.6pt},
    xtick align=outside, ytick align=outside,
    grid=major, grid style={solid, gray!15, very thin},
    ymin=-25, ymax=85,
    ytick={-20,0,20,40,60,80},
  },
  costaxis/.style={
    width=4.3cm, height=3.6cm,
    xmin=-0.05, xmax=1.15,
    xtick={0,0.5,1.0},
    ylabel style={font=\small, xshift=3pt},
    tick label style={font=\scriptsize},
    axis lines=left,
    axis line style={gray!60, line width=0.6pt},
    xtick align=outside, ytick align=outside,
    grid=major, grid style={solid, gray!15, very thin},
    ymin=-55, ymax=90,
    ytick={-50,0,50},
  },
}
\usepackage{graphicx}
\usepackage{float}
\usepackage{caption}
\usepackage{subcaption}

\usepackage{booktabs}
\usepackage{multirow}
\usepackage{makecell}
\usepackage{colortbl}
\usepackage{adjustbox}
\arrayrulecolor{black}
\usepackage{amsmath,amssymb,amsfonts}
\usepackage{mathtools}
\usepackage{amsthm}
\newtheorem{theorem}{Theorem}[section]

\newtheorem{corollary}[theorem]{Corollary}
\newtheorem{proposition}[theorem]{Proposition}

\theoremstyle{definition}

\theoremstyle{remark}

\usepackage[normalem]{ulem}
\usepackage{enumitem}
\usepackage{hyperref}
\hypersetup{hidelinks}
\usepackage{url}
\usepackage{nicefrac}
\usepackage{microtype}
\usepackage{pifont}
\usepackage{wrapfig}
\usepackage{listings}
\usepackage[most]{tcolorbox}
\usepackage{placeins}
\usepackage{ragged2e}
\usepackage{etoolbox}

\AtEndEnvironment{figure}{\vspace{-0.3cm}}
\AtEndEnvironment{figure*}{\vspace{-0.3cm}}
\AtEndEnvironment{table}{\vspace{-0.3cm}}
\AtEndEnvironment{table*}{\vspace{-0.3cm}}

\definecolor{sSimple}{HTML}{4C72B0}
\definecolor{sFull}  {HTML}{55A868}
\definecolor{sWo}    {HTML}{C44E52}
\definecolor{tCD}    {HTML}{5E9E89}
\definecolor{tCE}    {HTML}{67A9CF}
\definecolor{tT2S}   {HTML}{E69580}

\lstdefinestyle{promptstyle}{
  basicstyle=\ttfamily\scriptsize,
  breaklines=true,
  breakatwhitespace=false,
  frame=single,
  rulecolor=\color{gray!50},
  backgroundcolor=\color{gray!6},
  xleftmargin=6pt,
  xrightmargin=6pt,
  framesep=4pt,
  columns=fullflexible,
  keepspaces=true,
  showstringspaces=false,
}
\definecolor{promptblue}{HTML}{2E5AA8}
\tcbset{
  agentpromptbox/.style={
    colback=promptblue!4,
    colframe=promptblue!55,
    boxrule=0.5pt,
    arc=1.5mm,
    outer arc=1.5mm,
    left=2mm, right=2mm, top=1mm, bottom=1mm,
    fonttitle=\bfseries\small,
    before skip=5pt, after skip=5pt,
  }
}
\lstdefinestyle{agentprompt}{
  basicstyle=\ttfamily\fontsize{7.3pt}{8.1pt}\selectfont,
  frame=none,
  breaklines=true,
  breakatwhitespace=true,
  breakindent=0pt,
  columns=fullflexible,
  keepspaces=true,
  numbers=none,
  xleftmargin=2pt,
  showstringspaces=false,
}

\definecolor{cblue}{HTML}{4E79A7}
\definecolor{corange}{HTML}{F28E2B}
\definecolor{cgreen}{HTML}{59A14F}
\definecolor{cred}{HTML}{E15759}
\definecolor{cpurple}{HTML}{7B68EE}
\definecolor{cgold}{HTML}{FFC107}
\definecolor{cgray}{HTML}{A5A5A5}
\definecolor{lightblue}{HTML}{D6E4F0}
\definecolor{lightorange}{HTML}{FBE5D6}
\definecolor{lightgreen}{HTML}{E2EFDA}
\definecolor{lightred}{HTML}{FCE4EC}
\definecolor{sectionblue}{HTML}{1F4E79}
\definecolor{cteal}{HTML}{009688}
\definecolor{cbrown}{HTML}{D4813A}
\definecolor{cEX}{HTML}{E69580}
\definecolor{cNSF}{HTML}{CE6D5E}
\definecolor{cSRR}{HTML}{8F5C4E}
\definecolor{cMRE}{HTML}{67A9CF}
\definecolor{cMSA}{HTML}{1C6488}
\definecolor{cF1}{HTML}{5E9E89}
\definecolor{cAUC}{HTML}{9281B7}
\definecolor{cdarkgray}{HTML}{404040}

\title{From Multi-Agent to Single-Agent: When Is \\ Skill Distillation Beneficial?}

\author{
  Binyan Xu$^{1,2}$\thanks{This work was done during an internship at Tencent.},
  Dong Fang$^{2}$\thanks{Corresponding authors.},
  Haitao Li$^{2}$\footnotemark[2],
  and Kehuan Zhang$^{1}$ \\[4pt]
  $^{1}$The Chinese University of Hong Kong, Hong Kong, China \\
  $^{2}$LIGHTSPEED, Shenzhen, China \\[2pt]
  \texttt{\{binyxu,khzhang\}@ie.cuhk.edu.hk, df572@outlook.com, 729156675@qq.com}
}

\graphicspath{{./}{./figures/}}

\begin{document}
\maketitle

\input{sections/00_abstract}
\input{sections/01_introduction}
\input{sections/02_related_work}
\input{sections/03_method_revised}
\input{sections/04_experiments}
\input{sections/06_limitations}
\input{sections/07_conclusion}

\bibliographystyle{plainnat}
\bibliography{references}

\clearpage
\raggedbottom
\input{sections/appendix}

\end{document}

%% file: sections/00_abstract.tex
\begin{abstract}
Multi-agent systems (MAS) for structured data-science tasks externalize analytical control through workflows spanning stages, tools, shared state, verification, and repair. Distilling such workflows into a single-agent skill can reduce orchestration overhead, but it remains unclear which workflow components should cross the control boundary. We distinguish capability resources, which expand what an agent can do, from pipeline guidance, which constrains which solutions it explores. On the same causal-estimation instances, adding task-qualified source-pipeline guidance to a capability-matched skill changes normalized utility by $+19.6$ points under method-selection accuracy but $-10.3$ points under numerical error. To explain this reversal, we introduce \emph{Behavior-Outcome Freedom} ($F$), a pre-synthesis diagnostic of signed behavior--outcome rank mismatch, and formalize its candidate-conditional role through Signed Anchor-Rank Transfer. Motivated by this mechanism, we propose \textbf{AdaSkill}, which preserves validated capability resources, removes runtime orchestration, and conditionally inherits pipeline guidance using a calibrated rule over $F$. Across 16 capability-matched interventions, the native-scale Full-minus-Discard effect decreases across the continuous $F$ scale ($r=-0.80$, $p<0.001$), while a 15-treatment atomic sweep localizes the reversal to pipeline guidance. Across 11 datasets spanning four structured data-science task families, AdaSkill combines strong task performance with substantially lower deployment overhead.
\end{abstract}

%% file: sections/01_introduction.tex
\section{Introduction}

Multi-agent systems (MAS) for structured data-science tasks increasingly rely on explicit workflows to organize analytical control across roles, stages, shared state, tool calls, verification, and repair~\citep{cao2026apex,verma2025cais,shen2024matmcd,ouyang2025fela}. Such workflows can achieve strong task performance, but they externalize control into a runtime orchestration layer that must be repeatedly executed at deployment time. Skill distillation offers a natural alternative: transfer reusable workflow competence into a single tool-using agent, thereby reducing multi-agent coordination overhead while preserving the analytical competence encoded by the source system.

Existing approaches pursue two complementary forms of agent distillation. MAS-distillation methods compress interaction graphs, debate dynamics, or execution trajectories into a single student model~\citep{chen2024magdi,aluru2025smagdi,luo2026agentark}, while inference-time methods synthesize, verify, attribute, or refine reusable skills from source workflows and execution experience~\citep{qiu2025agentdistill,alzubi2026evoskill,ma2026skillgen,yuan2026clawtrace}. These methods establish how transferred agent logic can be internalized, constructed, evaluated, or optimized. What remains unresolved is whether functionally different source components should cross the control boundary under the same inheritance rule.

\input{figures/fig_fig_solution-landscape-realistic.tex}
However, a MAS workflow does not transfer a functionally homogeneous object. Callable tools and declarative knowledge are \emph{capability resources} that expand what a target agent can access, whereas ordering rules, decision gates, and other \emph{pipeline guidance} impose source-derived control over which reachable solutions the agent explores. As illustrated in Figure~\ref{fig:solution-landscape}, when a task-qualified pipeline contracts behavior toward successful anchors, positively ordered behavior--outcome geometry allows that contraction to carry outcome-rank information; as the ordering weakens or reverses, the same source control loses reliable support and can suppress valuable alternatives. Thus, even a published, high-performing workflow need not provide beneficial structural control after distillation. This leads to our central question: \emph{when does transferring workflow structure from a multi-agent system benefit a single-agent skill?}

Our central insight is that the value of structural transfer is governed jointly by behavior--outcome geometry and by the strength and direction of the transferred intervention. We therefore distinguish \emph{capability transfer}, which preserves tools and declarative knowledge, from \emph{structural transfer}, which preserves source-derived pipeline guidance. To characterize the upstream geometry before skill synthesis, we introduce \textbf{Behavior-Outcome Freedom} ($F$), the signed rank mismatch between pairwise behavioral distance and score distance. Low $F$ indicates positively ordered geometry in which behavioral contraction can faithfully carry outcome-rank information; as $F$ approaches one, monotone organization weakens, and above one the geometry becomes anti-concordant. Importantly, $F$ is an effect modifier rather than a stand-alone estimator of candidate lift: a task-qualified source workflow supplies a grounded structural proposal, while the realized strength and direction of that proposal remain candidate-specific. Our Signed Anchor-Rank Transfer analysis formalizes this candidate-conditional relationship.

Motivated by this observation, we propose \textbf{AdaSkill}, an adaptive workflow-to-skill distillation policy that preserves validated capability resources, removes runtime orchestration, and makes pipeline guidance the sole adaptive object. AdaSkill decomposes a task-qualified source workflow into callable tools, declarative knowledge, pipeline guidance, and runtime orchestration; retains validated tools and knowledge; removes orchestration mechanisms that are unnecessary for single-agent execution; and applies a development-calibrated rule, frozen before target evaluation, that uses $F$ to determine whether source-pipeline guidance should be inherited. We study this mechanism through scorer-controlled interventions, atomic Pipeline-only treatments, capability-matched Full-versus-Discard comparisons, and workflow-excluded construction across 11 datasets spanning text-to-SQL, causal estimation, causal discovery, and feature engineering. The resulting skills combine strong controlled performance with substantially lower deployment overhead than their orchestrated source workflows. Our main contributions are summarized as follows:
\begin{itemize}[leftmargin=1.5em, itemsep=2pt]
  \item \textbf{Functional Decomposition of Workflow Transfer:} We distinguish capability transfer from structural transfer and show why tools and declarative knowledge should not be governed by the same inheritance rule as source-pipeline control. This decomposition turns workflow distillation from monolithic compression into a selective transfer problem.
  \item \textbf{Behavior--Outcome Mechanism and Diagnostic:} We introduce Behavior-Outcome Freedom ($F$) as a pre-synthesis diagnostic of signed behavior--outcome geometry and provide a candidate-conditional theoretical account through Signed Anchor-Rank Transfer and the Structural-Lift Decomposition. Scorer reversal, 15 atomic Pipeline-only treatments, and 16 capability-matched Full-versus-Discard interventions show structural value declining with $F$.
  \item \textbf{Adaptive Single-Agent Skill Distillation:} We develop AdaSkill as a constructive policy that preserves validated capability resources, removes runtime orchestration, and conditionally inherits source-pipeline guidance using a calibrated rule frozen before target evaluation. Across 11 datasets spanning program construction, numerical estimation, graph recovery, and predictive modeling, AdaSkill achieves strong performance with lower deployment cost and latency.
\end{itemize}


%% file: figures/fig_fig_solution-landscape-realistic.tex
\begin{wrapfigure}{r}{0.51\textwidth}
  \vspace{0.1cm}
  \centering
  \includegraphics[width=\linewidth]{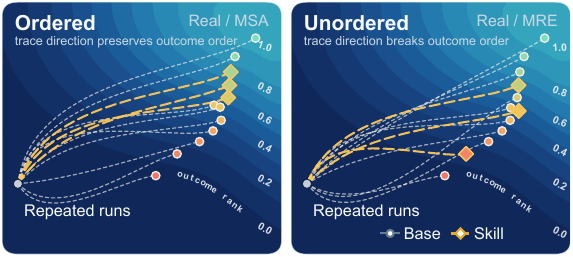}
  \caption{\textbf{Behavior--outcome geometry governs structural lift.}
  Repeated runs share an input. Trajectory direction schematically encodes behavior; endpoints encode outcome rank, inducing the pairwise distance ranks compared by $F$.
  \textbf{Left:} behavior- and outcome-distance ranks align, so skill guidance concentrates search among high-rank behaviors.
  \textbf{Right:} the ordering breaks, so rigid guidance suppresses valid paths, favoring flexibility.}
  \label{fig:solution-landscape}
  \vspace{0.3\baselineskip}
\end{wrapfigure}

%% file: sections/02_related_work.tex
\section{Background and Related Work}

\textbf{Orchestrated Agentic Workflows.} Tool-augmented agents~\citep{schick2024toolformer,yao2023react} increasingly solve complex tasks through explicit orchestration. AutoGen~\citep{wu2023autogen} externalizes control through multi-agent conversation, while MetaGPT~\citep{hong2024metagpt} encodes role-specialized software workflows. In data science, DS-Agent uses case-based control~\citep{guo2024dsagent}; APEX-SQL~\citep{cao2026apex} implements staged planning, schema pruning, profiling, synthesis, and verification for text-to-SQL; and CAIS~\citep{verma2025cais} organizes estimation decisions grounded in causal inference~\citep{imbens2015causal}. MATMCD~\citep{shen2024matmcd} combines agent modules with causal graph search~\citep{spirtes2000causation,zheng2018dags}, while FELA~\citep{ouyang2025fela} uses a role-specialized evolutionary workflow. These systems differ in agent count and autonomy, but share a control topology built from ordered stages, shared artifacts, tool interfaces, and feedback loops. Our object of study is this explicit workflow control. Agent count is an implementation detail of the source system.

\textbf{From External Orchestration to Internalized Control.} Single-agent multi-turn execution and skill libraries can reproduce behaviors commonly assigned to multi-agent systems at lower inference cost~\citep{li2026single,xu2026rethinking}. Moving control into one agent also concentrates context management, an increasingly explicit object of agent design~\citep{xu2026contextual,xu2026llm}. AgentArk~\citep{luo2026agentark}, SMAGDI~\citep{aluru2025smagdi}, Latent Agents~\citep{yi2026latent}, and MAPD~\citep{liu2026mapd} compress interactions or protocols into model parameters or a student policy, extending classical knowledge distillation~\citep{hinton2015distilling}. These approaches primarily ask how faithfully a source policy can be compressed. AdaSkill asks which parts should cross the control boundary. The distinction between capability resources and structural constraints makes transfer selectivity, not compression fidelity, the central variable.

\textbf{Skill Discovery and Evolution.} Recent systems optimize skills from execution feedback. EvoSkill uses iterative failure analysis~\citep{alzubi2026evoskill}, while Voyager develops reusable capabilities through open-ended embodied exploration~\citep{wang2023voyager}. Trace2Skill consolidates agent experiences into declarative skills~\citep{ni2026trace2skill}, and ASDA adapts distilled skills for financial reasoning~\citep{yim2026asda}. AutoSkill, SkillRL, and Memento-Skills evolve skill libraries through experience, recursive learning, or agent-designed control~\citep{yang2026autoskill,xia2026skillrl,zhou2026memento}. SkillGen contrasts successful and failed trajectories to synthesize an auditable skill and verifies its net effect before deployment~\citep{ma2026skillgen}. ClawTrace attributes step-level cost and applies preserve, prune, and repair operations~\citep{yuan2026clawtrace}. SkillOpt treats the skill document as optimizable external state~\citep{yang2026skillopt}, and SKILL-DISCO compiles shared successful-trace structure into executable procedural subgraphs~\citep{guo2026skilldisco}. These methods optimize after candidate skills or trajectories exist. We identify the transfer distinction that precedes optimization: capability components make resources available, while source-pipeline clauses explicitly direct behavior among reachable solutions. $F$ measures whether the baseline behavior--outcome geometry supports that structural control. In multi-agent RL, \citet{zhou2025double} compress agent interactions into a shared policy network, and \citet{wang2026skill} use skill-conditioned self-distillation for multi-turn agents.

\textbf{Adjacent Transfer, Auditing, and Evaluation.} Transfer is also studied under adversarial objectives, including universal cross-task attacks~\citep{xu2025one} and optimized generative triggers~\citep{xu2026breaking}. Security systems use internal or external evidence to audit model behavior and poisoned data~\citep{xu2026internal,xu2025clip}. Agent deployment adds a trace-economic layer to the value of automation~\citep{xu2026agent}. These literatures sharpen the boundary here: AdaSkill studies constructive workflow control and evaluates it through matched task outcomes. Raw review scores are not commensurate across research areas~\citep{xu2026reviewer}; the evidence in this paper is therefore tied to explicit interventions and objective metrics.


%% file: sections/03_method_revised.tex
\section{Methodology}
\label{sec:methodology}

\subsection{Workflow Distillation Contains Two Transfer Functions}
\label{sec:methodology-overview}
Workflow-to-skill distillation relocates control from an explicit runtime system to an inference-time skill. The source workflow encodes stages, role-conditioned calls, shared state, validation, and feedback loops. The target agent internalizes selected parts of this policy. We distinguish components by function. \emph{Capability components}, including callable tools and declarative domain knowledge, make actions and information available to the agent. \emph{Structural components}, including pipeline ordering, decision gates, and task-decomposition instructions, explicitly prioritize or prohibit solution paths. Capability text can also change behavior probabilities in a contextual model; our distinction concerns the control function being transferred. Runtime orchestration coordinates modules or agents in the source system and disappears after conversion. The target of analysis is the placement and effect of workflow control, not the number of agents.

We call a source workflow \emph{task-qualified} when it is a published, high-performing system designed and evaluated for the same analytical objective, its pipeline clauses map to decisions scored by the target metric, and its tools pass target-interface checks. Objective identity comes from the benchmark specification, clause--metric mappings come from the source paper and code, and interface validity is executable. Qualification defines a task-grounded candidate class; transfer sign remains the intervention estimand.

Structural value has three factors: outcome relief across reachable behaviors, the strength of the induced shift, and alignment with high-value regions. A task-qualified source workflow supplies a structural proposal; behavior-outcome geometry determines whether contraction can carry ordered outcome information. If behaviorally distant runs also differ reliably in score, pipeline choice is consequential. Weak coupling leaves little ordered outcome structure to exploit, and anti-concordance reverses the implication of contraction. We use \emph{search space} for the observable support of outputs and response-free traces induced by the agent and harness. \textbf{Behavior-Outcome Freedom} ($F$) probes this signed rank geometry before skill synthesis. AdaSkill turns the mechanism into an executable design: qualification proposes the candidate, $F$ orders geometric support for its transfer, and development interventions calibrate one operating boundary. Figure~\ref{fig:overview} summarizes the probe, converter, and deployment path.
\input{figures/fig_fig_overview.tex}

\subsection{Behavior-Outcome Freedom}
\label{sec:metric-freedom}
Freedom $F$ measures the signed rank fidelity between behavioral distance and outcome distance. In causal estimation, Method Selection Accuracy (MSA) separates runs by the selected causal method, while Mean Relative Error (MRE) rewards numerical proximity to the ground truth. The operational output representation uses estimator class. A richer audit represents CE behavior hierarchically: different estimator classes have distance one, while executions within a class are separated by configuration distance over covariates, bandwidths, kernels, estimands, estimator calls, and tool paths. It preserves the decision scored by MSA and resolves implementation choices that matter to MRE. $F$ captures this distinction on the support explored by a fixed evaluation protocol.

\textbf{Population estimand.} Fix an evaluation unit $q$ (a question, or a dataset-level unit when runs are compared globally), and let $Z,Z'\overset{\mathrm{iid}}{\sim}\mathbb{P}_{0,q}=\mathbb{P}_{A,H\mid q}$ be two runs induced by a fixed base agent $A$ and harness/prompting protocol $H$. Let $\phi(Z)\in\mathcal{X}$ be a behavioral representation, let $d$ be a symmetric behavioral dissimilarity, and let $s(Z)\in[0,1]$ be the metric score. For
$U=d(\phi(Z),\phi(Z'))$ and $V=|s(Z)-s(Z')|$, define
\begin{equation}
  F_q(s;d,\phi,\mathbb{P}_{0,q}) \;\triangleq\; 1-\rho_S(U,V),
  \label{eq:population-freedom}
\end{equation}
provided both rank variables have nonzero variance. With ties, we define population Spearman correlation as
$\rho_S(U,V)=\operatorname{Corr}(R_U(U),R_V(V))$, where
$R_U(u)=\Pr(U<u)+\tfrac12\Pr(U=u)$ and analogously for $R_V$; this is the population counterpart of average ranks. For per-question tasks, the canonical dataset--metric statistic is the arithmetic mean of $F_q$ over identifiable questions $q\sim Q$; for dataset-level tasks, $q$ denotes the pooled dataset-level unit. RQ2 retains these dataset--metric coordinates to test structural effect modification. The converter operates one level above them: each metric module receives a pre-specified aggregate over its diagnostic datasets, or the designated conversion benchmark when the source workflow defines one. Thus the 16 intervention coordinates and the seven routing inputs share the same pre-synthesis estimates but are not the same inferential unit.

\textbf{Sample estimator.} Given $n$ baseline runs for a fixed $q$, each run $i$ yields $(x^{\text{out}}_i,x^{\text{trace}}_i,s_i)$. We define $D^{\text{score}}_{ij}=|s_i-s_j|$, $D^{\text{out}}_{ij}=d(x^{\text{out}}_i,x^{\text{out}}_j)$, and $D^{\text{trace}}_{ij}=d(x^{\text{trace}}_i,x^{\text{trace}}_j)$. Let $R^{\mathcal{X}}_{ij}$ and $R^{\text{score}}_{ij}$ be average ranks of the upper-triangular elements, with $\mathcal{X}$ denoting either output or trace space. The sample Mantel-style Spearman coefficient~\citep{mantel1967detection} is
\begin{equation}
  r_M(\mathcal{X}) =
  \frac{\displaystyle\sum_{i<j}(R^{\mathcal{X}}_{ij}-\bar R^{\mathcal{X}})(R^{\text{score}}_{ij}-\bar R^{\text{score}})}
  {\displaystyle\sqrt{\sum_{i<j}(R^{\mathcal{X}}_{ij}-\bar R^{\mathcal{X}})^2\sum_{i<j}(R^{\text{score}}_{ij}-\bar R^{\text{score}})^2}},
  \label{eq:mantel-spearman}
\end{equation}
and the sample estimate is
\begin{equation}
  F_{\mathcal{X}}=1-r_M(\mathcal{X}).
  \label{eq:metric-freedom}
\end{equation}

The rank form supplies a common, scale-invariant statistic for categorical, set-valued, and continuous behaviors with ties. $F$ measures the orientation and fidelity of the behavioral geometry on which source-pipeline transfer acts. It orders geometric support for transfer within the task-qualified candidate class; AdaSkill converts that ordering into an operational policy.

When $r_M\geq0$, $F_{\mathcal{X}}\approx0$ indicates tight behavior-outcome rank coupling: behaviorally distant runs also rank far apart in score. Values near one indicate weak signed monotone ordering between behavior and score distances on the sampled support. Since $F\in[0,2]$, values above one correspond to \emph{anti-concordance} ($r_M<0$). Such values can remain predictive, but they reverse the outcome implication of behavioral contraction and therefore do not support contraction toward successful behaviors. We retain and report the sign of $r_M$ when interpreting this regime. If either distance vector is constant, $F$ is undefined. We omit such questions from the protocol-level aggregate. For per-question routing, the operational rule requires at least three identifiable questions after the fixed sampling budget. Estimates below this minimum remain in the complete descriptive audit but cannot independently justify retention; every system-level association is also reported under the strict $K\geq3$ subset. The routing order is fixed: $F_{\mathrm{out}}$ is the primary statistic because its task-structural representation is specified with the evaluator, and $F_{\mathrm{trace}}$ is used only when $F_{\mathrm{out}}$ is unidentifiable. When both exist, $F_{\mathrm{trace}}$ provides a complementary process-space analysis and never overrides the output-space route. If neither representation reaches the identifiability minimum after additional sampling, AdaSkill takes the default Discard route, emitting the validated capability layer without mandatory source-pipeline clauses. The pairwise entries are dependent because they share runs; uncertainty is therefore resampled at the run or question level, never by treating the $\binom n2$ pairs as independent observations.

\textbf{Instantiations and empirical reproducibility.} We compute $F$ using $M$ diagnostic questions sampled from the diagnostic pool and $N$ Base Agent runs per question. To mitigate output collapse, each run is seeded with a methodologically distinct prior generated by a lightweight \emph{diversity planner}. The planner deliberately probes the outcome geometry across plausible strategies; it does not estimate their natural deployment frequency. The primary estimand is therefore the exploratory support over which a pipeline can suppress valid alternatives. \textbf{Independent-run control.} Removing the planner and sampling ordinary Base Agent executions recovers the same CE scorer ordering as the budget grows: at $M{=}N{=}20$, planner and no-planner trace estimates are $0.55$ versus $0.53$ for MSA and $0.85$ versus $0.85$ for MRE. The planner reaches this geometry at the operating budget; it does not create the reversal (Appendix~\ref{app:mf-sensitivity}). In finite samples, we report \emph{mixed questions} with observed score variation so that the rank statistic is defined. The population counterpart conditions on questions whose score-difference variable $V$ is non-degenerate; with only $N$ runs per question, rare variation may go unobserved, so the screening rule is itself part of the reporting protocol:
\begin{itemize}[leftmargin=1.5em, itemsep=1pt]
  \item \textbf{Output freedom} ($F_{\text{out}}$): $\mathcal{X}$ is the space of task-authored structured behavior, and $d$ is a task-appropriate dissimilarity (Table~\ref{tab:output-dist}). The CE representation audit enriches estimator identity with a deterministic execution signature parsed from response-free tool calls; no outcome value enters the signature.
  \item \textbf{Trace freedom} ($F_{\text{trace}}$): $\mathcal{X}$ is the embedding space of process traces containing intermediate assistant messages, tool calls, and tool results. The final response is removed before embedding with \texttt{text-embedding-3-large}~\citep{openai2024embeddings}; distance is scaled cosine dissimilarity $(1-\cos_{\mathrm{sim}})/2$.
\end{itemize}
Appendix~\ref{app:metric-freedom} reports budget, representation, and sampling sensitivity. For a bounded higher-is-better score $u$, we define headroom-normalized lift as $(u_{\mathrm{skill}}-u_{\mathrm{base}})/(1-u_{\mathrm{base}})$. For lower-is-better MRE $e$, we use proportional error reduction $(e_{\mathrm{base}}-e_{\mathrm{skill}})/e_{\mathrm{base}}$. RQ2 applies these definitions to atomic and matched structural interventions.

\subsection{Formal Account of Structural Transfer}
\label{sec:theory}
For any behavior-mediated structural intervention, Theorem~\ref{thm:structural-lift} gives the exact identity
\begin{equation}
  \mathrm{Lift}(\pi)=a_\pi S_\pi G,
  \qquad |\mathrm{Lift}(\pi)|\leq S_\pi G,
  \label{eq:structural-lift-main}
\end{equation}
where $G$ is outcome relief across baseline-reachable behaviors, $S_\pi$ is intervention strength, and $a_\pi$ is directional alignment. This factorization separates the landscape measured before synthesis from the strength and direction realized by a candidate intervention.

$F$ is a signed projection of distance-rank geometry; it is not a raw-scale estimator of $G$. If $\Gamma_{\mathrm{rank}}$ is the fraction of outcome-distance rank variation explained by behavioral distance, Proposition~\ref{prop:geometric-determinacy} gives
\begin{equation}
  \rho_S(U,V)^2\leq\Gamma_{\mathrm{rank}},
  \qquad 1-\Gamma_{\mathrm{rank}}\leq2F-F^2.
  \label{eq:geometric-bound-main}
\end{equation}
Near-zero $F$ therefore identifies positively oriented rank structure. The intervention-level statement uses successful anchors. Let $\Delta_\pi$ denote standardized behavior-rank contraction toward a high-score Base event of mass $h$, and let $\Lambda_\pi$ denote the corresponding score-distance-rank contraction. Theorem~\ref{thm:anchor-transfer} proves, for orientation $\kappa\in\{-1,+1\}$,
\begin{equation}
  |\Lambda_\pi-\kappa\Delta_\pi|
  \leq
  \sqrt{\frac{2\chi_\pi^2\{1-\kappa(1-F)\}}{h}}.
  \label{eq:anchor-transfer-main}
\end{equation}
For positive orientation ($\kappa=+1$), the error is $\sqrt{2\chi_\pi^2F/h}$; near anti-concordance ($\kappa=-1$), it is $\sqrt{2\chi_\pi^2(2-F)/h}$ and the leading direction reverses. A calibration corollary converts $\Lambda_\pi$ to native-scale lift and exposes its residual. The account separates three objects: $F$ measures the landscape ordering before synthesis, $\Delta_\pi$ measures whether a concrete source candidate contracts toward successful behavior, and native lift measures whether that contraction transfers to the target score. Baseline geometry alone cannot identify the direction of an unexecuted intervention: candidates can share the same $F$ yet induce opposite $a_\pi$. Thus $F$ is a structural effect modifier, not a stand-alone estimator of candidate lift. The theorem states candidate-conditional transfer; the matched intervention estimates the realized structural effect, and the cross-fitted converter evaluates a calibrated policy. Appendix~\ref{app:theory} proves the full statements and impossibility boundaries.

\subsection{Adaptive Workflow-to-Skill Converter}
\label{sec:mas2skill}
AdaSkill compiles a task-qualified orchestrated workflow into a skill for a tool-using agent. Its design is direct: preserve validated capability resources, remove runtime orchestration, and make pipeline guidance the sole adaptive object.

\textbf{Phases 1 and 2: Analysis and Inventory.} The distiller applies a clause-level functional test to decompose the source architecture. A callable implementation is a \emph{tool}. A declarative proposition about the domain, data, estimator assumptions, or tool semantics is \emph{knowledge} only when it contains no action directive. Any state-to-action rule is \emph{pipeline guidance}, operationally assigned to the Pipeline class: this includes recommendations and prohibitions, ordered steps, decision gates, tool-selection rules, retries, and stopping conditions. A mixed clause is classified as Pipeline, never knowledge. Runtime handoffs, shared state, voting, debate, and inter-agent control are \emph{orchestration mechanisms}. This is an operational partition, not a latent human label: the action-directive test defines the treatment, no manual relabeling occurs, and Appendix~\ref{app:operational-protocol} publishes the rule and validator for audit. The inventory retains only tools that map to the target task and pass import and interface checks. It deduplicates declarative knowledge and rejects clauses that are irrelevant to the target specification, inconsistent with the tool interface, or unsupported by the source artifact. These validated resources form the capability layer; pipeline guidance and orchestration constrain execution in different forms.

\textbf{Phase 3: Calibrate Once, Then Freeze.} Runtime orchestration is removed, and the same validated capability resources remain available to both structural candidates. Domain knowledge is rendered as descriptive reference material; any clause that changes an action, priority, or order remains pipeline guidance. AdaSkill converts the continuous $F$ ordering into a binary deployment decision: retain below a calibrated threshold and discard above it. Weak coupling near one and anti-concordance above one both reject the source contraction. The threshold is estimated from development workflow regimes and is then frozen before the target regime is observed. RQ3 implements this contract by cross-fitting: in each family-excluded fold, the other regimes alone select $t$, and that single value is applied unchanged to every metric of the excluded family. Different folds yield $0.426$--$0.692$ because their calibration sets differ; this is cross-validation of one calibration algorithm, not family-specific tuning. The pooled boundary $0.65$ is used only as a rounded descriptive reference in diagrams and sensitivity audits. Unidentifiable metrics take the default Discard route. For a single-metric family, the route directly selects one structural configuration. For a multi-metric family, the converter integrates the metric-specific route vector into one task-level skill, retaining clauses supported by one objective while removing constraints that impair another. This hybrid is jointly synthesized, not a metric-wise concatenation of Full and Discard outputs.

\textbf{Phase 4: Validation.} A clause-level audit reapplies the functional test to the emitted skill. For a discard module, it rejects every source-derived state-to-action clause, including imperatives hidden under a ``knowledge'' heading; for a retain module, it records such clauses as pipeline. The validator also checks tool imports, descriptive reference tone, and the fixed threshold boundary (Figure~\ref{fig:overview}, Panel C). Appendix~\ref{app:operational-protocol} gives the converter prompt, input and output contract, multi-metric composition, and validation checks.

%% file: figures/fig_fig_overview.tex
\begin{figure}[t]
\centering
\includegraphics[width=0.94\textwidth,height=0.34\textheight,keepaspectratio]{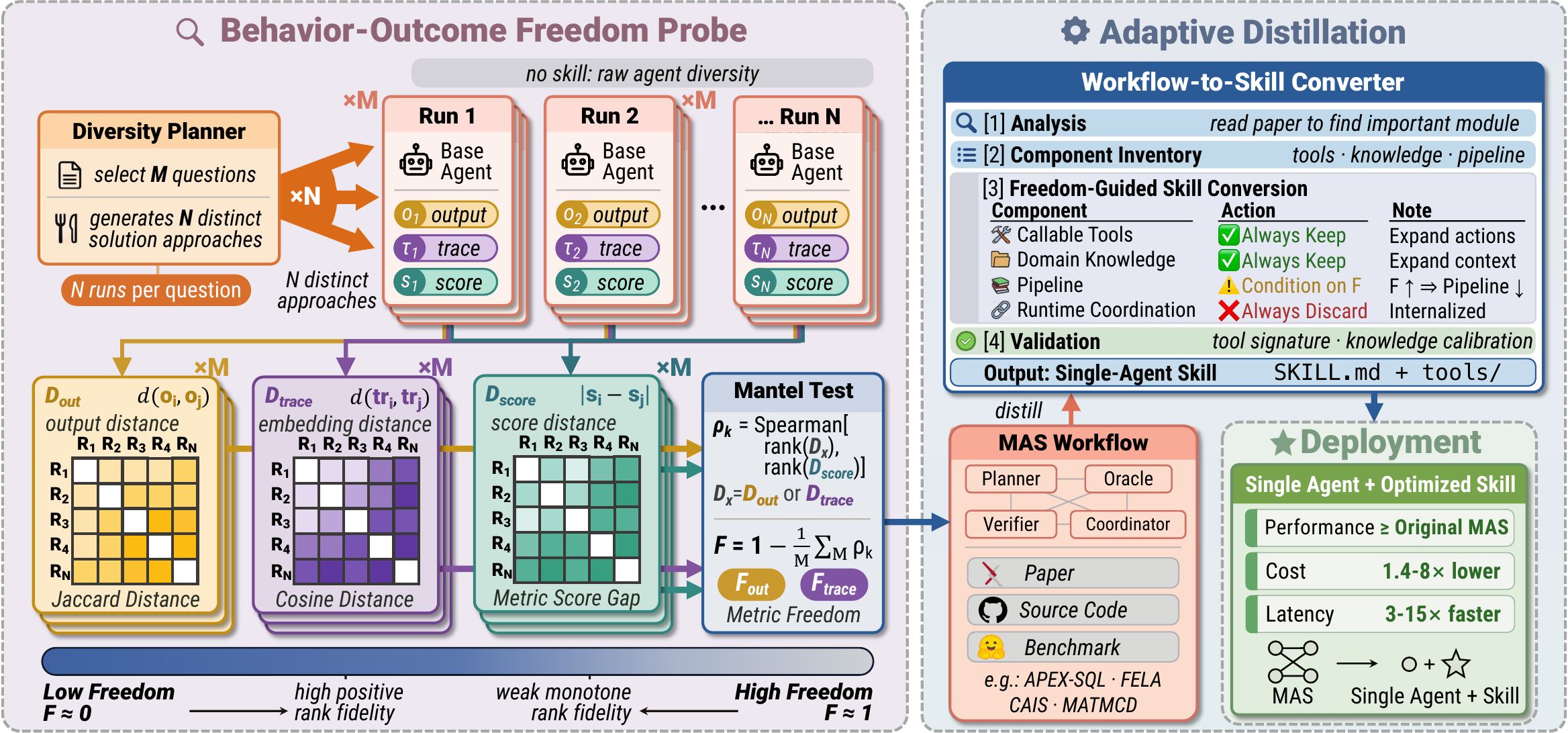}

\caption{\textbf{AdaSkill operationalizes signed structural transfer.} A task-qualified source proposes a structural candidate and $F$ measures pre-synthesis behavior--outcome rank geometry. The converter preserves eligible capability resources, conditionally inherits source-pipeline control, removes runtime orchestration, and emits a skill for direct deployment by a tool-using agent.}
\label{fig:overview}

\end{figure}

%% file: sections/04_experiments.tex
\section{Experiments}
\label{sec:results}

\subsection{Setup}

\input{figures/fig_tab_source-workflows.tex}

\textbf{Tasks and systems.} We evaluate 11 datasets spanning four structurally different task families: Text-to-SQL (BIRD-147~\citep{li2024bird}, Spider-120~\citep{lei2025spider2}), Causal Estimation (three CauSciBench datasets~\citep{acharya2026causcibench}), Causal Discovery (AutoMPG, DWDClimate, Sachs, and Child~\citep{quinlan1993autompg,mooij2016distinguishing,sachs2005causal,spiegelhalter1992learning}), and Feature Engineering (Taobao and Dia~\citep{taobao2018ad,teboul2022diabetes}). Table~\ref{tab:source-workflows} summarizes their task-qualified source workflows. We compare the native \textbf{Source Workflow}, a tool-using \textbf{Base Agent}, \textbf{AdaSkill}, \textbf{EvoSkill}~\citep{alzubi2026evoskill}, and a pipeline-faithful \textbf{Workflow Compiler}~\citep{li2026single}. All skill-agent conditions share the Sonnet~4.6 backend, splits, tools, and evaluator; source workflows retain their native orchestration and serve as the deployment reference. Appendix~\ref{app:experimental-design} gives source qualification, baseline coverage, and full implementation details.

\textbf{Metrics.} CE uses method-selection accuracy (MSA) and mean relative error (MRE); T2SQL uses execution accuracy (EX) and schema metrics; CD uses edge F1; FE uses test AUC.

\textbf{Evaluation protocol.} Diagnosis, conversion, and final evaluation use disjoint roles fixed before synthesis. $F$ uses Base Agent runs and gold scores only for outcome-distance ranks; the converter receives metric-level $F$, not diagnostic or test instances. Threshold calibration uses development regimes and is frozen before target evaluation (Appendix~\ref{app:experimental-design}).

\textbf{Evidence chain.} RQ1 changes only the scorer to test whether outcome geometry can reverse pipeline value. RQ2 is the primary mechanism test: across 16 capability-matched interventions, Full and Discard retain identical validated capability resources and remove identical runtime orchestration, differing only in pipeline guidance. RQ3 cross-fits and freezes one threshold rule before executing the composed skills in an excluded workflow regime. Appendix~\ref{app:pipeline-treatment-audit} reports the complete inferential specification, strict-identifiability analysis, and robustness checks.

\subsection{Results Overview}
\label{sec:end-to-end-results}

\textbf{Task-level performance and efficiency.}
Table~\ref{tab:main-performance-efficiency} presents the complete deployed outcomes before the mechanism analysis. Base Agent, Workflow Compiler, EvoSkill, and AdaSkill use the same Sonnet~4.6 harness. AdaSkill combines the strongest mean controlled performance with lower deployment cost and latency across the tested domains. Against the native source-workflow deployment reference, AdaSkill wins five of six performance aggregates and is 2.4--17.7$\times$ faster. Its mean T2SQL EX is 67.1 versus 68.5 for APEX-SQL, at 38\% lower measured cost and 3.6$\times$ lower latency. Pareto plots, accounting assumptions, and the complete 11-dataset breakdown are in Appendices~\ref{app:pareto} and~\ref{app:full-bar}.

\input{figures/fig_fig_overview-perf.tex}
\FloatBarrier

\input{figures/fig_fig_metric-freedom-scatter.tex}

\textbf{Overall skill utility follows the pre-synthesis geometry.}
Across all 16 dataset--metric tuples, deployed AdaSkill lift falls with canonical output-space $F$ ($r=-0.83$, $p=8.3\times10^{-5}$) and response-free trace-space $F$ ($r=-0.79$, $p<0.001$). Figure~\ref{fig:metric-freedom-scatter} reports every dataset and metric as the system-level pattern to be explained. It is not the primary mechanism test because AdaSkill itself consumes $F$. RQ1 controls the scorer, RQ2 independently isolates pipeline structure as the component carrying the reversal, and RQ3 tests an executable rule across excluded workflow regimes.

\subsection{RQ1: Why Can the Same Pipeline Help and Harm?}
\label{sec:metric-freedom-results}

\textbf{The same behavior has different value under different scorers.}
Figure~\ref{fig:mantel-opposite-lifts} computes $F$ from the same 60 pairs of Base Agent runs in causal estimation and changes only the scorer. Each point relates behavior distance to score difference. The canonical estimator-class representation gives MSA $F{=}0.10$ and MRE $F{=}1.03$.

\FloatBarrier
\input{figures/fig_fig_mantel-opposite-lifts.tex}

Enriching that representation with response-free execution configuration gives $F{=}0.18$ and $1.09$: the scorer reversal survives covariates, estimator calls, hyperparameters, and tool paths (Appendix Table~\ref{tab:ce-structured-behavior}). A second representation removes both final responses and selected-method fields before embedding the trace; it gives $F_{\mathrm{MSA}}^{\mathrm{trace}}{=}0.55$ and $F_{\mathrm{MRE}}^{\mathrm{trace}}{=}1.12$. Resolving within-method execution variation therefore preserves the reversal, and response-free traces recover the same ordering without the selected-method field. Metric choice changes the geometry of the same run distribution.

The capability-matched structural treatment reverses on the same instances (Table~\ref{tab:pipeline-intervention-16}). On Textbook, Full minus Discard is $+19.6$ normalized points for MSA and $-10.3$ for MRE. On Real, the corresponding effects are $+25.0$ and $-2.1$. The task, source workflow, capability layer, and evaluation instances remain fixed; the scorer changes. Structural value is therefore metric-conditioned even within one task distribution.

\subsection{RQ2: When Does Pipeline Transfer Help?}
\label{sec:component-mechanism}

\textbf{Pipeline guidance carries the instability.}
Figure~\ref{fig:component-ablation} decomposes skill content into tools, knowledge, and pipeline structure, then reconstructs the capability layer and the complete skill. Every performance point uses the same final-gain-recovery scale: the fraction of Final AdaSkill's Base-to-Final improvement recovered by a candidate, computed per dataset before metric aggregation. Tools and knowledge rise with $F$ ($r=0.67$ and $0.52$), and their capability-only composition strengthens this profile ($r=0.95$). Pipeline alone reverses the ordering ($r=-0.74$); adding the complete component set preserves the negative profile ($r=-0.61$). The systematic reversal is therefore localized to structural transfer. Capability resources and behavioral constraints have empirically different effect profiles.

\input{figures/fig_fig_component-ablation.tex}

At dataset--metric resolution, the same conclusion holds without aggregation. Across 15 atomic Pipeline-only interventions, pre-synthesis $F$ is negatively associated with normalized pipeline lift (Pearson $r=-0.623$, $p=0.013$; Spearman $\rho=-0.568$, $p=0.027$). This sweep compares one isolated structural component with the Base Agent. The next experiment strengthens identification by holding the transferred capability layer fixed.

\textbf{Matched structural interventions recover the geometric ordering.}
We hold tools, knowledge, evaluator, and execution budget fixed and compare full-pipeline and discard-pipeline skills. Define the direct effect $\delta_{\mathrm{pipe}}=u_{\mathrm{full}}-u_{\mathrm{discard}}$, so positive values favor structural retention; $\tau_{\mathrm{pipe}}$ denotes its headroom-normalized counterpart in figures and audit tables. Across all 16 matched interventions, workflow-regime fixed effects estimate a $-9.04$ utility-point slope in $\delta_{\mathrm{pipe}}$ per unit $F$; dataset-clustered inference gives a 95\% interval of $[-13.23,-4.85]$ and $p<0.001$. The threshold-free effect sizes are ROC AUC $1.000$ for the sign of $\delta_{\mathrm{pipe}}$ and native-scale Pearson $r=-0.797$ ($p=0.0002$; Fig.~\ref{fig:pipeline-treatment}). Requiring at least three identifiable questions leaves 13 interventions and recovers the same ordering: slope $-9.63$, 95\% interval $[-13.77,-5.48]$, $p<0.001$, and ROC AUC $1.000$. The four family-excluded refits preserve negative slopes from $-5.45$ to $-9.77$; their operational boundaries are shown separately and range from $0.426$ to $0.692$. This capability-matched contrast isolates structural transfer. Appendix~\ref{app:pipeline-treatment-audit} reports every treatment and inference specification.

The ordering is not a proxy for baseline task difficulty: baseline utility is unrelated to the matched treatment effect ($r=0.033$, $p=0.903$). An unsigned distance-dependence statistic recovers the geometric association but cannot distinguish concordance from anti-concordance. $F$ supplies the orientation of the geometry: it rewards behavior distances that track score distances and rejects source contraction when the ordering reverses. The direction induced by a concrete pipeline remains candidate-specific and is measured by the matched intervention.

\textbf{The estimand is metric-conditioned structural value.}
A dataset--metric pair is the scientific intervention context because both $F$ and pipeline utility are defined conditional on the scorer. CE MSA and MRE intentionally reuse a task and baseline runs while inducing opposite structural effects. Workflow-regime fixed effects estimate this metric-conditioned mechanism, and dataset-clustered uncertainty keeps metrics sharing a dataset in the same dependence block. Complete and strict-identifiability analyses recover the same structural ordering.

\input{figures/fig_fig_pipeline-treatment.tex}

\subsection{RQ3: A Constructive Instantiation}
\label{sec:heldout-routing}

\input{figures/fig_fig_lofo-cost-pair.tex}
\textbf{Workflow-excluded constructive test.}
RQ2 supplies the threshold-free scientific result. RQ3 freezes one executable calibration algorithm and tests whether the mechanism can guide construction outside the workflow regime used to select its threshold. For each fold, we remove the target regime, select $t$ from the remaining regimes, and freeze it before target evaluation. The converter acts on seven metric modules, not 16 dataset--metric intervention units. Its inputs are CE MSA $0.167$, CE MRE $0.782$, CD F1 $0.602$, T2SQL EX $0.686$, T2SQL NSF $0.816$, unidentifiable T2SQL SRR, and FE AUC $1.037$. CE and CD use the pre-specified mean across their diagnostic datasets; T2SQL and FE use their designated conversion benchmarks, BIRD and Taobao. Six modules meet the identifiability minimum, while SRR takes the default Discard route. Dataset-level coordinates such as Child F1 $0.291$ remain RQ2 intervention observations and are not separate converter inputs. For single-metric CD and FE, AdaSkill executes the selected candidate. For multi-metric CE and T2SQL, the route vector produces one jointly composed and directly executed skill. Figure~\ref{fig:lofo-policy} shows the workflow-excluded result: mean normalized lift is 31.19, versus 21.84 for Always Full, 21.93 for Always Discard, and 33.71 for the final AdaSkill synthesized with all workflow regimes available. The routed policy retains 92.5\% of the full-data result and exceeds the stronger fixed policy in 97.8\% of selection-aware paired bootstrap samples. This constructive test is separate from the 16-intervention mechanism test in Figure~\ref{fig:pipeline-treatment}.

\textbf{Construction amortizes under deployment.}
Figure~\ref{fig:cost-amortization} combines measured task-specific conversion workload with per-query deployment cost. Workflow Compiler and the matched Full and Discard candidates each pay one workflow-to-skill conversion; AdaSkill pays the same conversion plus the \$6.12 Freedom probe. Source Workflow and Base Agent have no new construction charge, while EvoSkill includes iterative trace collection and rebuilding. CE uses the performance-aligned Base Agent mean of \$0.296 per query and matched Full, Discard, and AdaSkill costs of \$0.223, \$0.211, and \$0.207. Across the displayed horizon, AdaSkill crosses the relevant converted alternatives after 3--382 deployed queries. Appendix~\ref{app:experimental-design} gives the accounting boundary and token ledger.

\FloatBarrier

\subsection{Sensitivity and Uncertainty Analysis}
\label{sec:sensitivity-analysis}

\textbf{Budget sensitivity.}
The diagnostic budget controls how much of the behavior--outcome geometry is resolved before synthesis. Figure~\ref{fig:heatmap-MN} sweeps $F_{\mathrm{out}}$ over $M\in[1,10]$ diagnostic questions and $N\in[3,12]$ runs per question. The MSA estimate stabilizes once $N\geq5$, while MRE remains on the discard side throughout the grid. Increasing $M$ beyond six produces only marginal movement. For CE, we therefore use $M{=}N{=}6$, the visible knee of the accuracy--cost surface; T2SQL uses $M{=}10$ diagnostic questions and $N{=}6$ runs per question. At the measured average of \$0.17 per run, the CE operating point costs \$6.12 and is $2.8\times$ cheaper than $M{=}N{=}10$.

\input{figures/fig_fig_heatmap-MN-wrap.tex}

\textbf{Sampling robustness.}
The sampling mechanism changes coverage speed, not the recovered ordering. With the diversity planner, the operating-budget trace estimates are already close to their large-budget values. Under independent Base Agent executions, the same scorer ordering emerges as the budget grows: at $M{=}N{=}20$, planner and no-planner trace estimates are $0.55$ versus $0.53$ for MSA and $0.85$ versus $0.85$ for MRE. Re-estimating $F$ and skill lift with GPT-5.1 also preserves the system-level relationship in output and response-free trace spaces ($r=-0.71$ and $r=-0.79$). The planner therefore exposes valid alternatives with fewer calls, while the ordering persists under a changed behavior law. Appendix~\ref{app:mf-sensitivity} reports the complete grids, independent-run control, and cost surface.

\textbf{Statistical uncertainty.}
Question-level intervals and paired gain intervals are reported in Appendix~\ref{app:question-uncertainty}. AdaSkill reduces mean CE MRE over the Base Agent by 5.6 points with a paired 95\% interval of $[1.0,10.5]$; resolved dataset-level gains also include Textbook MSA, Synthetic MRE, BIRD EX, and Spider EX comparisons. These final comparisons evaluate the skill emitted by the adaptive converter directly.

The experiments connect mechanism to construction. Scorer control establishes that structural utility is metric-conditioned. Atomic and matched interventions localize the reversal to pipeline guidance and provide the primary geometric ordering. Competing diagnostics locate the signal in joint behavior--outcome geometry. Overall utility and the family-excluded policy show the corresponding system behavior under direct skill construction.

%% file: figures/fig_tab_source-workflows.tex
\begin{table}[t]
\centering
\caption{\textbf{Published, high-performing source workflows and externalized control.} Each source exposes an analytical decision loop through staged execution, specialized modules, shared artifacts, or feedback.}
\label{tab:source-workflows}
{\small
\setlength{\tabcolsep}{4pt}
\renewcommand{\arraystretch}{1.08}
\arrayrulecolor{black}
\begin{tabular}{>{\centering\arraybackslash}m{1.15cm}
                  >{\RaggedRight\arraybackslash}m{1.60cm}
                  >{\RaggedRight\arraybackslash}m{3.35cm}
                  >{\RaggedRight\arraybackslash}m{6.60cm}}
\toprule[0.1em]
\rowcolor{black!8}
\textbf{Task} & \textbf{Source} & \textbf{Control architecture} & \textbf{Externalized analytical decisions} \\
\midrule[0.05em]
\rowcolor{black!3}
\textbf{T2SQL} & APEX-SQL
& Staged and parallel workflow modules
& Hypothesis planning; dual-path schema pruning; parallel profiling; global synthesis; SQL exploration, execution, and verification. \\
\textbf{CE} & CAIS
& Four-stage, eight-micro-tool pipeline
& Query parsing; data analysis; variable interpretation; method selection; estimation; validation; refinement; reporting. \\
\rowcolor{black!3}
\textbf{CD} & MATMCD
& Two modules in a centralized workflow
& Retrieved-context augmentation; causal-constraint construction; statistical graph estimation and refinement. \\
\textbf{FE} & FELA
& Role-specialized evolutionary workflow
& Idea and code generation; criticism; evaluation; evolutionary selection over shared memory. \\
\bottomrule[0.1em]
\end{tabular}
}
\end{table}

%% file: figures/fig_fig_overview-perf.tex
\begin{table}[t]
\centering
\caption{\textbf{Dataset-level performance and domain-level efficiency across four task families.} Performance is reported in percent; T2SQL cells list EX, NSF, and SRR in that order. Latency and cost are averaged across datasets within each task family. Boldface denotes the best displayed result, including ties. Taobao component AUCs are recovered from Figure~\ref{fig:component-ablation}.}
\label{tab:main-performance-efficiency}
\small
\setlength{\tabcolsep}{2.4pt}
\renewcommand{\arraystretch}{1.08}
\begin{adjustbox}{max width=1.0\linewidth}
\begin{tabular}{*{10}{c}}
\toprule[0.1em]
\textbf{Dataset} & \textbf{Metric(s)} & \makecell{\textbf{Source}\\\textbf{Workflow}} & \makecell{\textbf{Base}\\\textbf{Agent}} & \textbf{EvoSkill} & \makecell{\textbf{Workflow}\\\textbf{Compiler}} & \textit{Tools} & \textit{Knowledge} & \textit{Pipeline} & \textbf{AdaSkill} \\
\midrule[0.05em]
\rowcolor{gray!12}\multicolumn{10}{c}{\textbf{Causal Estimation}} \\
Textbook & MSA $\uparrow$ / MRE $\downarrow$ & 83.3/54.0 & 61.5/31.2 & 71.8/25.4 & 82.1/26.3 & 69.2/28.3 & 79.5/25.6 & 84.6/26.5 & \textbf{89.7/23.1} \\
Synthetic & MSA $\uparrow$ / MRE $\downarrow$ & 75.9/16.2 & \textbf{100.0}/10.6 & \textbf{100.0}/7.1 & \textbf{100.0}/12.9 & 95.6/5.5 & 91.1/6.0 & \textbf{100.0}/5.5 & \textbf{100.0/4.3} \\
Real & MSA $\uparrow$ / MRE $\downarrow$ & 78.3/32.0 & 87.1/18.3 & 93.5/21.6 & 90.3/17.3 & 87.1/17.0 & 90.3/18.0 & \textbf{96.8}/17.8 & \textbf{96.8/15.9} \\
\cmidrule[0.05em](lr){1-10}
\multicolumn{2}{c}{Latency (min) / Cost (\$) $\downarrow$} & 2.1/\textbf{.172} & \textbf{0.6}/.296 & 1.1/.270 & 1.9/.608 & 0.8/.208 & 0.9/.305 & 0.9/.295 & 0.9/.207 \\
\midrule[0.05em]
\rowcolor{gray!12}\multicolumn{10}{c}{\textbf{Causal Discovery}} \\
AutoMPG & F1 $\uparrow$ & \textbf{73.7} & 67.0 & 70.0 & 70.0 & 65.0 & 67.0 & 67.0 & 65.2 \\
DWDClimate & F1 $\uparrow$ & 66.3 & 55.8 & \textbf{74.0} & 64.0 & 54.0 & \textbf{74.0} & 61.0 & 70.1 \\
Sachs & F1 $\uparrow$ & 55.7 & 86.9 & 87.0 & 67.0 & 89.0 & 89.0 & \textbf{95.0} & 92.4 \\
Child & F1 $\uparrow$ & 65.3 & 80.3 & 90.0 & 86.0 & 85.0 & 79.0 & 81.0 & \textbf{95.2} \\
\cmidrule[0.05em](lr){1-10}
\multicolumn{2}{c}{Latency (min) / Cost (\$) $\downarrow$} & 32.5/3.01 & 3.0/2.22 & 4.2/2.68 & 2.8/1.58 & 1.8/\textbf{.43} & 1.1/.48 & \textbf{0.8}/.52 & 3.1/2.50 \\
\midrule[0.05em]
\rowcolor{gray!12}\multicolumn{10}{c}{\textbf{Feature Engineering}} \\
Taobao & AUC $\uparrow$ & 65.3 & 66.0 & 66.4 & 65.8 & 65.9 & 66.7 & 65.4 & \textbf{67.0} \\
Dia & AUC $\uparrow$ & \textbf{81.2} & 80.4 & 80.6 & 79.8 & 81.1 & 81.0 & 80.3 & \textbf{81.2} \\
\cmidrule[0.05em](lr){1-10}
\multicolumn{2}{c}{Latency (min) / Cost (\$) $\downarrow$} & 630.0/13.63 & 44.0/5.11 & 72.3/6.97 & 42.9/\textbf{2.78} & 73.8/7.33 & 70.4/7.04 & 89.4/9.64 & \textbf{35.6}/4.26 \\
\midrule[0.05em]
\rowcolor{gray!12}\multicolumn{10}{c}{\textbf{Text-to-SQL}} \\
BIRD-147 & EX / NSF / SRR $\uparrow$ & \textbf{69.4}/60.8/\textbf{93.9} & 59.2/90.5/71.8 & 66.0/93.0/76.2 & 67.4/93.7/79.6 & 66.0/92.9/75.5 & 66.7/93.1/74.2 & 64.6/91.7/68.7 & 67.4/\textbf{94.1}/80.8 \\
Spider-120 & EX / NSF / SRR $\uparrow$ & \textbf{67.5}/47.0/\textbf{88.3} & 60.0/11.4/0.8 & 62.5/20.8/2.5 & 57.5/13.4/1.7 & 63.3/28.5/3.3 & 64.2/45.3/5.0 & 60.8/16.2/1.7 & 66.7/\textbf{60.5}/7.5 \\
\cmidrule[0.05em](lr){1-10}
\multicolumn{2}{c}{Latency (min) / Cost (\$) $\downarrow$} & 6.0/.72 & \textbf{1.3}/.39 & 1.6/.48 & 2.4/.59 & 1.7/\textbf{.37} & 1.6/.51 & 2.1/.88 & 1.7/.45 \\
\bottomrule[0.1em]
\end{tabular}
\end{adjustbox}
\end{table}

%% file: figures/fig_fig_metric-freedom-scatter.tex
\begin{figure}[tbp]
\centering
\includegraphics[width=.88\linewidth]{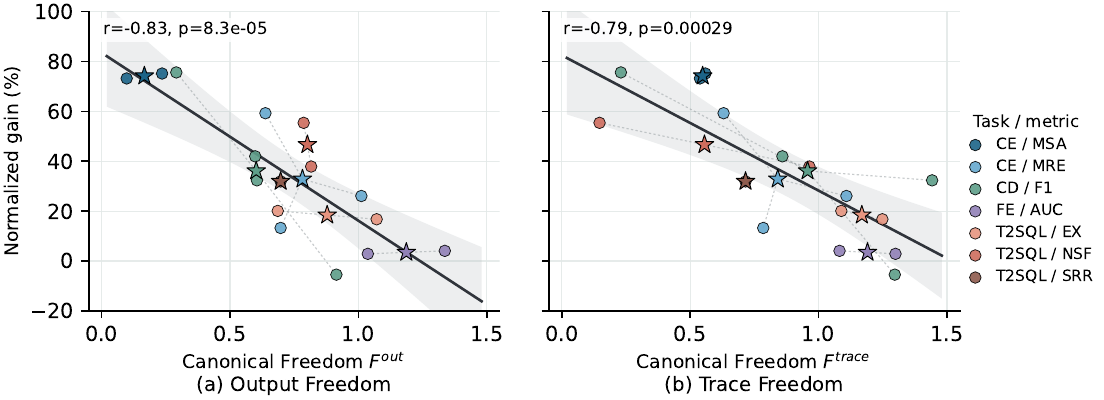}
\caption{\textbf{Overall skill utility follows Behavior-Outcome Freedom.} Canonical output-space \textbf{(a)} and response-free trace-space \textbf{(b)} estimates recover the same negative ordering with AdaSkill lift ($r=-0.83$ and $r=-0.79$) across all 16 dataset--metric tuples. Every vertical coordinate is the directly deployed output of the adaptive converter and includes capability transfer; Figure~\ref{fig:pipeline-treatment} isolates the matched structural treatment. Circles are individual datasets; stars are metric means.}
\label{fig:metric-freedom-scatter}
\end{figure}

%% file: figures/fig_fig_mantel-opposite-lifts.tex
\providecolor{ccMSA}   {HTML}{2F6F9F}
\providecolor{ccMRE}   {HTML}{B96552}
\providecolor{cFIT} {HTML}{4F857C}
\providecolor{cBG}  {HTML}{E7ECEF}
\providecolor{mLow} {HTML}{2F6F9F}
\providecolor{mHigh}{HTML}{78A9C2}
\providecolor{mMid} {HTML}{B96552}
\providecolor{cANNO}{HTML}{6B7280}

\pgfplotsset{
  fig6sc/.style={
    axis lines=left,
    axis line style={gray!60, line width=0.5pt},
    xtick align=outside, ytick align=outside,
    tick label style={font=\scriptsize},
    xlabel style={font=\small, yshift=0pt},
    ylabel style={font=\small, xshift=0pt, yshift=-3pt},
    title style={font=\scriptsize\bfseries},
    grid=none, clip=false,
  }
}

\begin{wrapfigure}{r}{0.54\textwidth}
\vspace{-8pt}
\centering
\resizebox{\linewidth}{!}{%
\begin{tikzpicture}

\begin{axis}[
  name=pA, fig6sc, width=4.5cm, height=3.4cm,
  axis background/.style={fill=ccMSA!3},
  xmin=-0.70, xmax=1.72,
  ymin=-0.52, ymax=1.56,
  xtick={0,1},
  xticklabels={,},
  ytick={0,1},
  yticklabels={0, 1},
  ylabel={$\Delta$MSA},
  xlabel={},
]
\draw[gray!35, dashed, thin] (axis cs:-0.70,0) -- (axis cs:1.72,0);
\draw[gray!35, dashed, thin] (axis cs:-0.70,1) -- (axis cs:1.72,1);
\addplot[color=cFIT, line width=1.3pt, opacity=0.72]
  coordinates {(-0.70,0)(0.5,0)(0.5,1)(1.72,1)};
\draw[fill=cBG!40, draw=none, fill opacity=0.85]
  (axis cs:0,0) ellipse [x radius=0.3126, y radius=0.3023];
\draw[fill=cBG!40, draw=none, fill opacity=0.85]
  (axis cs:1,0) ellipse [x radius=0.1008, y radius=0.0975];
\draw[fill=cBG!40, draw=none, fill opacity=0.85]
  (axis cs:1,1) ellipse [x radius=0.3075, y radius=0.2973];
\addplot[only marks, mark=*, mark size=1.2pt,
  mark options={fill=mLow, draw=gray!60, line width=0.3pt, opacity=0.95}]
  coordinates {
    (0.0204,-0.1747)(0.1081,-0.1729)(0.1682,0.1470)(0.0127,-0.2359)
    (0.0955,0.1290)(-0.1359,0.1863)(-0.1151,-0.1901)(-0.0917,0.0443)
    (-0.0488,-0.0227)(0.0764,-0.1873)(0.0062,-0.1586)(0.1906,-0.0462)
    (0.0161,-0.1158)(-0.0420,0.0116)(0.0959,0.1818)(-0.0066,-0.2622)
    (-0.0573,0.1443)(-0.0877,0.0221)(0.0969,0.1341)(0.0243,0.2160)
    (-0.1728,0.0936)(-0.0353,-0.1419)(0.0910,-0.2080)(-0.0838,0.0930)
    (-0.0316,-0.0909)(0.0055,0.0218)(0.0385,-0.2117)(-0.0504,-0.2264)
    (-0.1496,0.0514)
  };
\addplot[only marks, mark=*, mark size=1.2pt,
  mark options={fill=mHigh, draw=gray!60, line width=0.3pt, opacity=0.95}]
  coordinates {(1.0143,0.0223)(0.9778,-0.0513)(0.9471,-0.0299)};
\addplot[only marks, mark=*, mark size=1.2pt,
  mark options={fill=mLow, draw=gray!60, line width=0.3pt, opacity=0.95}]
  coordinates {
    (0.9006,0.8539)(0.8964,0.9475)(1.1308,1.0333)(1.0284,1.1634)
    (1.0590,0.8990)(1.0999,1.0498)(0.9556,1.2053)(0.8328,0.9186)
    (0.9341,0.8565)(1.0323,0.9014)(1.0599,1.0112)(0.9763,0.8245)
    (1.0755,1.1574)(1.0969,1.1788)(0.8826,1.0537)(0.9491,1.1973)
    (0.9614,1.0592)(1.1458,1.1736)(1.1416,0.8809)(0.9802,1.2568)
    (1.0307,0.7817)(0.9001,1.0428)(1.1575,1.1418)(0.9128,1.0323)
    (0.9472,1.1873)(1.0829,1.2174)(1.0092,0.7780)(0.8546,1.0859)
  };
\node[font=\scriptsize\bfseries, color=ccMSA, fill=ccMSA!10,
  draw=ccMSA!45, line width=0.4pt, rounded corners=1.2pt,
  inner xsep=2.5pt, inner ysep=1.2pt, anchor=south, align=center]
  at (axis cs:0.5,1.48)
  {$F_{\mathrm{MSA}}^{\mathrm{out}}{=}0.10$};
\node[font=\small\bfseries, anchor=south]
  at (axis cs:-0.63,1.52) {(a)};
\end{axis}

\begin{axis}[
  name=pB, fig6sc, width=4.5cm, height=3.4cm,
  axis background/.style={fill=ccMSA!3},
  at={($(pA.north east)+(0.4cm,0)$)}, anchor=north west,
  xmin=-0.02, xmax=1.02,
  ymin=-0.52, ymax=1.56,
  xtick={0,0.5,1.0},
  xticklabels={,,},
  ytick={0,1},
  yticklabels={},
  ylabel={},
  xlabel={},
]
\draw[gray!35, dashed, thin] (axis cs:-0.02,0) -- (axis cs:1.02,0);
\draw[gray!35, dashed, thin] (axis cs:-0.02,1) -- (axis cs:1.02,1);
\addplot[color=cFIT, line width=1.3pt, opacity=0.72]
  coordinates {(-0.02,0)(0.6667,0)(0.6667,1)(1.02,1)};
\addplot[only marks, mark=*, mark size=1.2pt,
  mark options={fill=mLow, draw=gray!50, line width=0.3pt, opacity=0.95}]
  coordinates {
    (0.1111,0.0137)(0.8973,0.9969)(0.0000,0.0179)(0.3084,0.0099)
    (0.9873,0.9797)(0.1945,0.0238)(0.5572,0.0131)(0.8153,1.0143)
    (0.6529,0.9814)(0.4271,-0.0025)(0.0000,-0.0065)(0.1103,0.0213)
    (0.6290,1.0072)(0.8528,1.0161)(0.0044,-0.0028)(0.8837,0.9864)
    (0.1518,0.0027)(0.9849,0.9782)(0.3320,0.0164)(0.6698,1.0066)
    (0.0000,0.0129)(0.4311,-0.0073)(0.8615,1.0235)(0.3107,0.0197)
    (0.1408,0.0139)(0.9774,0.9847)(0.5546,-0.0017)(0.8025,0.9772)
    (0.6685,0.9827)(0.1132,0.0092)(1.0000,1.0122)(0.4011,0.0234)
    (0.5796,-0.0087)(0.2782,-0.0065)(0.9102,0.9985)(0.0775,-0.0155)
    (0.4465,-0.0185)(0.0161,-0.0012)(0.9947,0.9863)(0.7969,1.0085)
    (0.6648,0.9969)(0.5742,0.0166)(0.2625,0.0100)(0.3144,-0.0094)
  };
\addplot[only marks, mark=*, mark size=1.2pt,
  mark options={fill=mHigh, draw=gray!50, line width=0.3pt, opacity=0.95}]
  coordinates {
    (0.5693,1.0166)(0.2322,1.0152)(0.4474,0.9944)(1.0000,-0.0106)
    (0.7650,0.0091)(0.2801,0.9820)(0.5040,0.9850)(0.4379,0.9754)
    (0.7423,0.0143)(0.1187,1.0082)(0.6503,0.0103)(0.2191,1.0140)
    (0.8920,-0.0021)(0.7435,0.0034)(0.3367,0.9820)(0.0000,0.9807)
  };
\node[font=\scriptsize\bfseries, color=ccMSA, fill=ccMSA!10,
  draw=ccMSA!45, line width=0.4pt, rounded corners=1.2pt,
  inner xsep=2.5pt, inner ysep=1.2pt, anchor=south, align=center]
  at (axis cs:0.5,1.48)
  {$F_{\mathrm{MSA}}^{\mathrm{trace}}{=}0.55$};
\node[font=\small\bfseries, anchor=south]
  at (axis cs:-0.01,1.52) {(b)};
\end{axis}

\begin{axis}[
  name=pC, fig6sc, width=4.5cm, height=3.4cm,
  axis background/.style={fill=ccMRE!3},
  at={($(pA.south west)+(0,-0.45cm)$)}, anchor=north west,
  xmin=-0.65, xmax=1.70,
  ymin=-0.05, ymax=1.15,
  xtick={0,1},
  xticklabels={0, 1},
  ytick={0,0.50,1.0},
  yticklabels={0, .5, 1},
  ylabel={$\Delta$MRE},
  xlabel={},
]
\addplot[color=cFIT, line width=1.3pt, opacity=0.72]
  coordinates {(-0.65,0.5077)(0.5,0.5077)(0.5,0.4928)(1.70,0.4928)};
\addplot[fill=cBG!40, draw=none, opacity=0.85] coordinates {
  (-0.0678,0.0647)(-0.1398,0.0971)(-0.1533,0.1295)(-0.1716,0.1618)
  (-0.2345,0.1942)(-0.2200,0.2265)(-0.2166,0.2589)(-0.1845,0.2912)
  (-0.1253,0.3236)(-0.0818,0.3559)(-0.1105,0.3883)(-0.2294,0.4207)
  (-0.2165,0.4530)(-0.2405,0.4854)(-0.2009,0.5177)(-0.1740,0.5501)
  (-0.2469,0.5824)(-0.3000,0.6148)(-0.2924,0.6471)(-0.1992,0.6795)
  (-0.1180,0.7119)(-0.1447,0.7442)(-0.1527,0.7766)(-0.1473,0.8089)
  (-0.0939,0.8413)(-0.0170,0.8736)(-0.0010,0.9060)(-0.0146,0.9383)
  (-0.0763,0.9707)(-0.0612,1.0031)
  (0.0612,1.0031)(0.0763,0.9707)(0.0146,0.9383)(0.0010,0.9060)
  (0.0170,0.8736)(0.0939,0.8413)(0.1473,0.8089)(0.1527,0.7766)
  (0.1447,0.7442)(0.1180,0.7119)(0.1992,0.6795)(0.2924,0.6471)
  (0.3000,0.6148)(0.2469,0.5824)(0.1740,0.5501)(0.2009,0.5177)
  (0.2405,0.4854)(0.2165,0.4530)(0.2294,0.4207)(0.1105,0.3883)
  (0.0818,0.3559)(0.1253,0.3236)(0.1845,0.2912)(0.2166,0.2589)
  (0.2200,0.2265)(0.2345,0.1942)(0.1716,0.1618)(0.1533,0.1295)
  (0.1398,0.0971)(0.0678,0.0647)
} -- cycle;
\addplot[color=cBG!70, line width=1.4pt, opacity=0.90]
  coordinates {(-0.2177,0.5077)(0.2177,0.5077)};
\addplot[fill=cBG!40, draw=none, opacity=0.85] coordinates {
  (0.7778,0.0054)(0.7154,0.0398)(0.7893,0.0742)(0.8366,0.1086)
  (0.8555,0.1430)(0.8904,0.1774)(0.9146,0.2118)(0.9074,0.2462)
  (0.8781,0.2806)(0.8294,0.3150)(0.7942,0.3494)(0.8147,0.3838)
  (0.8837,0.4182)(0.9143,0.4526)(0.9154,0.4870)(0.8936,0.5214)
  (0.8956,0.5558)(0.9473,0.5902)(0.9782,0.6246)(0.9412,0.6590)
  (0.8638,0.6934)(0.8342,0.7278)(0.8353,0.7622)(0.8181,0.7966)
  (0.7946,0.8310)(0.7396,0.8654)(0.7000,0.8998)(0.7135,0.9342)
  (0.7698,0.9687)(0.8688,1.0031)
  (1.1312,1.0031)(1.2302,0.9687)(1.2865,0.9342)(1.3000,0.8998)
  (1.2604,0.8654)(1.2054,0.8310)(1.1819,0.7966)(1.1647,0.7622)
  (1.1658,0.7278)(1.1362,0.6934)(1.0588,0.6590)(1.0218,0.6246)
  (1.0527,0.5902)(1.1044,0.5558)(1.1064,0.5214)(1.0846,0.4870)
  (1.0857,0.4526)(1.1163,0.4182)(1.1853,0.3838)(1.2058,0.3494)
  (1.1706,0.3150)(1.1219,0.2806)(1.0926,0.2462)(1.0854,0.2118)
  (1.1096,0.1774)(1.1445,0.1430)(1.1634,0.1086)(1.2107,0.0742)
  (1.2846,0.0398)(1.2222,0.0054)
} -- cycle;
\addplot[color=cBG!70, line width=1.4pt, opacity=0.90]
  coordinates {(0.9020,0.4928)(1.0980,0.4928)};
\definecolor{ptS0} {HTML}{F4D232}\definecolor{ptS1} {HTML}{CA413D}
\definecolor{ptS2} {HTML}{341056}\definecolor{ptS3} {HTML}{D24837}
\definecolor{ptS4} {HTML}{92195E}\definecolor{ptS5} {HTML}{F38D0F}
\definecolor{ptS6} {HTML}{F38D0F}\definecolor{ptS7} {HTML}{B42E4C}
\definecolor{ptS8} {HTML}{92195E}\definecolor{ptS9} {HTML}{341056}
\definecolor{ptS10}{HTML}{F1821A}\definecolor{ptS11}{HTML}{F4D232}
\definecolor{ptS12}{HTML}{F38D0F}\definecolor{ptS13}{HTML}{881561}
\definecolor{ptS14}{HTML}{F1821A}\definecolor{ptS15}{HTML}{92195E}
\definecolor{ptS16}{HTML}{341056}\definecolor{ptS17}{HTML}{F3DC41}
\definecolor{ptS18}{HTML}{D24837}\definecolor{ptS19}{HTML}{341056}
\definecolor{ptS20}{HTML}{92195E}\definecolor{ptS21}{HTML}{3F1060}
\definecolor{ptS22}{HTML}{F1821A}\definecolor{ptS23}{HTML}{D24837}
\definecolor{ptS24}{HTML}{F4D232}\definecolor{ptS25}{HTML}{341056}
\definecolor{ptS26}{HTML}{CA413D}\definecolor{ptS27}{HTML}{92195E}
\definecolor{ptS28}{HTML}{881561}
\foreach \x/\y/\col in {
  0.0236/1.0000/ptS0,-0.0040/0.7813/ptS1,0.0091/0.5223/ptS2,
  0.0371/0.2167/ptS3,0.0189/0.3585/ptS4,0.0075/0.0889/ptS5,
  0.0083/0.1084/ptS6,-0.0274/0.3130/ptS7,-0.0657/0.3278/ptS8,
  -0.0089/0.5552/ptS9,-0.0400/0.9215/ptS10,-0.0128/0.9993/ptS11,
  0.0495/0.1298/ptS12,-0.0372/0.6961/ptS13,-0.0618/0.8887/ptS14,
  -0.0306/0.2972/ptS15,-0.0289/0.5083/ptS16,0.0227/0.0000/ptS17,
  0.0080/0.2268/ptS18,0.0397/0.6184/ptS19,0.0230/0.3159/ptS20,
  -0.0131/0.4502/ptS21,0.0440/0.8831/ptS22,-0.0466/0.2368/ptS23,
  -0.0668/0.9711/ptS24,-0.0574/0.5533/ptS25,0.0311/0.7538/ptS26,
  -0.0053/0.3606/ptS27,-0.0474/0.6721/ptS28
}{ \edef\tmp{\noexpand\fill[ccMRE, opacity=0.82] (axis cs:\x,\y) circle (1.2pt);}\tmp }
\definecolor{ptD0} {HTML}{F4D232}\definecolor{ptD1} {HTML}{F38D10}
\definecolor{ptD2} {HTML}{91195E}\definecolor{ptD3} {HTML}{341057}
\definecolor{ptD4} {HTML}{B42E4C}\definecolor{ptD5} {HTML}{341057}
\definecolor{ptD6} {HTML}{B42E4C}\definecolor{ptD7} {HTML}{F4D232}
\definecolor{ptD8} {HTML}{F7B415}\definecolor{ptD9} {HTML}{AB2852}
\definecolor{ptD10}{HTML}{3F1060}\definecolor{ptD11}{HTML}{F7B415}
\definecolor{ptD12}{HTML}{3F1060}\definecolor{ptD13}{HTML}{F38D10}
\definecolor{ptD14}{HTML}{F18314}\definecolor{ptD15}{HTML}{F18314}
\definecolor{ptD16}{HTML}{3F1060}\definecolor{ptD17}{HTML}{F38D10}
\definecolor{ptD18}{HTML}{F18314}\definecolor{ptD19}{HTML}{CA413C}
\definecolor{ptD20}{HTML}{D24837}\definecolor{ptD21}{HTML}{F4D232}
\definecolor{ptD22}{HTML}{341057}\definecolor{ptD23}{HTML}{D24837}
\definecolor{ptD24}{HTML}{91195E}\definecolor{ptD25}{HTML}{F3DC41}
\definecolor{ptD26}{HTML}{F18314}\definecolor{ptD27}{HTML}{F4D232}
\definecolor{ptD28}{HTML}{F18314}\definecolor{ptD29}{HTML}{F38D10}
\definecolor{ptD30}{HTML}{341057}
\foreach \x/\y/\col in {
  1.0001/0.0000/ptD0,0.9513/0.8727/ptD1,1.0275/0.6987/ptD2,
  0.9925/0.4074/ptD3,0.9833/0.7403/ptD4,0.9722/0.3887/ptD5,
  1.0182/0.7093/ptD6,0.9807/0.0193/ptD7,0.9423/0.9353/ptD8,
  0.9465/0.2675/ptD9,1.0647/0.5486/ptD10,1.0572/0.9871/ptD11,
  1.0280/0.5436/ptD12,0.9672/0.8804/ptD13,1.0657/0.1210/ptD14,
  1.0390/0.1077/ptD15,1.0304/0.5244/ptD16,0.9929/0.8765/ptD17,
  0.9681/0.1294/ptD18,0.9435/0.2127/ptD19,1.0564/0.7941/ptD20,
  0.9938/0.0097/ptD21,0.9583/0.3875/ptD22,0.9728/0.7526/ptD23,
  1.0111/0.6434/ptD24,0.9547/1.0000/ptD25,1.0499/0.1062/ptD26,
  1.0362/0.0197/ptD27,1.0307/0.1121/ptD28,0.9905/0.8594/ptD29,
  1.0178/0.4517/ptD30
}{ \edef\tmp{\noexpand\fill[ccMRE, opacity=0.82] (axis cs:\x,\y) circle (1.2pt);}\tmp }
\node[font=\scriptsize\bfseries, color=ccMRE, fill=ccMRE!10,
  draw=ccMRE!45, line width=0.4pt, rounded corners=1.2pt,
  inner xsep=2.5pt, inner ysep=1.2pt, anchor=south, align=center]
  at (axis cs:0.5,1.07)
  {$F_{\mathrm{MRE}}^{\mathrm{out}}{=}1.03$};
\node[font=\small\bfseries, anchor=south]
  at (axis cs:-0.57,1.12) {(c)};
\end{axis}

\begin{axis}[
  name=pD, fig6sc, width=4.5cm, height=3.4cm,
  axis background/.style={fill=ccMRE!3},
  at={($(pC.north east)+(0.4cm,0)$)}, anchor=north west,
  xmin=-0.02, xmax=1.02,
  ymin=-0.05, ymax=1.15,
  xtick={0,0.5,1.0},
  xticklabels={0,.5,1},
  ytick={0,0.50,1.0},
  yticklabels={},
  ylabel={},
  xlabel={},
]
\addplot[color=cFIT, line width=1.3pt, opacity=0.72, domain=-0.02:1.02, samples=2]
  {0.5551 - 0.1101*x};
\fill[ccMRE, opacity=0.82]  (axis cs:0.1111,1.0000) circle (1.2pt);
\fill[ccMRE, opacity=0.82]  (axis cs:0.8973,0.0000) circle (1.2pt);
\fill[ccMRE, opacity=0.82]   (axis cs:0.0000,0.7813) circle (1.2pt);
\fill[ccMRE, opacity=0.82]    (axis cs:0.3084,0.5223) circle (1.2pt);
\fill[ccMRE, opacity=0.82]  (axis cs:0.9873,0.8727) circle (1.2pt);
\fill[ccMRE, opacity=0.82]   (axis cs:0.1945,0.2167) circle (1.2pt);
\fill[ccMRE, opacity=0.82]   (axis cs:0.5572,0.3585) circle (1.2pt);
\fill[ccMRE, opacity=0.82]   (axis cs:0.8153,0.6987) circle (1.2pt);
\fill[ccMRE, opacity=0.82]    (axis cs:0.6529,0.4074) circle (1.2pt);
\fill[ccMRE, opacity=0.82]  (axis cs:0.4271,0.0889) circle (1.2pt);
\fill[ccMRE, opacity=0.82]   (axis cs:0.5693,0.7403) circle (1.2pt);
\fill[ccMRE, opacity=0.82]   (axis cs:0.2322,0.3887) circle (1.2pt);
\fill[ccMRE, opacity=0.82]   (axis cs:0.4474,0.7093) circle (1.2pt);
\fill[ccMRE, opacity=0.82]  (axis cs:0.0000,0.1084) circle (1.2pt);
\fill[ccMRE, opacity=0.82]   (axis cs:0.1103,0.3130) circle (1.2pt);
\fill[ccMRE, opacity=0.82]  (axis cs:1.0000,0.0193) circle (1.2pt);
\fill[ccMRE, opacity=0.82]  (axis cs:0.6290,0.9353) circle (1.2pt);
\fill[ccMRE, opacity=0.82]   (axis cs:0.7650,0.2675) circle (1.2pt);
\fill[ccMRE, opacity=0.82]    (axis cs:0.2801,0.5486) circle (1.2pt);
\fill[ccMRE, opacity=0.82]  (axis cs:0.8528,0.9871) circle (1.2pt);
\fill[ccMRE, opacity=0.82]    (axis cs:0.5040,0.5436) circle (1.2pt);
\fill[ccMRE, opacity=0.82]  (axis cs:0.4379,0.8804) circle (1.2pt);
\fill[ccMRE, opacity=0.82]  (axis cs:0.7423,0.1210) circle (1.2pt);
\fill[ccMRE, opacity=0.82]  (axis cs:0.1187,0.1077) circle (1.2pt);
\fill[ccMRE, opacity=0.82]   (axis cs:0.0044,0.3278) circle (1.2pt);
\fill[ccMRE, opacity=0.82]   (axis cs:0.8837,0.5244) circle (1.2pt);
\fill[ccMRE, opacity=0.82]    (axis cs:0.1518,0.5552) circle (1.2pt);
\fill[ccMRE, opacity=0.82]  (axis cs:0.9849,0.8765) circle (1.2pt);
\fill[ccMRE, opacity=0.82]  (axis cs:0.3320,0.9215) circle (1.2pt);
\fill[ccMRE, opacity=0.82]  (axis cs:0.6698,0.1294) circle (1.2pt);
\fill[ccMRE, opacity=0.82]  (axis cs:0.0000,0.9993) circle (1.2pt);
\fill[ccMRE, opacity=0.82]  (axis cs:0.4311,0.1298) circle (1.2pt);
\fill[ccMRE, opacity=0.82]   (axis cs:0.8615,0.2127) circle (1.2pt);
\fill[ccMRE, opacity=0.82]   (axis cs:0.3107,0.6961) circle (1.2pt);
\fill[ccMRE, opacity=0.82]   (axis cs:0.1408,0.8887) circle (1.2pt);
\fill[ccMRE, opacity=0.82]   (axis cs:0.9774,0.7941) circle (1.2pt);
\fill[ccMRE, opacity=0.82]   (axis cs:0.5546,0.2972) circle (1.2pt);
\fill[ccMRE, opacity=0.82]  (axis cs:0.8025,0.0097) circle (1.2pt);
\fill[ccMRE, opacity=0.82]    (axis cs:0.6503,0.5083) circle (1.2pt);
\fill[ccMRE, opacity=0.82]   (axis cs:0.2191,0.3875) circle (1.2pt);
\fill[ccMRE, opacity=0.82]  (axis cs:0.8920,0.0000) circle (1.2pt);
\fill[ccMRE, opacity=0.82]   (axis cs:0.6685,0.7526) circle (1.2pt);
\fill[ccMRE, opacity=0.82]   (axis cs:0.7435,0.2268) circle (1.2pt);
\fill[ccMRE, opacity=0.82]    (axis cs:0.1132,0.6184) circle (1.2pt);
\fill[ccMRE, opacity=0.82]   (axis cs:1.0000,0.6434) circle (1.2pt);
\fill[ccMRE, opacity=0.82]   (axis cs:0.4011,0.3159) circle (1.2pt);
\fill[ccMRE, opacity=0.82]    (axis cs:0.5796,0.4502) circle (1.2pt);
\fill[ccMRE, opacity=0.82]  (axis cs:0.3367,1.0000) circle (1.2pt);
\fill[ccMRE, opacity=0.82]  (axis cs:0.0000,0.1062) circle (1.2pt);
\fill[ccMRE, opacity=0.82]  (axis cs:0.2782,0.8831) circle (1.2pt);
\fill[ccMRE, opacity=0.82]  (axis cs:0.9102,0.0197) circle (1.2pt);
\fill[ccMRE, opacity=0.82]   (axis cs:0.0775,0.2368) circle (1.2pt);
\fill[ccMRE, opacity=0.82]  (axis cs:0.4465,0.9711) circle (1.2pt);
\fill[ccMRE, opacity=0.82]    (axis cs:0.0161,0.5533) circle (1.2pt);
\fill[ccMRE, opacity=0.82]   (axis cs:0.9947,0.1121) circle (1.2pt);
\fill[ccMRE, opacity=0.82]  (axis cs:0.7969,0.8594) circle (1.2pt);
\fill[ccMRE, opacity=0.82]    (axis cs:0.6648,0.4517) circle (1.2pt);
\fill[ccMRE, opacity=0.82]   (axis cs:0.5742,0.7538) circle (1.2pt);
\fill[ccMRE, opacity=0.82]   (axis cs:0.2625,0.3606) circle (1.2pt);
\fill[ccMRE, opacity=0.82]   (axis cs:0.3144,0.6721) circle (1.2pt);
\node[font=\scriptsize\bfseries, color=ccMRE, fill=ccMRE!10,
  draw=ccMRE!45, line width=0.4pt, rounded corners=1.2pt,
  inner xsep=2.5pt, inner ysep=1.2pt, anchor=south, align=center]
  at (axis cs:0.5,1.07)
  {$F_{\mathrm{MRE}}^{\mathrm{trace}}{=}1.12$};
\node[font=\small\bfseries, anchor=south]
  at (axis cs:-0.01,1.12) {(d)};
\end{axis}

\node[font=\small, anchor=north, align=center]
  at ($(pC.south west)!0.5!(pC.south east)+(0,-0.30cm)$)
  {$d_{\mathrm{out}}$};
\node[font=\small, anchor=north, align=center]
  at ($(pD.south west)!0.5!(pD.south east)+(0,-0.30cm)$)
  {$d_{\mathrm{trace}}$};

\node[anchor=north, font=\scriptsize, inner sep=0pt]
  at ($(pC.south west)!0.5!(pD.south east)+(0,-0.72cm)$) {%
  \begin{tikzpicture}[baseline=0pt]
    \fill[ccMSA] (0,0) circle (1.8pt);
    \node[anchor=west, font=\scriptsize] at (0.15,0) {MSA pairs};
    \fill[ccMRE] (1.95,0) circle (1.8pt);
    \node[anchor=west, font=\scriptsize] at (2.10,0) {MRE pairs};
    \draw[cFIT, line width=1.2pt, opacity=0.72] (4.05,-0.03) -- (4.55,-0.03);
    \node[anchor=west, font=\scriptsize] at (4.62,0) {Best fit};
  \end{tikzpicture}};

\end{tikzpicture}%
}

\caption{\textbf{$F$ quantifies behavior-outcome coupling.} Each point pairs two Base Agent runs. Changing only the scorer separates MSA from MRE in output and response-free trace spaces.}
\label{fig:mantel-opposite-lifts}
\vspace{-8pt}
\end{wrapfigure}

%% file: figures/fig_fig_component-ablation.tex
\begin{figure}[t]
\centering
\includegraphics[width=.98\linewidth]{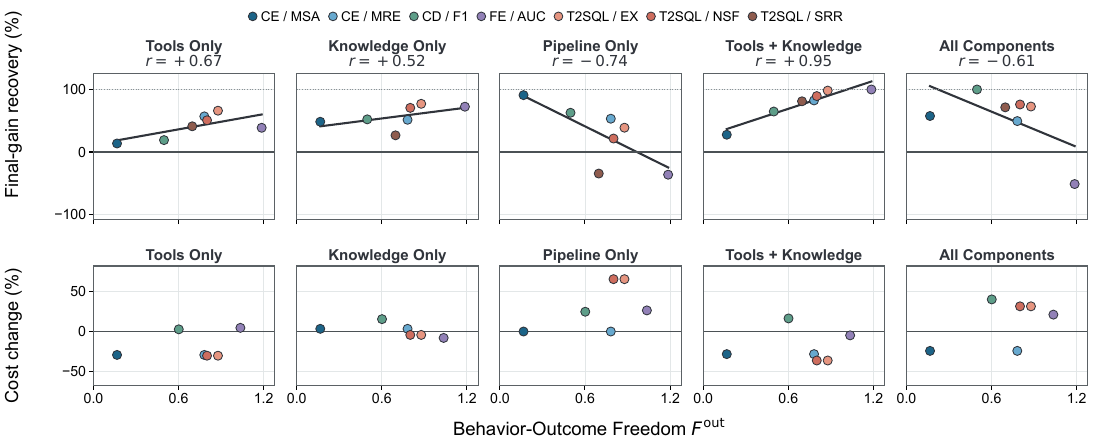}
\vspace{-3pt}
\caption{\textbf{Component ablation separates capability augmentation from structural constraint.} The performance row uses one scale throughout: each dataset--metric value is the fraction of Final AdaSkill's improvement over the Base Agent recovered by the displayed configuration, and points are metric means. Tools+Knowledge is the capability-matched Discard candidate; All Components is the matched Full candidate. Capability-only configurations rise with $F$, while Pipeline and All Components reverse this ordering. The cost row reports component-matched deployment-cost changes for all five configurations; metrics produced by the same execution share one task-level cost.}
\label{fig:component-ablation}
\end{figure}

%% file: figures/fig_fig_pipeline-treatment.tex
\begin{figure}[t]
\centering
\includegraphics[width=.96\linewidth]{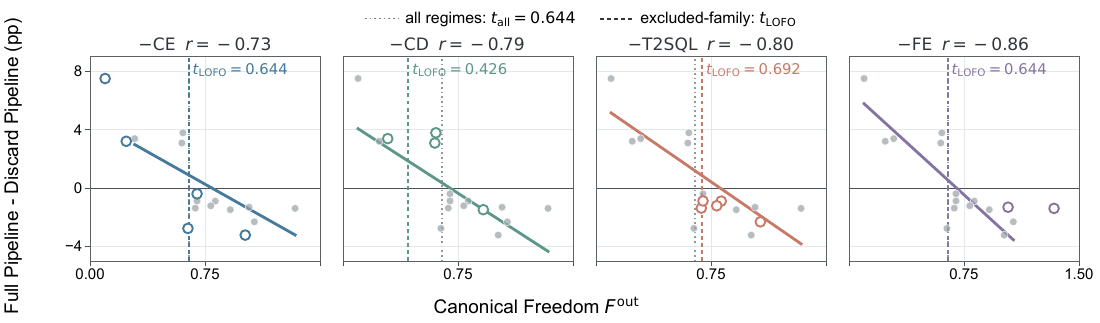}
\vspace{2pt}
\caption{\textbf{Behavior-outcome geometry orders the value of pipeline transfer.}
Each panel removes one complete workflow regime and refits the structural relationship on the remaining contrasts. Gray points form the training fold; open colored points are the excluded regime and do not affect the fitted line. The gray dotted reference is the all-regime boundary $t_{\rm all}=0.644$; colored dashed lines mark the independently calibrated LOFO boundaries. Excluding CE, CD, T2SQL, and FE gives $t_{\rm LOFO}=0.644$, $0.426$, $0.692$, and $0.644$. Every fold preserves the negative ordering ($r=-0.86$ to $-0.73$). Across all 16 matched contrasts, $r=-0.797$ ($p=0.00022$) and ROC AUC $=1.000$ (exact $p=0.0002289$).}
\label{fig:pipeline-treatment}
\end{figure}

%% file: figures/fig_fig_lofo-cost-pair.tex
\begin{figure}[t]
  \centering
  \input{figures/fig_fig_lofo-policy.tex}
  \hfill
  \input{figures/fig_fig_cost-amortization.tex}
\end{figure}

%% file: figures/fig_fig_lofo-policy.tex
\begin{minipage}[t]{0.485\textwidth}
\centering
\vspace{0pt}
\definecolor{lofoTeal}{HTML}{2A8C82}
\definecolor{lofoBlue}{HTML}{557F9E}
\definecolor{lofoSlate}{HTML}{7A838B}
\definecolor{lofoGold}{HTML}{CE8733}
\definecolor{lofoInk}{HTML}{30343B}
\definecolor{lofoGrid}{HTML}{D9E0E3}
\resizebox{0.87\linewidth}{!}{%
\begin{tikzpicture}[font=\sffamily]
  \path[use as bounding box] (-3.25,-3.02) rectangle (3.25,3.12);

  \begin{scope}[shift={(0,.18)},scale=.82]
    \foreach \r in {.153,.766,1.378,1.991,2.450}{
      \draw[lofoGrid,line width=.42pt] (90:\r)
      \foreach \a in {38.571,-12.857,-64.286,-115.714,-167.143,-218.571}
        {-- (\a:\r)} -- cycle;
    }
    \foreach \a in {90,38.571,-12.857,-64.286,-115.714,-167.143,-218.571}
      \draw[lofoGrid,line width=.38pt] (0,0)--(\a:2.45);

    \node[font=\tiny,text=lofoInk!58,anchor=e] at (180:.153) {0};
    \node[font=\tiny,text=lofoInk!58,anchor=e] at (180:.766) {20};
    \node[font=\tiny,text=lofoInk!58,anchor=e] at (180:1.378) {40};
    \node[font=\tiny,text=lofoInk!58,anchor=e] at (180:1.991) {60};
    \node[font=\tiny,text=lofoInk!58,anchor=e] at (180:2.450) {75};

    \path[draw=lofoBlue,line width=.9pt,fill=lofoBlue,fill opacity=.055]
      (90:1.455)--(38.571:.654)--(-12.857:1.259)--(-64.286:.663)--
      (-115.714:.804)--(-167.143:.851)--(-218.571:.065)--cycle;
    \path[draw=lofoSlate,line width=.9pt,densely dashed]
      (90:.773)--(38.571:1.046)--(-12.857:.915)--(-64.286:.765)--
      (-115.714:1.087)--(-167.143:.946)--(-218.571:.241)--cycle;
    \path[draw=lofoTeal,line width=1.55pt,fill=lofoTeal,fill opacity=.11]
      (90:2.426)--(38.571:1.160)--(-12.857:.915)--(-64.286:.733)--
      (-115.714:1.217)--(-167.143:1.065)--(-218.571:.241)--cycle;
    \path[draw=lofoGold,line width=1.2pt,dash pattern=on 4pt off 2pt]
      (90:2.426)--(38.571:1.160)--(-12.857:1.259)--(-64.286:.769)--
      (-115.714:1.313)--(-167.143:1.131)--(-218.571:.241)--cycle;

    \foreach \a/\r in {90/1.455,38.571/.654,-12.857/1.259,-64.286/.663,-115.714/.804,-167.143/.851,-218.571/.065}
      \fill[lofoBlue] (\a:\r) circle (1.05pt);
    \foreach \a/\r in {90/.773,38.571/1.046,-12.857/.915,-64.286/.765,-115.714/1.087,-167.143/.946,-218.571/.241}
      \fill[lofoSlate] (\a:\r) circle (1.05pt);
    \foreach \a/\r in {90/2.426,38.571/1.160,-12.857/.915,-64.286/.733,-115.714/1.217,-167.143/1.065,-218.571/.241}
      \fill[lofoTeal] (\a:\r) circle (1.45pt);
    \foreach \a/\r in {90/2.426,38.571/1.160,-12.857/1.259,-64.286/.769,-115.714/1.313,-167.143/1.131,-218.571/.241}
      \fill[lofoGold] (\a:\r) circle (1.15pt);

    \node[font=\scriptsize\bfseries,text=lofoInk,anchor=south] at (90:2.66) {CE--MSA};
    \node[font=\scriptsize\bfseries,text=lofoInk,anchor=south west] at (38.571:2.62) {CE--MRE};
    \node[font=\scriptsize\bfseries,text=lofoInk,anchor=west] at (-12.857:2.62) {CD--F1};
    \node[font=\scriptsize\bfseries,text=lofoInk,anchor=north west] at (-64.286:2.60) {T2SQL--EX};
    \node[font=\scriptsize\bfseries,text=lofoInk,anchor=north east] at (-115.714:2.60) {T2SQL--NSF};
    \node[font=\scriptsize\bfseries,text=lofoInk,anchor=east] at (-167.143:2.62) {T2SQL--SRR};
    \node[font=\scriptsize\bfseries,text=lofoInk,anchor=south east] at (-218.571:2.62) {FE--AUC};
  \end{scope}

  \begin{scope}[shift={(-2.82,-2.70)}]
    \draw[lofoBlue,line width=.9pt] (0,0)--(.30,0);
    \node[font=\tiny,text=lofoInk,anchor=west] at (.36,0) {Full};
    \draw[lofoSlate,line width=.9pt,densely dashed] (1.15,0)--(1.45,0);
    \node[font=\tiny,text=lofoInk,anchor=west] at (1.51,0) {Discard};
    \draw[lofoTeal,line width=1.55pt] (2.69,0)--(2.99,0);
    \node[font=\tiny\bfseries,text=lofoTeal!82!black,anchor=west] at (3.05,0) {LOFO};
    \draw[lofoGold,line width=1.2pt,dash pattern=on 4pt off 2pt] (4.08,0)--(4.38,0);
    \node[font=\tiny\bfseries,text=lofoGold!80!black,anchor=west] at (4.44,0) {Full-data};
  \end{scope}
\end{tikzpicture}
}
\vspace{-2pt}
\captionof{figure}{\textbf{Frozen routing across excluded workflow regimes.}
Seven metric-module decisions show Base-to-ceiling headroom recovered. SRR defaults to Discard; native outcomes are in Appendix Table~\ref{tab:f-economics}.}
\label{fig:lofo-policy}
\end{minipage}

%% file: figures/fig_fig_cost-amortization.tex
\begin{minipage}[t]{0.485\textwidth}
  \centering
  \vspace{0pt}
  \includegraphics[width=\linewidth]{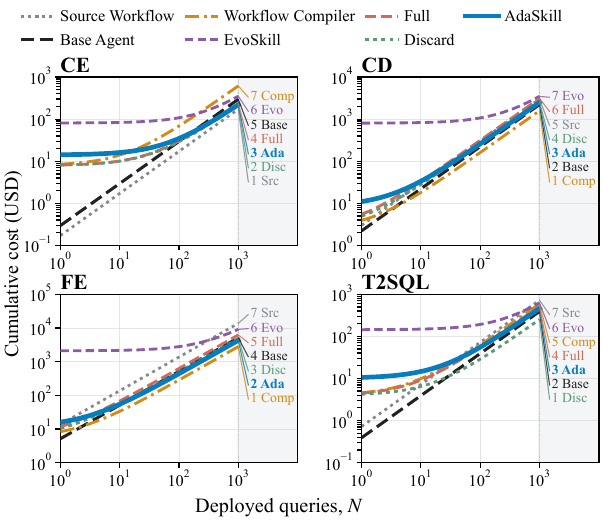}
  \vspace{-11pt}
  \captionof{figure}{\textbf{Cost amortizes with deployment.} Total cost follows $C_0+Nc$ across the four task families. AdaSkill ranks second or third at $N=10^3$; right-edge labels show the complete low-to-high cost ordering.}
  \label{fig:cost-amortization}
\end{minipage}

%% file: figures/fig_fig_heatmap-MN-wrap.tex
\begin{wrapfigure}[15]{R}{0.53\linewidth}
  \vspace{-8pt}
  \input{figures/fig_fig_heatmap-MN.tex}
  \vspace{-8pt}
\end{wrapfigure}

%% file: sections/06_limitations.tex
\section{Limitations}
\label{sec:limitations}

Our conclusions are limited to structured data-science workflows with explicit pipeline control and objective evaluation. Behavior-Outcome Freedom is metric-, agent-, representation-, and protocol-conditioned and requires meaningful outcome variation across Base Agent runs. Extending the framework to interactive or human-scored tasks will require corresponding behavior representations and outcome models.


%% file: sections/07_conclusion.tex
\section{Conclusion}

Multi-agent systems (MAS) for structured data-science tasks rely on workflows that externalize analytical control, but distilling these workflows into single-agent skills requires deciding which parts should survive transfer. We show that capability resources and pipeline guidance obey different mechanisms: the former expand what an agent can do, while the latter act as structural interventions whose value depends on behavior--outcome geometry. Behavior-Outcome Freedom ($F$) captures this geometry, and AdaSkill uses it to preserve validated capability, remove runtime coordination, and selectively inherit pipeline guidance. Across 11 datasets, this design achieves strong performance with substantially lower deployment overhead.


%% file: sections/appendix.tex
\appendix

\section{Theoretical Scope and Proofs}
\label{app:theory-section}
\label{app:theory}

This appendix develops the theory in three layers. It proves the Structural-Lift Decomposition, establishes the geometric determinacy bound for $F$, and proves Signed Anchor-Rank Transfer for arbitrary scores. A calibration corollary converts the ordinal result to continuous metric-scale lift, with the binary successful-anchor result as an exact special case. We then give the positive-orientation question-conditioned corollary and the sign, scale, and support impossibility results that delimit any $F$-only guarantee.

\subsection{Formal Setup and Statements}

Let $X$ denote the task instance, $B$ a pipeline-controlled behavioral representation, and $Y\in\mathbb{R}$ a higher-is-better score. The baseline factorizes as $P_0(dx,db,dy)=\nu(dx)\mu_0(db\mid x)K(dy\mid x,b)$, while a structural intervention replaces $\mu_0$ by $\mu_\pi$ and leaves the conditional outcome kernel fixed. Assume $\mu_\pi(\cdot\mid x)\ll\mu_0(\cdot\mid x)$, let $w_\pi=d\mu_\pi/d\mu_0$ admit a jointly measurable $L^2(P_0)$ version, and assume $Y\in L^2(P_0)$. Define
\begin{align}
  m(x,b)&=\mathbb{E}_0[Y\mid X=x,B=b],
  &\mu(x)&=\mathbb{E}_0[Y\mid X=x],\\
  G^2&=\mathbb{E}_0\{m(X,B)-\mu(X)\}^2,
  &S_\pi^2&=\mathbb{E}_0(w_\pi-1)^2=\chi_\pi^2,\\
  a_\pi&=\frac{\mathbb{E}_0[(w_\pi-1)(m-\mu)]}{S_\pi G},
  &&
  \label{eq:structural-components}
\end{align}
with $a_\pi=0$ when $S_\pi G=0$.

\begin{theorem}[Structural-Lift Decomposition]
\label{thm:structural-lift}
Under the conditions above,
\begin{equation}
  \mathrm{Lift}(\pi)
  :=\mathbb{E}_{P_\pi}Y-\mathbb{E}_{P_0}Y
  =a_\pi S_\pi G,
  \qquad |\mathrm{Lift}(\pi)|\leq S_\pi G.
  \label{eq:structural-lift}
\end{equation}
For conditionally independent $B,B'\sim\mu_0(\cdot\mid X)$,
\begin{equation}
  G^2=\frac12\mathbb{E}\{m(X,B)-m(X,B')\}^2.
  \label{eq:landscape-relief}
\end{equation}
\end{theorem}

For the pair variables $U$ and $V$ in Eq.~\eqref{eq:population-freedom}, let $\widetilde R_U$ and $\widetilde R_V$ be their centered, variance-one population mid-ranks and define
\begin{equation}
  \Gamma_{\mathrm{rank}}
  :=\operatorname{Var}\!\left(\mathbb{E}[\widetilde R_V\mid U]\right).
  \label{eq:geometric-determinacy}
\end{equation}

\begin{proposition}[Freedom Probes Geometric Determinacy]
\label{prop:geometric-determinacy}
Whenever both rank variables are non-degenerate,
\begin{equation}
  0\leq\rho_S(U,V)^2\leq\Gamma_{\mathrm{rank}}\leq1,
  \qquad 1-\Gamma_{\mathrm{rank}}\leq2F-F^2.
  \label{eq:geometric-bound}
\end{equation}
The first bound is tight if and only if $\mathbb{E}[\widetilde R_V\mid U]=\rho_S(U,V)\widetilde R_U$ almost surely.
\end{proposition}

Let $Z\sim P_0$, let $A$ and $C$ be the standardized population mid-ranks of $U$ and $V$, and fix a high-score anchor event $\mathcal H$ with $h=P_0(\mathcal H)>0$ and $Q=P_0(\cdot\mid\mathcal H)$. For $P_\pi\ll P_0$, define
\begin{equation}
  \Delta_\pi
  :=\mathbb{E}_{P_0\otimes Q}A(Z,Z^\star)
  -\mathbb{E}_{P_\pi\otimes Q}A(Z,Z^\star),
  \qquad
  \Lambda_\pi
  :=\mathbb{E}_{P_0\otimes Q}C(Z,Z^\star)
  -\mathbb{E}_{P_\pi\otimes Q}C(Z,Z^\star).
  \label{eq:anchor-contraction}
\end{equation}

\begin{theorem}[Signed Anchor-Rank Transfer]
\label{thm:anchor-transfer}
If $\chi_\pi^2=\mathbb{E}_0(w_\pi-1)^2<\infty$, then for either $\kappa\in\{-1,+1\}$,
\begin{equation}
  \left|\Lambda_\pi-\kappa\Delta_\pi\right|
  \leq
  \sqrt{\frac{2\chi_\pi^2\{1-\kappa\rho_S(U,V)\}}{h}}.
  \label{eq:anchor-rank-transfer}
\end{equation}
\end{theorem}

\subsection{Proof of the Structural-Lift Decomposition}
\label{app:structural-lift-proof}

Write $g(X,B)=m(X,B)-\mu(X)$ and $h_\pi(X,B)=w_\pi(X,B)-1$. Because $w_\pi$ is a conditional density ratio,
\begin{equation}
  \mathbb{E}_0[h_\pi\mid X]=0.
  \label{eq:conditional-density-centering}
\end{equation}
The common conditional outcome kernel gives
\begin{align}
  \mathrm{Lift}(\pi)
  &=\mathbb{E}_0[w_\pi(X,B)m(X,B)]-\mathbb{E}_0m(X,B)\\
  &=\mathbb{E}_0[h_\pi m]
   =\mathbb{E}_0[h_\pi g].
  \label{eq:lift-inner-product}
\end{align}
The two factors in the last inner product have $L^2(P_0)$ norms $S_\pi$ and $G$. Their normalized inner product is $a_\pi$, proving Eq.~\eqref{eq:structural-lift}. Cauchy--Schwarz gives $|\mathrm{Lift}(\pi)|\leq S_\pi G$. When $S_\pi G>0$, equality holds if and only if $h_\pi=cg$ $P_0$-almost surely for a constant $c$ compatible with $w_\pi\geq0$.

For conditionally independent $B,B'\sim\mu_0(\cdot\mid X)$,
\begin{align}
  \mathbb{E}[\{m(X,B)-m(X,B')\}^2\mid X]
  &=2\operatorname{Var}(m(X,B)\mid X).
\end{align}
Averaging over $X$ proves Eq.~\eqref{eq:landscape-relief}. The bound is locally sharp. If $G>0$ and $g\in L^\infty(P_0)$, then $w_\lambda=1+\lambda g$ is a valid conditional density ratio for every $0<|\lambda|<1/\lVert g\rVert_\infty$, and it satisfies $S_\lambda=|\lambda|G$, $a_\lambda=\operatorname{sign}(\lambda)$, and $\mathrm{Lift}(\lambda)=\lambda G^2$.

The fixed-kernel assumption is substantive. If the intervention changes the conditional mean from $m_0$ to $m_\pi$, then
\begin{equation}
  \mathrm{Lift}(\pi)
  =\mathbb{E}_0[(w_\pi-1)(m_0-\mu_0)]
  +\mathbb{E}_{P_\pi}[m_\pi-m_0].
  \label{eq:capability-remainder}
\end{equation}
The second term captures changes within the same represented behavior. Off-support behavior contributes an additional unidentified term. Theorem~\ref{thm:structural-lift} therefore applies when pipeline effects are mediated by $B$; unrestricted tool or knowledge expansion and pipeline effects omitted by the representation require the remainder.

\subsection{Proof of Geometric Determinacy}
\label{app:geometric-determinacy-proof}

Let $X_U=\widetilde R_U$ and $X_V=\widetilde R_V$. Both have mean zero and variance one. Since $X_U$ is measurable with respect to $U$,
\begin{equation}
  \rho_S(U,V)
  =\mathbb{E}(X_UX_V)
  =\mathbb{E}\!\left[X_U\mathbb{E}(X_V\mid U)\right].
\end{equation}
Cauchy--Schwarz yields
\begin{equation}
  \rho_S(U,V)^2
  \leq\mathbb{E}X_U^2\,
  \mathbb{E}\!\left[\{\mathbb{E}(X_V\mid U)\}^2\right]
  =\Gamma_{\mathrm{rank}}.
\end{equation}
Because $\mathbb{E}\{\mathbb{E}(X_V\mid U)\}=0$, the second moment above equals $\operatorname{Var}\{\mathbb{E}(X_V\mid U)\}$. The law of total variance gives $0\leq\Gamma_{\mathrm{rank}}\leq\operatorname{Var}(X_V)=1$. Substituting $\rho_S=1-F$ proves Eq.~\eqref{eq:geometric-bound}. Equality holds if and only if $\mathbb{E}(X_V\mid U)=cX_U$ almost surely for a constant $c$, equivalently when the unstandardized conditional rank expectation is affine in $R_U$.

The canonical question-conditioned mean inherits a corresponding guarantee. Applying the proposition to each identifiable question and then Jensen's inequality gives
\begin{equation}
  \mathbb{E}_q\Gamma_q
  \geq \mathbb{E}_q(1-F_q)^2
  \geq (1-\mathbb{E}_qF_q)^2.
  \label{eq:mean-geometric-determinacy}
\end{equation}
This is the same mean aggregate consumed by the converter and used by the structural analysis.

This result separates signed monotone coupling from optimal $L^2$ prediction of outcome-distance rank from behavioral distance. $\Gamma_{\mathrm{rank}}$ is the corresponding correlation ratio, while Spearman correlation uses the single predictor $\widetilde R_U$. Consequently, $F\approx1$ indicates weak monotone coupling, but does not exclude a nonlinear dependence. When $F>1$, the monotone component is reversed; a value near two can be highly predictive in the anti-concordant direction. Writing $D=|\rho_S|$ and $O=\operatorname{sign}(\rho_S)\in\{-1,0,+1\}$ makes the distinction explicit: $D$ measures unsigned monotone coupling strength, while $O$ gives positive, absent, or negative orientation. The next theorem uses the two nonzero orientations explicitly.

\subsection{Proof of Signed Anchor-Rank Transfer}

Recall that $A$ and $C$ are standardized population mid-ranks of
$U=d(\phi(Z),\phi(Z'))$ and $V=|s(Z)-s(Z')|$, respectively. Standardization gives
$\mathbb{E}A=\mathbb{E}C=0$, $\mathbb{E}A^2=\mathbb{E}C^2=1$, and
$\mathbb{E}(AC)=\rho_S(U,V)=1-F$. Therefore, for either $\kappa\in\{-1,+1\}$,
\begin{equation}
  \mathbb{E}(A-\kappa C)^2
  =2\{1-\kappa\rho_S(U,V)\}.
  \label{eq:rank-stress-proof}
\end{equation}
For $\kappa=+1$ this is $2F$; for $\kappa=-1$ it is $2(2-F)$. The identity includes discrete variables because the population definition uses mid-ranks.

\begin{proof}[Proof of Theorem~\ref{thm:anchor-transfer}]
Fix $\kappa\in\{-1,+1\}$ and let $E_\kappa=A-\kappa C$. Let $\mathcal H$ be the fixed anchor event, $h=P_0(\mathcal H)>0$, and $Z^\star\sim Q=P_0(\cdot\mid\mathcal H)$. Conditioning the second baseline run on the anchor event and using Eq.~\eqref{eq:rank-stress-proof} yields
\begin{equation}
  \mathbb{E}_{\mathbb{P}_0\otimes Q}E_\kappa^2
  =\frac{\mathbb{E}_{\mathbb{P}_0\otimes\mathbb{P}_0}
  [E_\kappa^2\mathbf{1}\{Z^\star\in\mathcal H\}]}{h}
  \leq\frac{2\{1-\kappa\rho_S(U,V)\}}{h}.
  \label{eq:anchor-conditioned-stress}
\end{equation}
Using $E_\kappa=A-\kappa C$ and $d\mathbb{P}_\pi=w\,d\mathbb{P}_0$ in the first coordinate,
\begin{align}
  \kappa\Lambda_\pi-\Delta_\pi
  &=\mathbb{E}_{\mathbb{P}_0\otimes Q}[(w-1)E_\kappa].
  \label{eq:anchor-error-change}
\end{align}
Cauchy--Schwarz and Eq.~\eqref{eq:anchor-conditioned-stress} give
\begin{equation}
  \left|\kappa\Lambda_\pi-\Delta_\pi\right|
  \leq
  \sqrt{\mathbb{E}_{\mathbb{P}_0}(w-1)^2}
  \sqrt{\mathbb{E}_{\mathbb{P}_0\otimes Q}E_\kappa^2}
  \leq\sqrt{\frac{2\chi_\pi^2\{1-\kappa\rho_S(U,V)\}}{h}}.
\end{equation}
Using $|\kappa x|=|x|$ proves Eq.~\eqref{eq:anchor-rank-transfer}.
\end{proof}

\subsection{Calibration to Continuous Metric Lift}

Signed Anchor-Rank Transfer is exact whenever both distance-rank variables are non-degenerate because $F$ and $\Lambda_\pi$ live on the same ordinal score-distance scale. The following corollary quantifies the only additional step required to recover mean lift on a continuous metric's original scale.

\begin{corollary}[Calibrated Continuous-Score Transfer]
\label{cor:continuous-anchor-transfer}
Under Theorem~\ref{thm:anchor-transfer}, let
\begin{equation}
  g_{\mathcal H}(y)
  :=\mathbb{E}_{Z^\star\sim Q}[C(Z,Z^\star)\mid Y=y]
  =a_{\mathcal H}-b_{\mathcal H}y+\varepsilon_{\mathcal H}(y),
  \qquad b_{\mathcal H}>0,
\end{equation}
where $a_{\mathcal H}-b_{\mathcal H}y$ is the $L^2(P_0)$ affine projection and
$\sigma_{\mathcal H}^2=\mathbb{E}_0\varepsilon_{\mathcal H}(Y)^2$. Then, for either $\kappa\in\{-1,+1\}$,
\begin{equation}
  \left|\mathrm{Lift}(\pi)-\frac{\kappa\Delta_\pi}{b_{\mathcal H}}\right|
  \leq
  \frac{1}{b_{\mathcal H}}
  \left[
    \sqrt{\frac{2\chi_\pi^2\{1-\kappa\rho_S(U,V)\}}{h}}
    +\sqrt{\chi_\pi^2}\,\sigma_{\mathcal H}
  \right].
  \label{eq:continuous-anchor-transfer-app}
\end{equation}
\end{corollary}

\begin{proof}
Because $C$ depends on the score pair only through $V=|Y-Y^\star|$,
\begin{align}
  \Lambda_\pi
  &=\mathbb{E}_{P_0}g_{\mathcal H}(Y)-\mathbb{E}_{P_\pi}g_{\mathcal H}(Y)\\
  &=b_{\mathcal H}\mathrm{Lift}(\pi)
    -\mathbb{E}_{P_0}[(w-1)\varepsilon_{\mathcal H}(Y)].
\end{align}
Hence
\begin{equation}
  |\Lambda_\pi-b_{\mathcal H}\mathrm{Lift}(\pi)|
  \leq\sqrt{\chi_\pi^2}\,\sigma_{\mathcal H}
\end{equation}
by Cauchy--Schwarz. Combining this inequality with
$|\Lambda_\pi-\kappa\Delta_\pi|\leq
\sqrt{2\chi_\pi^2\{1-\kappa\rho_S(U,V)\}/h}$ from Theorem~\ref{thm:anchor-transfer}, then dividing by $b_{\mathcal H}$, proves the result.
\end{proof}

The corollary applies to bounded continuous metrics including MRE utility, F1, AUC, NSF, and SRR after their fixed higher-is-better orientation. It does not assume perfect-score observations. A score-threshold or quantile rule may be fixed from the baseline law before intervention evaluation; $b_{\mathcal H}$ and $\sigma_{\mathcal H}$ are then baseline-estimable. A small residual makes rank transfer informative on the native metric scale, while a large residual correctly exposes that an ordinal diagnostic cannot determine raw-scale gain. If $b_{\mathcal H}\leq0$, the chosen anchors do not define a positively calibrated metric direction and this metric-scale corollary abstains; the ordinal theorem remains valid.

\subsection{Exact Binary Specialization}

Let $Y\in\{0,1\}$, $p=P_0\{Y=1\}\in(0,1)$, and $\mathcal H=\{Y=1\}$. Because the anchor succeeds, $V(Z,Z^\star)=1-Y$. The standardized mid-rank $C$ therefore takes one value $c_0$ when $Y=1$ and another $c_1$ when $Y=0$. Since $V$ is Bernoulli with probability $q=2p(1-p)$ under the baseline product law,
\begin{equation}
  c_1-c_0=\frac{1}{\sqrt{q(1-q)}}=\frac{1}{\beta_p},
  \qquad
  \beta_p=\sqrt{2p(1-p)\{1-2p(1-p)\}}.
  \label{eq:binary-rank-gap}
\end{equation}
Thus $g_{\mathcal H}(Y)=c_1-Y/\beta_p$, so $b_{\mathcal H}=1/\beta_p$ and $\sigma_{\mathcal H}=0$. Equivalently,
\begin{equation}
  \Lambda_\pi
  =\mathbb{E}_{P_0\otimes Q}C
  -\mathbb{E}_{P_\pi\otimes Q}C
  =\frac{\mathrm{Lift}(\pi)}{\beta_p}.
  \label{eq:score-anchor-lift}
\end{equation}
Corollary~\ref{cor:continuous-anchor-transfer} therefore gives
\begin{equation}
  \left|\mathrm{Lift}(\pi)-\kappa\beta_p\Delta_\pi\right|
  \leq
  \beta_p\sqrt{\frac{2\chi_\pi^2\{1-\kappa\rho_S(U,V)\}}{p}}.
  \label{eq:anchor-transfer}
\end{equation}
The two orientations yield distinct transfer statements. For $\kappa=+1$, low $F$ gives $\mathrm{Lift}(\pi)\approx\beta_p\Delta_\pi$. For $\kappa=-1$, $F$ near two gives $\mathrm{Lift}(\pi)\approx-\beta_p\Delta_\pi$, or
\begin{equation}
  \left|\mathrm{Lift}(\pi)+\beta_p\Delta_\pi\right|
  \leq \beta_p\sqrt{\frac{2\chi_\pi^2(2-F)}{p}}.
  \label{eq:anti-anchor-transfer}
\end{equation}
In particular, if $F>1$ and
\begin{equation}
  \Delta_\pi>
  \sqrt{\frac{2\chi_\pi^2(2-F)}{p}},
  \label{eq:anti-concordant-harm-condition}
\end{equation}
then the upper side of Eq.~\eqref{eq:anti-anchor-transfer} is negative and $\mathrm{Lift}(\pi)<0$. Thus strong anti-concordance can certify that successful-anchor contraction is harmful. Values only slightly above one need not satisfy this certificate; they still provide no positive-orientation guarantee and are assigned to the discard side of AdaSkill's one-sided policy.

\subsection{Question-Conditioned Anchor Transfer}

For binary question-level metrics, the positive-transfer statement uses orientation $\kappa=+1$. A calibrated operating point $t<1$ excludes a metric-level statistic with nonpositive orientation; individual question-level estimates above one remain valid and enlarge the positive-transfer error term. Their opposite-sign implication is given separately by Eq.~\eqref{eq:anti-anchor-transfer}.

\begin{corollary}[Question-Conditioned Anchor Transfer]
\label{cor:question-anchor-transfer}
Let $\mathcal I$ be the population set of evaluation questions for which a binary score has non-degenerate behavioral and score-distance ranks and $0<p_x<1$. Let $\nu_{\mathcal I}=\nu(\cdot\mid\mathcal I)$. Suppose the binary specialization of Theorem~\ref{thm:anchor-transfer} with $\kappa=+1$ holds conditionally for $\nu_{\mathcal I}$-almost every $x$, with quantities $p_x$, $\beta_x$, $F_x$, $\Delta_{\pi,x}$, and finite $\chi_{\pi,x}^2$, and assume $\mathbb{E}_{\nu_{\mathcal I}}[\beta_x^2\chi_{\pi,x}^2/p_x]<\infty$. Then
\begin{align}
  \left|
  \mathbb{E}_{x\sim\nu_{\mathcal I}}\mathrm{Lift}_x(\pi)
  -\mathbb{E}_{x\sim\nu_{\mathcal I}}[\beta_x\Delta_{\pi,x}]
  \right|
  &\leq
  \mathbb{E}_{x\sim\nu_{\mathcal I}}
  \left[
  \beta_x\sqrt{\frac{2\chi_{\pi,x}^2F_x}{p_x}}
  \right] \\
  &\leq
  \sqrt{
  2\mathbb{E}_{x\sim\nu_{\mathcal I}}[F_x]
  \mathbb{E}_{x\sim\nu_{\mathcal I}}
  \left[\frac{\beta_x^2\chi_{\pi,x}^2}{p_x}\right]
  }.
  \label{eq:question-anchor-transfer}
\end{align}
\end{corollary}

\begin{proof}
Apply the binary specialization in Eq.~\eqref{eq:anchor-transfer} conditionally on $x$, take expectations, and use $|\mathbb{E}R|\leq\mathbb{E}|R|$ for the first inequality. The second follows from Cauchy--Schwarz applied to $\sqrt{F_x}$ and $\beta_x\sqrt{\chi_{\pi,x}^2/p_x}$.
\end{proof}

The corollary bounds mean lift over the population-identifiable question set. Write $\alpha=\nu(\mathcal I)$ and let $R_{\mathcal I}$ denote its first displayed upper bound. Since binary lift has absolute value at most one, the full benchmark satisfies
\begin{equation}
  \left|
  \mathbb{E}_{x\sim\nu}\mathrm{Lift}_x(\pi)
  -\alpha\mathbb{E}_{x\sim\nu_{\mathcal I}}[\beta_x\Delta_{\pi,x}]
  \right|
  \leq \alpha R_{\mathcal I}+1-\alpha.
  \label{eq:full-benchmark-anchor-transfer}
\end{equation}
In finite samples, observed mixed questions estimate $\mathcal I$; AdaSkill aggregates their identifiable $F_q$ values with the canonical mean used by the routing policy.

\subsection{Basic Properties of the Protocol-Conditioned Statistic}

\begin{proposition}[Range, invariance, and degeneracy]
\label{prop:F-basic-properties}
For pair variables $U=d(\phi(Z),\phi(Z'))$ and $V=|s(Z)-s(Z')|$ whose mid-distribution ranks have nonzero variance:
\begin{enumerate}[label=(\roman*)]
  \item $F=1-\rho_S(U,V)\in[0,2]$;
  \item $F$ is invariant to strictly increasing transformations of $U$ or $V$, including positive affine rescaling of metric scores;
  \item if $U$ or $V$ is constant, $F$ is undefined;
  \item changing $d$, $\phi$, or the protocol-induced distribution $\mathbb{P}_0$ can change $F$.
\end{enumerate}
\end{proposition}

\begin{proof}
Part (i) follows from $\rho_S\in[-1,1]$. Part (ii) follows because strictly increasing transformations preserve ranks. In Part (iii), a constant variable has zero rank variance, so the correlation denominator is zero. Part (iv) follows directly because each of $d$, $\phi$, and $\mathbb{P}_0$ enters the joint law of $(U,V)$ in Equation~\eqref{eq:population-freedom}.
\end{proof}

When $F>1$, $\rho_S<0$: behaviorally distant runs tend to be closer in score rank than behaviorally similar runs. This is anti-concordance, not independence. The unsigned quantity $1-|\rho_S|$ recognizes predictability near both endpoints but cannot route a contraction because it erases whether the predicted transfer sign is positive or negative. Signed Anchor-Rank Transfer shows that anchor contraction has the opposite leading effect as $F$ approaches two. Our empirical plots therefore place these observations on the monotone discard side. An intervention deliberately oriented against the source contraction defines a different policy.

\section{Experimental Design and Evaluation Details}
\label{app:experimental-design}

\subsection{Task and System Coverage}

Text-to-SQL constructs and repairs programs over schemas; Causal Estimation (CE) selects and configures estimators; Causal Discovery (CD) searches graphs under causal constraints; and Feature Engineering (FE) generates and evaluates executable artifacts. These regimes vary output type, scorer, search structure, and source control while exposing the same identifiable objects: executable behavior, objective outcomes, and task-qualified pipeline intervention. The corresponding source workflows are APEX-SQL~\citep{cao2026apex}, CAIS~\citep{verma2025cais}, MATMCD~\citep{shen2024matmcd}, and FELA~\citep{ouyang2025fela}.

Every source is a published, high-performing workflow designed and evaluated for its target task, supplying a substantive structural proposal rather than an arbitrary instruction set. Qualification maps APEX-SQL control to schema and executable-query metrics, CAIS to method selection and estimation error, MATMCD to graph F1, and FELA to downstream AUC; callable resources must also pass the target interface. These gates are fixed from source artifacts before any target skill is evaluated. In the matched intervention, Full and Discard retain the same validated tools and declarative knowledge and remove the same runtime orchestration; they differ only in whether task-qualified source-pipeline guidance is retained.

EvoSkill diagnoses training failures and iteratively rewrites a skill without decomposing source control or using $F$. EvoSkill results are available for CE, CD, and T2SQL. FE reports Base Agent, Workflow Compiler, AdaSkill, and FELA Skill, a tool-using agent executing the source FELA skill. Source Workflow rows retain their native orchestration and operating configurations. Appendix~\ref{app:case-study} follows the complete CAIS conversion from eight shared-state micro-tools to one skill.

\subsection{Metric Definitions}

For CE, Method Selection Accuracy is $\mathrm{MSA}=N^{-1}\sum_i\mathbf{1}[c(\hat m_i)=c(m_i)]$, where $c$ maps aliases to canonical method classes. Mean Relative Error is $\mathrm{MRE}=N^{-1}\sum_i\min\{ |\hat\tau_i-\tau_i|/|\tau_i|,1\}$ and is reported as a percentage, with lower values better. For T2SQL, EX scores the executed result. Schema metrics score the auxiliary \texttt{selected\_columns} declaration emitted with the SQL, not columns reconstructed from the final query. If $P_i$ is that declaration and $G_i$ is the gold SQL column set, then $\mathrm{SRR}=N^{-1}\sum_i\mathbf{1}[P_i\supseteq G_i]$, $\mathrm{NSR}=N^{-1}\sum_i|P_i\cap G_i|/|G_i|$, $\mathrm{NSP}=N^{-1}\sum_i|P_i\cap G_i|/|P_i|$, and NSF is the harmonic mean of NSR and NSP. A query can therefore execute correctly while its explicit schema declaration is incomplete. CD uses edge-level F1, and FE uses test area under the receiver operating characteristic curve (AUC).

\subsection{Data Separation and Inference}

Diagnosis, conversion, and final evaluation use disjoint data roles fixed before skill synthesis. $F$ is computed from repeated Base Agent runs on diagnostic questions, using gold metric scores only to construct outcome-distance ranks. CE uses six diagnostic questions per dataset; T2SQL uses ten; both execute six runs per question. Final CE results are scored on locked Textbook, Synthetic, and Real sets with 39, 45, and 31 valid cases. Final T2SQL results use 147 BIRD and 120 Spider questions; the suffixes in BIRD-147 and Spider-120 denote these evaluation-set cardinalities. No final-evaluation instance or answer enters diagnosis or conversion.

For per-question tasks, $F$ aggregates identifiable diagnostic questions whose observed scores vary across runs. The diversity planner expands coverage across task-valid strategies, so the probe measures exploratory behavioral support. It returns only metric-level $F$ values to the converter; probe traces, question identities, and answers are never conversion inputs. The converter receives the source paper, workflow code, task specification, and tool interfaces, but no benchmark instances or answers. Threshold calibration uses only structural outcomes from the development workflow regimes; the resulting rule is frozen before evaluating the excluded target regime. Each locked evaluation set is accessed once after skill synthesis.

Each diagnostic estimate is built from all pairwise distances among repeated Base Agent runs, with at least 30 diagnostic executions per CE dataset and 50 per T2SQL dataset. CD and FE performance tables report independent execution means and standard deviations; CE and T2SQL preserve question identities for paired uncertainty. The atomic and matched intervention analyses use dataset--metric outcomes as their inferential units and never count run pairs as independent samples. Workflow-regime fixed effects absorb level differences, and dataset-clustered inference respects metrics that share executions. Sonnet~4.6 supplies the controlled intervention study. A complete replication of the task and dataset suite, Base conditions, and skill conditions on GPT-5.1 tests whether the system-level ordering survives a new backbone (Appendix~\ref{app:backbone-generalization}).

\section{Behavior-Outcome Freedom}
\label{app:metric-freedom}

\subsection{Computation Details}
\label{app:mf-details}

For each (task, dataset, metric) tuple, we compute freedom from two complementary perspectives: $F_{\text{out}}$ (structured output distances) and $F_{\text{trace}}$ (embedding-based process-trace distances). Both are computed from the ``Base Agent'' condition before any skill is generated or evaluated. The agent, harness, planner, and sampling procedure jointly define the measured support.

\input{figures/fig_tab_output-dist.tex}

\paragraph{Diversity planner.}
Low-temperature inference often concentrates on a small set of dominant outputs, so each question is seeded by a lightweight diversity planner. The planner is the diagnostic intervention used to expose task-valid alternatives whose suppression a structural skill may cause. It is one LLM call using the same backbone and receives the original problem to generate $N$ methodologically distinct solution strategies, such as different statistical methods or schema traversal orders. Each strategy is prepended as a brief hint, and the base agent executes independently for every hint. The resulting estimand describes exploratory task-valid support, not the natural frequency of deployment trajectories. The prompt requires methodologically sound approaches and permits fine-grained variations, including different covariate sets within one method. No additional sampling temperature or top-$p$ manipulation is used. The full prompt template follows.

\begin{tcolorbox}[agentpromptbox,breakable,title=Freedom probe: diversity-planner prompt]
\begin{lstlisting}[style=agentprompt]
You are given a problem below. Generate {N} reasonable approaches to solve it.

Requirements:
1. Every approach MUST be methodologically sound and appropriate for the
   problem -- do not invent approaches just to fill the list
2. Approaches may be similar to each other; fine-grained differences are
   acceptable (e.g., same method with different covariate sets, different
   hyperparameters, or slightly different model specifications)
3. If the problem only supports a small number of truly reasonable
   strategies, generate variations within those strategies; do not
   force unrelated ones
4. Be specific about WHAT to do: name the method, key steps, and any
   important implementation choices. Focus on the high-level methodology
   and strategy, not specific values, feature names, or implementation
   details (those are for the executor to decide)
5. Each approach should be self-contained and actionable

Output format (use exactly this format):
## Approach 1: [Brief Name]
[2-4 sentences describing the core idea and key steps]

## Approach 2: [Brief Name]
...

## Original Problem

{problem}

**IMPORTANT: YOUR TASK RIGHT NOW:**
Do NOT solve the problem above. Do NOT output JSON, SQL, a matrix, or
any final answer. Your ONLY job is to generate {N} reasonable approaches
using the format at the top of this prompt.
\end{lstlisting}
\end{tcolorbox}

\noindent Each generated approach is then prepended to the original problem as a one-line hint (\texttt{Hint: \{approach\}}) before the base agent executes. This planner is the only LLM prompt introduced by the Freedom probe. Output and trace extraction, distance construction, the Mantel statistic, aggregation, and bootstrap resampling are deterministic computations.

\paragraph{$F_{\text{out}}$ computation.}
We instantiate Eq.~\eqref{eq:metric-freedom} with $\mathcal{X}$ = output space. The behavioral distance $d(\cdot,\cdot)$ is task-specific (Table~\ref{tab:output-dist}).

\paragraph{$F_{\text{trace}}$ computation.}
We instantiate Eq.~\eqref{eq:metric-freedom} with $\mathcal{X}$ = trace embedding space. Each process trace $\tau_i$ is extracted from verbose run output and contains intermediate assistant messages, tool calls, and tool results. The final response is removed before the trace is embedded with \texttt{text-embedding-3-large} using persistent disk caching. The behavioral dissimilarity is $d(\tau_i, \tau_j) = (1 - \cos_{\text{sim}}(\text{embed}(\tau_i), \text{embed}(\tau_j)))/2$. This positive scaling does not change ranks or $F$. Thus $F_{\text{trace}}$ measures process variation, while $F_{\text{out}}$ separately measures structured final outputs. Runs where embedding fails are dropped from the pairwise computation.

\paragraph{Bootstrap uncertainty summaries.} For dataset-level estimates, we resample runs and reconstruct both pairwise matrices before recomputing $F$ ($B{=}1000$); this preserves the dependence among pairwise entries that share a run. For per-question metrics (CE, T2SQL), we first compute $F$ per identifiable question and then bootstrap questions while recomputing the canonical mean used by both routing and cross-task analysis. Some cells contain few identifiable questions, so the resulting standard errors are descriptive and are not presented as calibrated confidence intervals.

\paragraph{Exclusion criteria and default routing.} For per-question tasks (CE, T2SQL), we include only \emph{mixed questions} where scores vary across the observed runs, because otherwise the sample correlation denominator is zero. The diagnostic therefore has two reported outputs: the identifiable fraction $K/M$ and the conditional coupling estimate over those $K$ questions. Observed uniformity does not prove zero population variance; rare-varying questions may be omitted when $n$ is small. For per-dataset tasks (CD, FE), dissimilarities are computed globally over the available runs. A per-question metric must yield at least three identifiable questions to justify pipeline retention. CE-Synthetic-MSA and the sparse Spider NSF/SRR estimates remain audit entries. The evaluated T2SQL skill was routed by BIRD EX, which has four identifiable questions. For $F_{\text{trace}}$, runs where trace embedding fails are dropped from the pairwise computation. Routing uses $F_{\text{out}}$ whenever it reaches the minimum and falls back to $F_{\text{trace}}$ only when the output statistic is undefined. If neither representation reaches the minimum after additional sampling, the system takes the default Discard route.

This fallback retains the validated capability layer and emits no mandatory source-pipeline clauses. Table~\ref{tab:f-identifiability-audit} reports the complete finite-sample audit.

\begin{table}[htbp]
\centering
\caption{Finite-sample identifiability audit for per-question $F$. $M$ is the number of sampled questions; $K_{\mathrm{var}}$ has observed score variation; $K_{\mathrm{out}}$ and $K_{\mathrm{trace}}$ yield identifiable statistics in the two representations. Mean score variance is computed across $K_{\mathrm{out}}$ on a $[0,1]$ score scale. Values with $K<3$ are audit entries and cannot independently determine a route.}
\label{tab:f-identifiability-audit}
\small
\begin{adjustbox}{max width=\textwidth}
\begin{tabular}{lllrrrrr}
\toprule[0.1em]
\textbf{Task} & \textbf{Dataset} & \textbf{Metric} & $M$ & $K_{\mathrm{var}}$ & $K_{\mathrm{out}}$ & $K_{\mathrm{trace}}$ & \textbf{Mean score var.} \\
\midrule[0.05em]
CE & Textbook & MSA & 6 & 6 & 6 & 6 & 0.233 \\
CE & Textbook & MRE & 6 & 6 & 6 & 6 & 0.114 \\
CE & Real & MSA & 6 & 5 & 5 & 5 & 0.300 \\
CE & Real & MRE & 6 & 6 & 5 & 6 & 0.136 \\
CE & Synthetic & MSA & 6 & 1 & 1 & 1 & 0.300 \\
CE & Synthetic & MRE & 6 & 6 & 6 & 6 & 0.039 \\
T2SQL & BIRD-147 & EX & 10 & 4 & 4 & 4 & 0.275 \\
T2SQL & BIRD-147 & NSF & 10 & 6 & 6 & 6 & 0.006 \\
T2SQL & BIRD-147 & SRR & 10 & 2 & 2 & 2 & 0.250 \\
T2SQL & Spider-120 & EX & 10 & 2 & 2 & 2 & 0.250 \\
T2SQL & Spider-120 & NSF & 10 & 1 & 1 & 1 & 0.029 \\
T2SQL & Spider-120 & SRR & 10 & 0 & 0 & 0 & -- \\
\bottomrule[0.1em]
\end{tabular}
\end{adjustbox}
\end{table}

All score-varying questions produced valid trace embeddings, so the trace embedding failure count is zero. The sole representation mismatch is CE Real MRE: one question has varying scores but a constant output representation, leaving $K_{\mathrm{out}}=5$ and $K_{\mathrm{trace}}=6$. The underlying runs contain no execution failures for CE or BIRD and three failures among 50 Spider runs; the statistics use the remaining valid runs.

\subsection{Operational AdaSkill Protocol}
\label{app:operational-protocol}

\paragraph{Mechanical component boundary.}
Classification is clause-level and functional. \emph{Knowledge} is a declarative proposition about the domain, data schema, estimator assumptions, or tool semantics. It cannot prescribe, prefer, prohibit, order, repeat, or terminate an action. \emph{Pipeline} is any mapping from an observed state to an action policy, including soft recommendations, hard requirements, prohibitions, numbered procedures, conditional branches that select tools or methods, validation gates, retries, and stopping rules. If a clause mixes a fact with an action rule, the entire clause is pipeline. \emph{Orchestration} additionally requires runtime control across modules or agents through handoffs, shared state, voting, debate, or loops. This rule is independent of section names and surface wording: placing ``FORBIDDEN'' or ``do this first'' under a knowledge heading does not make it knowledge.

\paragraph{AdaSkill converter contract.}
Every reported conversion follows the same five-step contract. First, the converter inventories callable tools, declarative knowledge, state-to-action pipeline clauses, and orchestration mechanisms from the source paper and code. Second, it retains tool interfaces, rewrites knowledge as descriptive reference material, and removes agent identities, handoffs, shared-state operations, and runtime control loops. Third, it applies the development-calibrated threshold from Section~\ref{sec:mas2skill}: retain all task-qualified pipeline clauses for $F<t$ and remove them for $F\geq t$. In a family-excluded fold, the threshold is learned once from the other regimes and is then immutable for every metric in the target family. Fourth, it emits one Markdown skill with fixed sections for scope, tools, reference knowledge, metric-specific guidance, and failure checks. Fifth, an automated clause audit checks imports, section presence, binary routing compliance, and the absence of benchmark instances or answers. Headings do not override the functional classification, and a mixed fact--directive clause inherits the pipeline label. The converter receives metric-level $F$ values and the already frozen threshold but never receives probe traces, question identities, skill outcomes, or final-test feedback.

\begin{tcolorbox}[agentpromptbox,breakable,title=AdaSkill Converter: binary workflow-to-skill prompt]
\begin{lstlisting}[style=agentprompt]
You are the AdaSkill workflow-to-skill converter. Convert a task-qualified source workflow into one executable skill for a tool-using agent.

Inputs:
  source paper: {paper_path}
  source code: {codebase_path}
  task specification: {task_specification}
  tool interfaces: {tool_interfaces}
  metric-level Freedom values: {metric_freedom}
  frozen threshold: {threshold_t}
  output directory: {output_dir}

1. Inventory every source clause by function:
     tool = callable implementation;
     knowledge = declarative domain, data, estimator, or tool fact with no action directive;
     pipeline = any recommendation, prohibition, ordered step, state-to-action branch, validation gate, retry, or stopping rule;
     coordination = runtime handoff, shared state, voting, debate, or control loop.
   Classify every mixed fact-and-directive clause as pipeline. Section headings and polite wording do not change the class.
2. Preserve callable tools. Rewrite relevant knowledge as descriptive reference material. Remove agent identities, handoffs, shared state, voting, debate, and runtime control loops.
3. For each metric module, apply one binary structural rule:
   choose t in (0, 1);
   if F < t: retain the task-qualified pipeline clauses;
   if F >= t: remove all pipeline clauses.
   Do not infer a new threshold and do not create an intermediate route.
4. Compose the routed metric modules into one skill. Deduplicate shared tools and knowledge. Keep conflicting metric requirements explicitly scoped to their metric-specific decisions.
5. Write SKILL.md with fixed sections for scope, tools, reference knowledge, metric-specific guidance, and failure checks. Write required tool files under tools/.
6. Reclassify every emitted clause by the same functional rule. On a discard route, reject any source-derived state-to-action clause even if it appears under reference knowledge. Validate tool imports, required sections, binary routing compliance, and the absence of benchmark instances or answers.

Do not inspect probe traces, question identities, skill outcomes, validation feedback, or final-test data. Write the complete skill package directly to {output_dir}.
\end{lstlisting}
\end{tcolorbox}

\paragraph{Multi-metric composition.}
Metric modules are composed by the decision they score, not by averaging their $F$ values. In CE, MSA controls method-selection guidance and MRE controls numerical-estimation guidance. In T2SQL, SRR and NSF govern schema-evidence checks, while EX governs executable-query construction and repair. Shared tools and knowledge are deduplicated. A retained pipeline remains scoped to its metric-specific decision and cannot impose structure on a subproblem assigned to the discard route. This fixed mapping produces one executable skill from multiple metric-conditioned modules.

\paragraph{Structural robustness and dependence control.}
\label{app:pipeline-treatment-audit}
The atomic sweep contains 15 Pipeline-only interventions against the Base Agent and excludes the Dia Workflow Compiler proxy. Its normalized association is Pearson $r=-0.623$ ($p=0.013$) and Spearman $\rho=-0.568$ ($p=0.027$). This estimate uses neither the Full nor Discard outcomes.

The complete matched estimand uses all 16 capability-matched interventions. Because every metric has already been mapped to higher-is-better utility in $[0,1]$, its primary effect is the direct utility difference $u_{\mathrm{full}}-u_{\mathrm{discard}}$. Workflow-regime fixed effects estimate a slope of $-9.04$ utility points per unit $F$. Clustering shared-metric dependence at the dataset level gives a 95\% interval of $[-13.23,-4.85]$ and $p<0.001$. Native utility differences give Pearson $r=-0.797$ ($p=0.0002$), and the normalized effect gives $r=-0.669$ ($p=0.0046$). The strict route-grade subset requires $K\geq3$ for per-question diagnostics and retains 13 interventions; it recovers slope $-9.63$, interval $[-13.77,-5.48]$, and $p<0.001$. Leave-one-regime fixed-effect slopes remain negative after removing CE, CD, T2SQL, or FE. Baseline utility is unrelated to normalized treatment effect ($p=0.90$).

For CE, we additionally refine the task-authored representation without an embedding model. The hierarchical execution signature keeps estimator identity as the coarse decision and separates executions within an estimator by statistical packages, estimator calls, dataframe columns, causal-design tokens, keyword arguments, and whitelisted configurations such as bandwidth, cutoff, kernel, and estimand. It parses assistant tool calls only. Tool results, final responses, and estimated effects are never read. Table~\ref{tab:ce-structured-behavior} shows that the controlled scorer reversal survives: Textbook MSA/MRE move from $0.098/1.012$ to $0.181/1.085$, while Real moves from $0.236/0.697$ to $0.255/0.762$. Replacing the five canonical CE coordinates by these structured values preserves the matched ordering on the route-grade set (normalized $r=-0.604$, $p=0.029$; utility-point $r=-0.669$, $p=0.012$) and complete audit (normalized $r=-0.612$, $p=0.012$; utility-point $r=-0.688$, $p=0.0032$).

\input{figures/fig_tab_ce-structured-behavior.tex}

Across the 16 matched contrasts, the normalized effect gives Pearson $r=-0.669$ ($p=0.0046$) and Spearman $\rho=-0.735$ ($p=0.0012$); native-score differences give $r=-0.797$ ($p=0.0002$) and $\rho=-0.747$ ($p=0.0009$). Seven metric means give native-scale $r=-0.855$ ($p=0.014$; exact permutation $p=0.0113$), and every regime deletion remains negative. Table~\ref{tab:pipeline-intervention-16} reports the complete treatment audit.

\paragraph{Threshold-free ranking.}
Treating Full $>$ Discard as the positive structural action, $-F$ attains ROC AUC $1.000$ across all 16 interventions and on the strict-identifiability subset. Exact one-sided label enumeration conditional on the observed class counts gives $p=0.0002$ and $p=0.0008$. Every leave-one-regime-out deletion preserves perfect ordering.

The ordering is also insensitive to effects near zero. Sequential deletion of one through six smallest Full-minus-Discard contrasts preserves ROC AUC $1.000$; after six deletions the exact one-sided test remains significant ($p=0.0040$). Small treatment signs therefore do not create the ranking result. Table~\ref{tab:structural-ranking} summarizes both ranking audits.

\input{figures/fig_tab_structural-ranking.tex}
\FloatBarrier

Workflow-regime fixed effects isolate metric-conditioned effect modification, and dataset clustering keeps metrics sharing executions in the same dependence block. On all 16 interventions, the utility-point slope is $-9.04$ with 95\% interval $[-13.23,-4.85]$ ($p<0.001$). The strict-identifiability subset gives slope $-9.63$, interval $[-13.77,-5.48]$ ($p<0.001$). Native, normalized, and every regime-deletion analysis preserve the negative direction.

\input{figures/fig_tab_pipeline-intervention-16.tex}
Native-score effects, metric aggregation, family deletion, and the baseline-utility control all preserve the structural ordering. These tests target the matched intervention itself and remain distinct from AdaSkill's cross-fitted operating policy.

Baseline test utility is the locked Base-Agent result reported in the matched intervention table and is unrelated to the structural effect ($r=0.033$, $p=0.903$). An unsigned distance-dependence statistic corroborates the geometric premise but discards whether the association is concordant or anti-concordant. The signed rank construction in $F$ supplies the orientation of the behavior--outcome geometry; pipeline direction remains intervention-specific.

\input{figures/fig_tab_f-economics.tex}
\FloatBarrier

\subsection{Backbone Replication on GPT-5.1}
\label{app:backbone-generalization}

We repeat the complete dataset suite and both Base and skill-augmented conditions with GPT-5.1 under the same skill definitions, prompts, tools, and evaluation protocol. Figure~\ref{fig:metric-freedom-scatter-gpt} shows that the system-level relationship survives the changed behavior law. Output-space $F$ gives $r=-0.71$ ($p<0.01$), and response-free trace-space $F$ gives $r=-0.79$ ($p<0.001$).

\input{figures/fig_fig_metric-freedom-scatter_openai.tex}

\subsection{Question-Level Uncertainty for Final CE and T2SQL Results}
\label{app:question-uncertainty}

Table~\ref{tab:question-uncertainty} reports uncertainty over evaluation questions. Wilson intervals are used for MSA and EX. MRE intervals use 10,000 question-level bootstrap samples. Table~\ref{tab:gain-intervals} reports paired question-bootstrap intervals for direct method gains where retained outputs permit question matching. CE Overall MRE averages the three dataset-level paired mean differences with equal dataset weights; its stratified bootstrap resamples questions independently within each dataset before averaging. Synthetic Acc@$\tau$ thresholds each question at clipped relative error $\leq\tau$ and bootstraps the paired binary differences.

\begin{table}[htbp]
\centering
\caption{Final AdaSkill estimates with 95\% question-level intervals. MRE is lower-is-better; all entries are percentages.}
\label{tab:question-uncertainty}
\small
\begin{tabular}{lllrrr}
\toprule[0.1em]
\textbf{Task} & \textbf{Dataset} & \textbf{Metric} & \textbf{Estimate} & \textbf{95\% interval} & \textbf{$n$} \\
\midrule[0.05em]
CE & Textbook & MSA & 89.7 & $[76.4,95.9]$ & 39 \\
CE & Textbook & MRE & 23.1 & $[14.7,32.4]$ & 39 \\
CE & Synthetic & MSA & 100.0 & $[92.1,100.0]$ & 45 \\
CE & Synthetic & MRE & 4.3 & $[2.7,6.1]$ & 45 \\
CE & Real & MSA & 96.8 & $[83.8,99.4]$ & 31 \\
CE & Real & MRE & 15.9 & $[8.2,25.4]$ & 31 \\
T2SQL & BIRD-147 & EX & 67.4 & $[59.4,74.4]$ & 147 \\
T2SQL & Spider-120 & EX & 66.7 & $[57.8,74.5]$ & 120 \\
\bottomrule[0.1em]
\end{tabular}
\end{table}

\input{figures/fig_tab_gain-intervals.tex}

\subsection{Sensitivity to $M$ and $N$}
\label{app:mf-sensitivity}

A central design decision in computing $F$ is how many questions ($M$) and how many runs per question ($N$) to collect.
Both quantities directly affect the reliability of the Mantel-test estimate: too few runs leave pairwise distances noisy; too few questions make the question-level aggregate unstable.
At the same time, collecting $M{\times}N$ runs has a real cost. Figure~\ref{fig:heatmap-MN} uses the measured average of \$0.17 per run, so cost is linear in $M{\times}N$ and reaches \$20.4 at the maximum budget of $M{=}10$, $N{=}12$.

Figure~\ref{fig:heatmap-MN} reports $F_{\mathrm{out}}$, while Figure~\ref{fig:mf-sensitivity} reports the complementary $F_{\mathrm{trace}}$ sweep and planner ablation. These are different behavioral representations and need not coincide numerically at the same $(M,N)$ operating point. The trace-space sweep reveals a clear saturation pattern for both $F_{\mathrm{MSA}}^{\mathrm{trace}}$ and $F_{\mathrm{MRE}}^{\mathrm{trace}}$.
For $F_{\mathrm{MSA}}^{\mathrm{trace}}$ (panel~a), the estimate stabilises once $N \geq 5$: beyond this point the colour map changes by less than two percentage points per additional run column.
For $F_{\mathrm{MRE}}^{\mathrm{trace}}$ (panel~b), even $N{=}3$ places the estimate above the calibrated threshold range; additional runs reduce variation without changing the binary decision.
In both panels, increasing $M$ beyond~6 produces only marginal shifts across the grid, identifying $M{=}6$ as the operating point.

Taken together, the gold operating point at $M{=}6$, $N{=}6$ (\$6.12) is the empirical knee of the accuracy--cost curve, consuming 30\% of the maximum budget while preserving the binary separation. Relative to $M{=}N{=}10$ (\$17.00), it reduces measurement cost by $2.8\times$.
Cells with $N<3$ are greyed out because fewer than three runs yield too few pairwise entries for the rank statistic to be informative.

\paragraph{Diversity planner ablation.}
Figure~\ref{fig:mf-sensitivity} compares trace-space $F$ estimation with and without the diversity planner. Without it, each of the $N$ runs is an independent random execution of the base agent. Because low-temperature inference concentrates on a small number of dominant reasoning paths, many runs produce near-identical trajectories and undersample plausible alternatives. At $N{=}6$, the no-planner $F_{\mathrm{MSA}}^{\mathrm{trace}}$ estimate is 42 versus the large-budget reference of approximately 53, whereas the planner-conditioned estimate is 52. At $M{=}N{=}20$, the two protocols recover nearly identical metric ordering and values. The planner therefore accelerates coverage of exploratory support without manufacturing the scorer reversal.

\begin{figure}[htbp]
\centering
\begin{subfigure}[t]{\linewidth}
  \centering
  \includegraphics[width=0.96\linewidth]{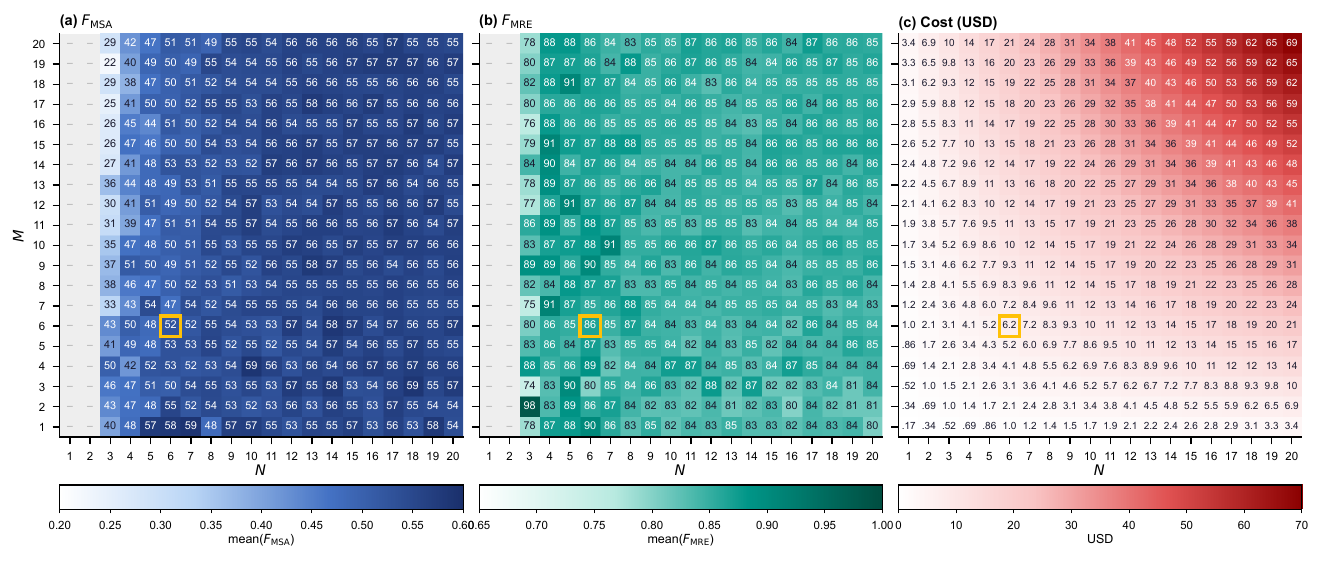}
  \caption{$F_{\mathrm{trace}}$ with diversity planner (used in all experiments).}
  \label{fig:mf-sensitivity-div}
\end{subfigure}\\[2pt]
\begin{subfigure}[t]{\linewidth}
  \centering
  \includegraphics[width=0.96\linewidth]{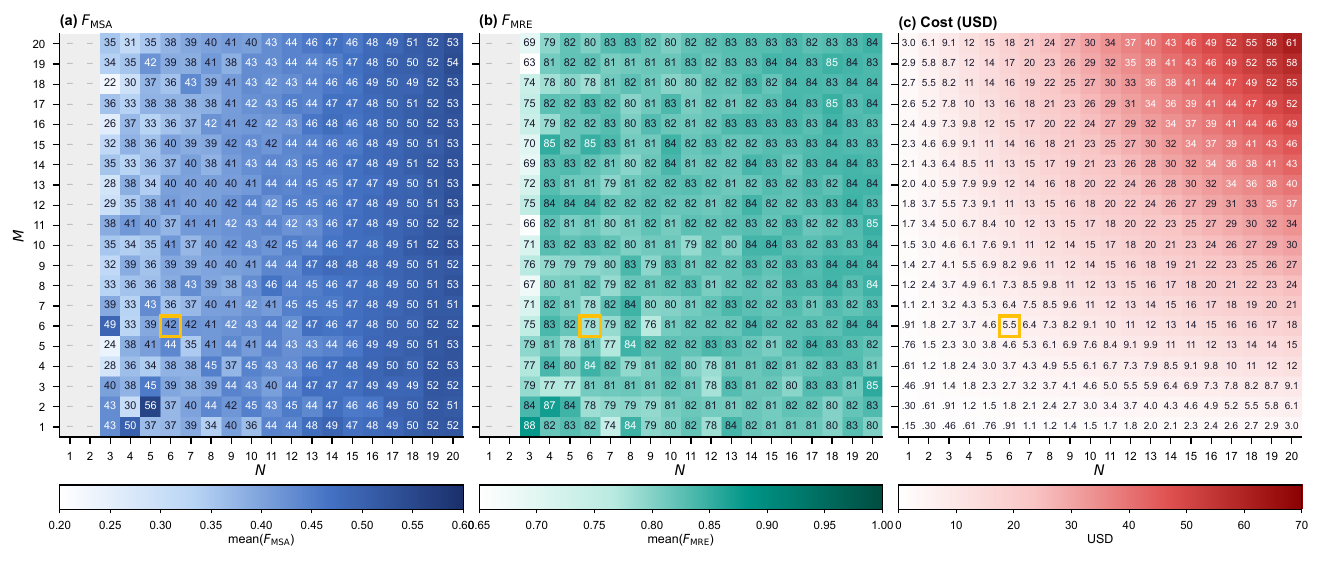}
  \caption{$F_{\mathrm{trace}}$ without diversity planner (independent random runs).}
  \label{fig:mf-sensitivity-nodiv}
\end{subfigure}
\caption{\textbf{Diversity planner ablation for trace-space Freedom.} Both panels show $F_{\mathrm{MSA}}^{\mathrm{trace}}$ and $F_{\mathrm{MRE}}^{\mathrm{trace}}$ ($\times100$) as a function of evaluation budget ($M$ questions, $N$ runs). At $M{=}N{=}20$, planner and no-planner protocols give similar displayed estimates ($F_{\mathrm{MSA}}^{\mathrm{trace}}{:}\;55$ vs.\ $53$; $F_{\mathrm{MRE}}^{\mathrm{trace}}{:}\;85$ vs.\ $85$). At $M{=}N{=}6$, the planner-conditioned values are closer to their respective large-budget values than the independent-run values. Figure~\ref{fig:heatmap-MN} reports a separate output-space sweep.}
\label{fig:mf-sensitivity}
\end{figure}

\subsection{Freedom across Diagnostic Questions}
\label{app:mf-distribution}

Figure~\ref{fig:mf-perrun} shows every retained identifiable question-level estimate.

\begin{figure}[htbp]
\centering
\includegraphics[width=\linewidth,trim=0 0 0 22pt,clip]{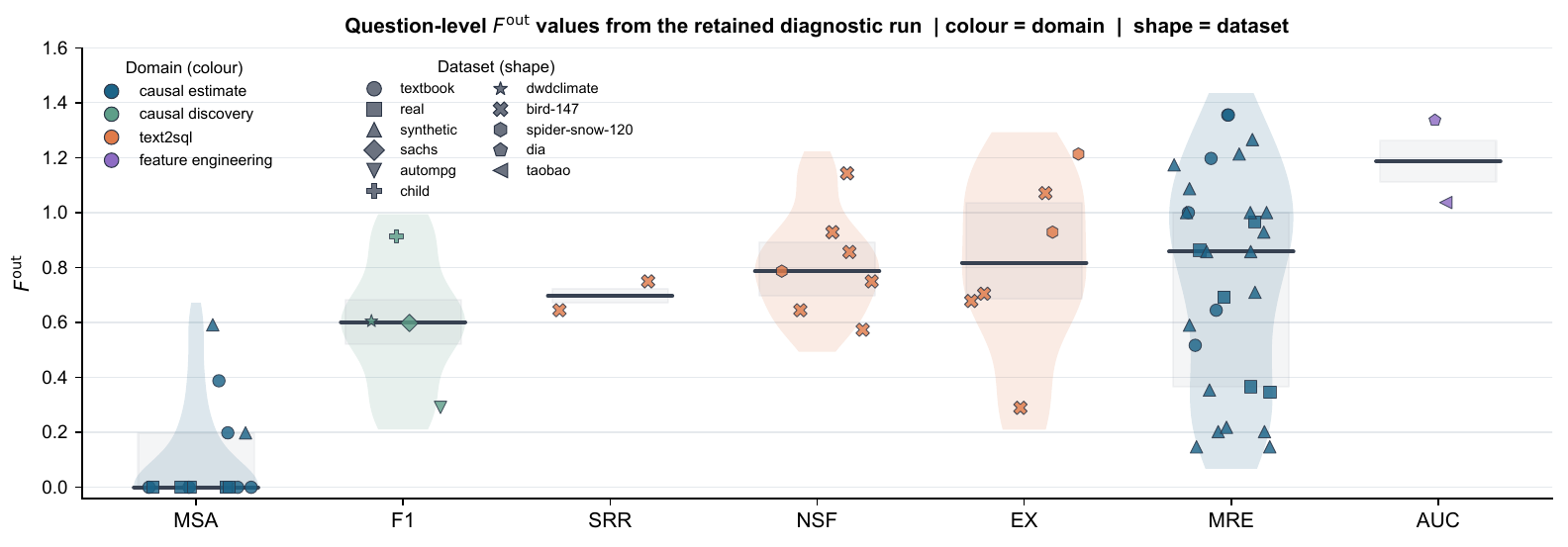}
\caption{Question-level $F^{\mathrm{out}}$ values from the retained diagnostic run for each dataset. Colour encodes domain and shape encodes dataset.}
\label{fig:mf-perrun}
\end{figure}

Each dot is an identifiable question-level $F^{\mathrm{out}}$ estimate from the retained diagnostic run for a dataset.
The canonical dataset--metric coordinate is their mean, used without alteration in both structural analysis and routing. Budget-level stability is evaluated separately in Figures~\ref{fig:heatmap-MN} and~\ref{fig:mf-sensitivity}, while Table~\ref{tab:f-identifiability-audit} reports how many questions identify each metric.

\subsection{$F$ Robustness across Datasets}
\label{app:mf-generalization}

We test cross-dataset stability using Causal Estimation (CE), three datasets (Textbook, Real, Synthetic), and two metrics (MSA, MRE). The retained Synthetic-MSA diagnostic has only one identifiable question and is excluded from this distributional comparison, matching the identifiability audit in Table~\ref{tab:f-identifiability-audit}.

Tables~\ref{tab:mf-gen-dist} and~\ref{tab:mf-gen-cross} report, respectively, the per-dataset $F$ distributions and the pairwise cross-dataset comparisons for both $F^{\mathrm{out}}$ and $F^{\mathrm{trace}}$.

\paragraph{Distribution.}
Table~\ref{tab:mf-gen-dist} shows stable ordering alongside numerical variation. The identifiable Textbook and Real MSA estimates are close in both representations. For MRE, $F^{\mathrm{out}}$ ranges from 0.638 to 1.012, while Real and Synthetic are closer (0.697 vs.\ 0.638, $\Delta\mu = 0.059$). These results preserve the MSA--MRE separation across the identifiable CE comparisons.

\paragraph{Cross-dataset comparison.}
Table~\ref{tab:mf-gen-cross} reports pairwise $\Delta\mu$ (absolute difference of means), 95\% interval overlap, and exploratory Mann--Whitney $U$ test $p$-values (no shared question IDs exist across datasets). The two valid MSA representation comparisons have interval overlap. For MRE/$F^{\mathrm{out}}$, Real and Synthetic have $\Delta\mu=0.059$ and $p=1.00$. MRE/$F^{\mathrm{trace}}$ for Textbook versus Synthetic gives $p=0.026$ with non-overlapping intervals, localizing the main cross-dataset difference to this representation and metric.

\begin{table}[htbp]
\centering
\caption{$F$ distribution per dataset for Causal Estimation. ``n'' = number of questions with a valid $F$ estimate; std and var computed with Bessel's correction.}
\label{tab:mf-gen-dist}
\small
\begin{adjustbox}{max width=\textwidth}
\begin{tabular}{ccc cccc c}
\toprule[0.1em]
\textbf{Dataset} & \textbf{Metric} & \textbf{Type} & $n$ & $\mu$ & $\sigma$ & $\sigma^2$ & Range \\
\midrule[0.05em]
\multirow{4}{*}{Textbook} & \multirow{2}{*}{MSA} & $F^{\mathrm{out}}$   & 6 & 0.098 & 0.163 & 0.027 & $[0.000,\ 0.388]$ \\
                        &                      & $F^{\mathrm{trace}}$ & 6 & 0.538 & 0.370 & 0.137 & $[0.147,\ 1.000]$ \\
\cmidrule[0.05em]{2-8}
                        & \multirow{2}{*}{MRE} & $F^{\mathrm{out}}$   & 6 & 1.012 & 0.361 & 0.130 & $[0.517,\ 1.355]$ \\
                        &                      & $F^{\mathrm{trace}}$ & 6 & 1.109 & 0.268 & 0.072 & $[0.780,\ 1.491]$ \\
\midrule[0.05em]
\multirow{4}{*}{Real}   & \multirow{2}{*}{MSA} & $F^{\mathrm{out}}$   & 5 & 0.236 & 0.256 & 0.066 & $[0.000,\ 0.592]$ \\
                        &                      & $F^{\mathrm{trace}}$ & 5 & 0.559 & 0.391 & 0.153 & $[0.147,\ 1.000]$ \\
\cmidrule[0.05em]{2-8}
                        & \multirow{2}{*}{MRE} & $F^{\mathrm{out}}$   & 5 & 0.697 & 0.367 & 0.135 & $[0.147,\ 1.167]$ \\
                        &                      & $F^{\mathrm{trace}}$ & 6 & 0.828 & 0.488 & 0.238 & $[0.255,\ 1.426]$ \\
\midrule[0.05em]
\multirow{2}{*}{Synthetic} & \multirow{2}{*}{MRE} & $F^{\mathrm{out}}$   & 6 & 0.638 & 0.466 & 0.217 & $[0.147,\ 1.266]$ \\
                           &                      & $F^{\mathrm{trace}}$ & 6 & 0.630 & 0.182 & 0.033 & $[0.412,\ 0.861]$ \\
\bottomrule[0.1em]
\end{tabular}
\end{adjustbox}
\end{table}

\begin{table}[htbp]
\centering
\caption{Exploratory cross-dataset comparisons for Causal Estimation. No shared question IDs exist across datasets; Mann--Whitney $U$ tests ($p_{\mathrm{MW}}$) are unadjusted. ``CI'' records whether the reported 95\% intervals for $\mu_A$ and $\mu_B$ overlap. The dagger marks $p<0.05$ with non-overlapping intervals; the comparison is descriptive.}
\label{tab:mf-gen-cross}
\small
\begin{adjustbox}{max width=\textwidth}
\begin{tabular}{ccccccccc}
\toprule[0.1em]
\textbf{Metric} & \textbf{Type} & \textbf{Dataset A} & $\mu_A$ & \textbf{Dataset B} & $\mu_B$ & $\Delta\mu$ & CI & $p_{\mathrm{MW}}$ \\
\midrule[0.05em]
\multirow{2}{*}{MSA} & $F^{\mathrm{out}}$
  & Textbook & 0.098 & Real      & 0.236 & 0.138 & \checkmark & 0.368 \\
\cmidrule[0.05em]{2-9}
  & $F^{\mathrm{trace}}$
  & Textbook & 0.538 & Real      & 0.559 & 0.021 & \checkmark & 1.000 \\
\midrule[0.05em]
\multirow{6}{*}{MRE} & \multirow{3}{*}{$F^{\mathrm{out}}$}
  & Textbook & 1.012 & Real      & 0.697 & 0.314 & \checkmark & 0.314 \\
  &        & Textbook & 1.012 & Synthetic & 0.638 & 0.374 & \checkmark & 0.148 \\
  &        & Real   & 0.697 & Synthetic & 0.638 & 0.059 & \checkmark & 1.000 \\
\cmidrule[0.05em]{2-9}
  & \multirow{3}{*}{$F^{\mathrm{trace}}$}
  & Textbook & 1.109 & Real      & 0.828 & 0.281 & \checkmark & 0.394 \\
  &        & Textbook & 1.109 & Synthetic & 0.630 & 0.478 & $\times$   & 0.026$^\dagger$ \\
  &        & Real   & 0.828 & Synthetic & 0.630 & 0.197 & \checkmark & 0.589 \\
\bottomrule[0.1em]
\end{tabular}
\end{adjustbox}
\end{table}

\paragraph{Takeaway.}
Across these CE datasets, the MSA--MRE ordering remains stable despite numerical variation. Within this task family, a source-dataset estimate can therefore provide an initial structural ranking, followed by target-dataset re-estimation or held-out validation.

\subsection{T2SQL Freedom Provenance}
\label{app:metric-freedom-full}

\begin{table}[htbp]
\centering
\caption{Provenance and use of the T2SQL Freedom estimates. Canonical pre-synthesis means determine both the converter input and the matched structural coordinates in Table~\ref{tab:pipeline-intervention-16}.}
\label{tab:t2sql-f-provenance}
\small
\begin{tabular}{lll}
\toprule[0.1em]
\textbf{Estimate} & \textbf{Value} & \textbf{Use} \\
\midrule[0.05em]
Canonical $F_{\mathrm{out}}$ & 0.686--1.071 & Converter input and structural analysis \\
Canonical $F_{\mathrm{trace}}$ & 0.147--1.249 & Complementary process-space analysis \\
\bottomrule[0.1em]
\end{tabular}
\end{table}

\section{Per-Task Full Results}
\label{app:per-task}

\subsection{Causal Estimation}

Table~\ref{tab:causal-est} reports the complete Causal Estimation comparison used in RQ1, including the scorer-controlled MSA--MRE reversal and deployment measurements.

\input{figures/fig_tab_causal-est.tex}

\subsection{Causal Discovery: Extended Metrics}
\label{app:cd-main}

Table~\ref{tab:causal-disc-full} reports F1, precision, FPR, SHD, NHD, time, and cost across the four Causal Discovery datasets. Their canonical output-space estimates span $F^{\mathrm{out}}=0.291$--$0.914$. Results are mean $\pm$ std over 3 runs.

\input{figures/fig_tab_causal-disc.tex}

\subsection{Feature Engineering}
\label{app:fe-main}

Feature Engineering operates in a high-freedom regime: the canonical output-space estimates are $F^{\mathrm{out}}=1.037$--$1.337$. Table~\ref{tab:fe} reports AUC, selected feature counts, latency, and cost. AUC differences are small, while removing the source evolutionary loop sharply improves deployment efficiency.

\input{figures/fig_tab_fe.tex}

\subsection{Text-to-SQL}
\label{app:t2sql-main}

Table~\ref{tab:text2sql-side-by-side} reports the complete execution and schema-declaration results for BIRD-147 and Spider-120, together with task-level latency and cost.

\input{figures/fig_tab_text2sql.tex}


\section{Efficiency and Cost}

\label{app:full-bar}

\subsection{Performance and Cost Pareto Analysis}
\label{app:pareto}

Figure~\ref{fig:pareto} compares performance and cache-adjusted deployment cost using the selected final AdaSkill. AdaSkill lies on the non-dominated frontier in all six panels. The T2SQL EX panel shows a trade-off: the workflow reference has 1.4 points higher mean EX and 60\% higher measured cost. The remaining panels show either performance dominance by AdaSkill or a frontier trade-off with a cheaper baseline.

\input{figures/fig_fig_pareto.tex}

\subsection{Per-Dataset Performance and Efficiency}

Figure~\ref{fig:full-bar} provides per-dataset AdaSkill performance across all 11 datasets (Row~1). Rows~2--3 report task-average efficiency for T2SQL and CE, for which the recorded rows contain aggregate time and cost, and per-dataset efficiency for CD and FE.

\begin{figure}[htbp]
\centering
\definecolor{colRaw}{HTML}{B0BEC5}
\definecolor{colComp}{HTML}{5D8AA8}
\definecolor{colMAS}{HTML}{1A365D}
\definecolor{colOurs}{HTML}{E63946}
\resizebox{\linewidth}{!}{%
\begin{tikzpicture}
\begin{groupplot}[
  group style={
    group size=4 by 3,
    xlabels at=edge bottom,
    xticklabels at=edge bottom,
    ylabels at=edge left,
    vertical sep=10mm,
    horizontal sep=7mm,
  },
  ybar=1pt,
  grid=major, grid style={gray!15},
  tick label style={font=\tiny},
  x tick label style={font=\tiny, rotate=30, anchor=east},
  y tick label style={font=\tiny},
]

\nextgroupplot[
  width=0.21\linewidth, enlarge x limits=0.45, 
  height=3cm, bar width=2.7pt,
  symbolic x coords={BIRD-147, Spider-120}, xtick=data,
  title={\scriptsize\textbf{T2SQL}},
  ylabel={\scriptsize Perf.\ (\%)},
  ymin=48, ymax=78, ytick={50,60,70},
  legend to name=fullbarlegend,
  legend style={legend columns=4, font=\tiny, draw=cgray!40, column sep=3pt},
]
  \addplot[fill=colMAS!85, draw=none] coordinates {(BIRD-147,69.4)(Spider-120,67.5)};
  \addplot[fill=colRaw!85, draw=none] coordinates {(BIRD-147,59.2)(Spider-120,60.0)};
  \addplot[fill=colComp!85, draw=none] coordinates {(BIRD-147,67.4)(Spider-120,57.5)};
  \addplot[fill=colOurs!85, draw=none] coordinates {(BIRD-147,67.4)(Spider-120,66.7)};
  \addlegendentry{Source Workflow}
  \addlegendentry{Base Agent}
  \addlegendentry{Workflow Compiler}
  \addlegendentry{AdaSkill (Ours)}

\nextgroupplot[
  width=0.25\linewidth, enlarge x limits=0.20, 
  height=3cm, bar width=2.7pt,
  symbolic x coords={Textbook, Syn, Real}, xtick=data,
  title={\scriptsize\textbf{Causal Estimation}},
  ymin=55, ymax=108, ytick={60,80,100},
]
  \addplot[fill=colMAS!85, draw=none] coordinates {(Textbook,83.3)(Syn,75.9)(Real,78.3)};
  \addplot[fill=colRaw!85, draw=none] coordinates {(Textbook,61.5)(Syn,100.0)(Real,87.1)};
  \addplot[fill=colComp!85, draw=none] coordinates {(Textbook,82.1)(Syn,100.0)(Real,90.3)};
  \addplot[fill=colOurs!85, draw=none] coordinates {(Textbook,89.7)(Syn,100.0)(Real,96.8)};

\nextgroupplot[
  width=0.31\linewidth, enlarge x limits=0.14, 
  height=3cm, bar width=2.7pt,
  symbolic x coords={Auto, DWD, Sachs, Child}, xtick=data,
  title={\scriptsize\textbf{Causal Discovery}},
  ymin=50, ymax=102, ytick={60,80,100},
]
  \addplot[fill=colMAS!85, draw=none] coordinates {(Auto,73.7)(DWD,66.3)(Sachs,55.7)(Child,65.3)};
  \addplot[fill=colRaw!85, draw=none] coordinates {(Auto,67.0)(DWD,55.8)(Sachs,86.9)(Child,80.3)};
  \addplot[fill=colComp!85, draw=none] coordinates {(Auto,69.7)(DWD,64.2)(Sachs,66.7)(Child,86.1)};
  \addplot[fill=colOurs!85, draw=none] coordinates {(Auto,65.2)(DWD,70.1)(Sachs,92.4)(Child,95.2)};

\nextgroupplot[
  width=0.21\linewidth, enlarge x limits=0.45, 
  height=3cm, bar width=2.7pt,
  symbolic x coords={Taobao, Dia}, xtick=data,
  title={\scriptsize\textbf{Feature Eng.}},
  ymin=60, ymax=88, ytick={65,70,75,80,85},
]
  \addplot[fill=colMAS!85, draw=none] coordinates {(Taobao,65.3)(Dia,81.2)};
  \addplot[fill=colRaw!85, draw=none] coordinates {(Taobao,66.0)(Dia,80.4)};
  \addplot[fill=colComp!85, draw=none] coordinates {(Taobao,65.8)(Dia,79.8)};
  \addplot[fill=colOurs!85, draw=none] coordinates {(Taobao,67.0)(Dia,81.2)};

\nextgroupplot[
  width=0.21\linewidth, enlarge x limits=0.60,
  height=3cm, bar width=2.7pt,
  symbolic x coords={L,Avg,H}, xtick={Avg}, xmin=L, xmax=H,
  ylabel={\scriptsize Cost (\$)},
  ymin=0, ymax=1.2, ytick={0,0.4,0.8,1.2},
]
  \addplot[fill=colMAS!85, draw=none] coordinates {(Avg,0.72)};
  \addplot[fill=colRaw!85, draw=none] coordinates {(Avg,0.39)};
  \addplot[fill=colComp!85, draw=none] coordinates {(Avg,0.59)};
  \addplot[fill=colOurs!85, draw=none] coordinates {(Avg,0.45)};

\nextgroupplot[
  width=0.25\linewidth, enlarge x limits=0.60,
  height=3cm, bar width=2.7pt,
  symbolic x coords={L,Avg,H}, xtick={Avg}, xmin=L, xmax=H,
  ymin=0, ymax=0.85, ytick={0,0.3,0.6},
]
  \addplot[fill=colMAS!85, draw=none] coordinates {(Avg,0.172)};
  \addplot[fill=colRaw!85, draw=none] coordinates {(Avg,0.296)};
  \addplot[fill=colComp!85, draw=none] coordinates {(Avg,0.608)};
  \addplot[fill=colOurs!85, draw=none] coordinates {(Avg,0.207)};

\nextgroupplot[
  width=0.31\linewidth, enlarge x limits=0.14,
  height=3cm, bar width=2.7pt,
  symbolic x coords={Auto, DWD, Sachs, Child}, xtick=data,
  ymin=0, ymax=7.5, ytick={0,2,4,6},
]
  \addplot[fill=colMAS!85, draw=none] coordinates {(Auto,0.69)(DWD,1.05)(Sachs,3.82)(Child,6.49)};
  \addplot[fill=colRaw!85, draw=none] coordinates {(Auto,0.38)(DWD,0.45)(Sachs,0.33)(Child,0.50)};
  \addplot[fill=colComp!85, draw=none] coordinates {(Auto,0.154)(DWD,0.203)(Sachs,0.366)(Child,0.463)};
  \addplot[fill=colOurs!85, draw=none] coordinates {(Auto,0.37)(DWD,0.49)(Sachs,0.53)(Child,0.48)};

\nextgroupplot[
  width=0.21\linewidth, enlarge x limits=0.45,
  height=3cm, bar width=2.7pt,
  symbolic x coords={Taobao, Dia}, xtick=data,
  ymin=0, ymax=20, ytick={0,5,10,15,20},
]
  \addplot[fill=colMAS!85, draw=none] coordinates {(Taobao,11.46)(Dia,15.80)};
  \addplot[fill=colRaw!85, draw=none] coordinates {(Taobao,3.75)(Dia,6.48)};
  \addplot[fill=colComp!85, draw=none] coordinates {(Taobao,4.34)(Dia,3.58)};
  \addplot[fill=colOurs!85, draw=none] coordinates {(Taobao,3.48)(Dia,5.05)};

\nextgroupplot[
  width=0.21\linewidth, enlarge x limits=0.60,
  height=3cm, bar width=2.7pt,
  symbolic x coords={L,Avg,H}, xtick={Avg}, xmin=L, xmax=H,
  ylabel={\scriptsize Latency},
  ymode=log, ymin=30, ymax=800,
  ytick={60,300,600},
  yticklabels={1\,min,5\,min,10\,min},
]
  \addplot[fill=colMAS!85, draw=none] coordinates {(Avg,361)};
  \addplot[fill=colRaw!85, draw=none] coordinates {(Avg,75)};
  \addplot[fill=colComp!85, draw=none] coordinates {(Avg,143)};
  \addplot[fill=colOurs!85, draw=none] coordinates {(Avg,101)};

\nextgroupplot[
  width=0.25\linewidth, enlarge x limits=0.60,
  height=3cm, bar width=2.7pt,
  symbolic x coords={L,Avg,H}, xtick={Avg}, xmin=L, xmax=H,
  ymode=log, ymin=20, ymax=250,
  ytick={30,60,120},
  yticklabels={30\,s,1\,min,2\,min},
]
  \addplot[fill=colMAS!85, draw=none] coordinates {(Avg,123)};
  \addplot[fill=colRaw!85, draw=none] coordinates {(Avg,34)};
  \addplot[fill=colComp!85, draw=none] coordinates {(Avg,115)};
  \addplot[fill=colOurs!85, draw=none] coordinates {(Avg,52)};

\nextgroupplot[
  width=0.31\linewidth, enlarge x limits=0.14,
  height=3cm, bar width=2.7pt,
  symbolic x coords={Auto, DWD, Sachs, Child}, xtick=data,
  ymode=log, ymin=30, ymax=8000,
  ytick={60,600,3600},
  yticklabels={1\,min,10\,min,1\,hr},
]
  \addplot[fill=colMAS!85, draw=none] coordinates {(Auto,538)(DWD,749)(Sachs,2747)(Child,3774)};
  \addplot[fill=colRaw!85, draw=none] coordinates {(Auto,165.7)(DWD,203)(Sachs,153.6)(Child,201.8)};
  \addplot[fill=colComp!85, draw=none] coordinates {(Auto,52.3)(DWD,96.8)(Sachs,198.0)(Child,328.9)};
  \addplot[fill=colOurs!85, draw=none] coordinates {(Auto,112.0)(DWD,169.3)(Sachs,301.8)(Child,170.6)};

\nextgroupplot[
  width=0.21\linewidth, enlarge x limits=0.45,
  height=3cm, bar width=2.7pt,
  symbolic x coords={Taobao, Dia}, xtick=data,
  ymode=log, ymin=1000, ymax=100000,
  ytick={3600,36000},
  yticklabels={1\,hr,10\,hr},
]
  \addplot[fill=colMAS!85, draw=none] coordinates {(Taobao,32400)(Dia,43200)};
  \addplot[fill=colRaw!85, draw=none] coordinates {(Taobao,2235)(Dia,3042)};
  \addplot[fill=colComp!85, draw=none] coordinates {(Taobao,2682)(Dia,2466)};
  \addplot[fill=colOurs!85, draw=none] coordinates {(Taobao,1972)(Dia,2301)};

\end{groupplot}
\node[above, inner sep=2pt]
  at ($(group c1r1.north west)!0.5!(group c4r1.north east)+(0,10mm)$)
  {\pgfplotslegendfromname{fullbarlegend}};
\end{tikzpicture}%
}
\caption{\textbf{Final performance and efficiency breakdown.} \textbf{Row 1}: selected final-skill performance per dataset (EX\% / MSA\% / F1${\times}100$ / AUC${\times}100$). \textbf{Rows 2--3}: selected-final efficiency; T2SQL/CE show the task averages reported in Tables~\ref{tab:text2sql-side-by-side} and~\ref{tab:causal-est}, while CD/FE show per-dataset latency and cache-adjusted cost. Under our accounting protocol, the evaluated source workflows receive little cache credit because staged and role-conditioned calls change prompts across execution. ${\dagger}$FELA cost is estimated: Taobao \$7.5--\$15, Dia \$10--\$22 per run.}
\label{fig:full-bar}
\end{figure}


\input{sections/app_case_study}

%% file: figures/fig_tab_output-dist.tex
\begin{table}[htbp]
\centering
\caption{Output dissimilarities used for $F_{\text{out}}$. For $F_{\text{trace}}$, process-trace embeddings with final responses removed use scaled cosine dissimilarity uniformly across tasks.}
\label{tab:output-dist}
\small
\begin{adjustbox}{max width=\textwidth}
\begin{tabular}{llll}
\toprule[0.1em]
\textbf{Task} & \textbf{Output} & $d(o_i,o_j)$ for $F_{\text{out}}$ & \textbf{Aggregation} \\
\midrule[0.05em]
Causal Discovery & Edge set $E$ & $1 - \text{Jaccard}(E_i,E_j)$ & Per-dataset, $F$ over all $N$ runs \\
Causal Estimation & Selected method & $\mathbf{1}[m_i \neq m_j]$; hierarchical execution audit & Per-question mean over identifiable questions \\
Feature Engineering & Feature set $\mathcal{F}$ & $1 - \text{Jaccard}(\mathcal{F}_i,\mathcal{F}_j)$ & Per-dataset, $F$ over all $N$ runs \\
Text-to-SQL & Predicted SQL & $1 - \text{Jaccard}(\text{tok}_i,\text{tok}_j)$ & Per-question mean over identifiable questions \\
\bottomrule[0.1em]
\end{tabular}
\end{adjustbox}

\end{table}

%% file: figures/fig_tab_ce-structured-behavior.tex
\begin{table}[htbp]
\centering
\caption{\textbf{CE execution-signature audit.} The canonical selected-estimator distance is enriched hierarchically with response-free execution configuration parsed from retained tool calls. Different estimator classes have distance one; within-class distance is one half the Jaccard distance over statistical packages, estimator calls, dataframe columns, causal-design tokens, keyword arguments, and whitelisted numeric configurations. Tool results, final responses, and estimated effects are excluded. $K$ is the number of identifiable diagnostic questions.}
\label{tab:ce-structured-behavior}
\small
\setlength{\tabcolsep}{5.5pt}
\begin{tabular}{llrrr}
\toprule[0.1em]
\textbf{Dataset} & \textbf{Metric} & $K$ & $F_{\mathrm{method}}$ & $F_{\mathrm{exec}}$ \\
\midrule[0.05em]
Textbook & MSA & 6 & 0.098 & 0.181 \\
Textbook & MRE & 6 & 1.012 & 1.085 \\
Real & MSA & 5 & 0.236 & 0.255 \\
Real & MRE & 6 & 0.697 & 0.762 \\
Synthetic & MRE & 6 & 0.638 & 0.651 \\
\bottomrule[0.1em]
\end{tabular}
\end{table}

%% file: figures/fig_tab_structural-ranking.tex
\begin{table}[htbp]
\centering
\caption{\textbf{Threshold-free ranking of structural value.} The positive class is Full $>$ Discard. Scores use $-F$; the exact one-sided test enumerates all label assignments conditional on class count.}
\label{tab:structural-ranking}
\small
\begin{tabular}{lrrrrr}
\toprule[0.1em]
Audit & $n$ & Full $>$ Discard & Discard $>$ Full & ROC AUC & Exact $p$ \\
\midrule[0.05em]
Route-grade & 13 & 5 & 8 & \textbf{1.000} & \textbf{0.0008} \\
Complete & 16 & 5 & 11 & \textbf{1.000} & \textbf{0.0002} \\
\bottomrule[0.1em]
\end{tabular}
\end{table}

%% file: figures/fig_tab_pipeline-intervention-16.tex
\begin{table}[htbp]
\centering
\caption{\textbf{Complete matched pipeline interventions } Base, Full, and Discard are higher-is-better utilities; MRE is converted as $1-\mathrm{MRE}$. The final column is the normalized Full-minus-Discard treatment effect. Grey rows have fewer than three identifiable questions.}
\label{tab:pipeline-intervention-16}
\scriptsize
\setlength{\tabcolsep}{3.1pt}
\begin{tabular}{lllrrrrrr}
\toprule[0.1em]
\textbf{Task} & \textbf{Dataset} & \textbf{Metric} & \textbf{$K$} & \textbf{$F$} & \textbf{Base} & \textbf{Full} & \textbf{Discard} & \textbf{$\tau_{\rm pipe}$} \\
\midrule[0.05em]
CE & Textbook & MSA & 6 & .098 & .615 & .846 & .771 & +19.6 \\
CE & Real & MSA & 5 & .235 & .871 & .903 & .871 & +25.0 \\
CE & Textbook & MRE & 6 & 1.012 & .688 & .737 & .769 & -10.3 \\
CE & Real & MRE & 5 & .697 & .817 & .827 & .831 & -2.1 \\
CE & Synthetic & MRE & 6 & .638 & .894 & .924 & .951 & -26.0 \\
CD & AutoMPG & F1 & -- & .914 & .670 & .652 & .667 & -4.5 \\
CD & Child & F1 & -- & .291 & .803 & .952 & .918 & +17.3 \\
CD & DWDClimate & F1 & -- & .605 & .558 & .701 & .663 & +8.6 \\
CD & Sachs & F1 & -- & .598 & .869 & .924 & .893 & +23.7 \\
T2SQL & BIRD-147 & EX & 4 & .686 & .592 & .660 & .673 & -3.3 \\
T2SQL & BIRD-147 & NSF & 6 & .816 & .905 & .925 & .934 & -9.3 \\
\rowcolor{gray!12} T2SQL & BIRD-147 & SRR & 2 & .697 & .718 & .782 & .791 & -3.1 \\
FE & Taobao & AUC & -- & 1.037 & .660 & .657 & .670 & -3.8 \\
\rowcolor{gray!12} T2SQL & Spider-120 & EX & 2 & 1.071 & .600 & .642 & .665 & -5.8 \\
\rowcolor{gray!12} T2SQL & Spider-120 & NSF & 1 & .787 & .114 & .585 & .597 & -1.4 \\
FE & Dia & AUC & -- & 1.337 & .804 & .798 & .812 & -7.0 \\
\bottomrule[0.1em]
\end{tabular}
\end{table}

%% file: figures/fig_tab_f-economics.tex
\begin{table}[htbp]
\centering
\caption{\textbf{Dataset-level audit of cross-fitted routing and the full-information structural oracle.} Base, Full, and Discard are matched candidates. The LOFO policy output is the route-selected candidate for single-metric CD/FE (Selected), so equality with Full or Discard is exact by construction; for multi-metric CE/T2SQL it is one directly executed skill jointly composed from the metric-level route vector (Composed). The final column is a descriptive in-sample oracle that chooses the better observed Full/Discard candidate after target outcomes are revealed; it is not an AdaSkill execution. Values are native percentages; MRE is lower-is-better.}
\label{tab:f-economics}
\scriptsize
\setlength{\tabcolsep}{2.1pt}
\begin{adjustbox}{max width=\textwidth}
\begin{tabular}{lllrrrclrr}
\toprule[0.1em]
\textbf{Family} & \textbf{Dataset} & \textbf{Metric} & \textbf{Base} & \textbf{Full} & \textbf{Discard} & \textbf{Route} & \textbf{Output type} & \textbf{LOFO policy} & \textbf{In-sample oracle} \\
\midrule[0.05em]
CE & Textbook & MSA & 61.54 & 84.62 & 77.10 & $P$ & Composed & 89.70 & Full (84.62) \\
CE & Real & MSA & 87.10 & 90.32 & 87.10 & $P$ & Composed & 96.80 & Full (90.32) \\
CE & Textbook & MRE & 31.25 & 26.27 & 23.06 & $D$ & Composed & 23.10 & Discard (23.06) \\
CE & Real & MRE & 18.33 & 17.33 & 16.95 & $D$ & Composed & 15.90 & Discard (16.95) \\
CE & Synthetic & MRE & 10.56 & 7.63 & 4.88 & $D$ & Composed & 4.30 & Discard (4.88) \\
CD & AutoMPG & F1 & 67.00 & 65.20 & 66.67 & $D$ & Selected & 66.67 & Discard (66.67) \\
CD & Child & F1 & 80.30 & 95.20 & 91.80 & $D$ & Selected & 91.80 & Full (95.20) \\
CD & DWDClimate & F1 & 55.80 & 70.10 & 66.30 & $D$ & Selected & 66.30 & Full (70.10) \\
CD & Sachs & F1 & 86.90 & 92.40 & 89.30 & $D$ & Selected & 89.30 & Full (92.40) \\
T2SQL & BIRD-147 & EX & 59.20 & 65.99 & 67.35 & $P$ & Composed & 66.93 & Discard (67.35) \\
T2SQL & BIRD-147 & NSF & 90.50 & 92.52 & 93.40 & $D$ & Composed & 93.80 & Discard (93.40) \\
T2SQL & BIRD-147 & SRR & 71.80 & 78.23 & 79.10 & $D$ & Composed & 80.20 & Discard (79.10) \\
FE & Taobao & AUC & 66.02 & 65.70 & 67.00 & $D$ & Selected & 67.00 & Discard (67.00) \\
T2SQL & Spider-120 & EX & 60.00 & 64.20 & 66.50 & $P$ & Composed & 66.10 & Discard (66.50) \\
T2SQL & Spider-120 & NSF & 11.40 & 58.50 & 59.70 & $D$ & Composed & 60.20 & Discard (59.70) \\
FE & Dia & AUC & 80.40 & 79.82 & 81.20 & $D$ & Selected & 81.20 & Discard (81.20) \\
\bottomrule[0.1em]
\end{tabular}
\end{adjustbox}
\end{table}

%% file: figures/fig_fig_metric-freedom-scatter_openai.tex
\begin{figure}[htbp]
\centering
\includegraphics[width=.96\linewidth]{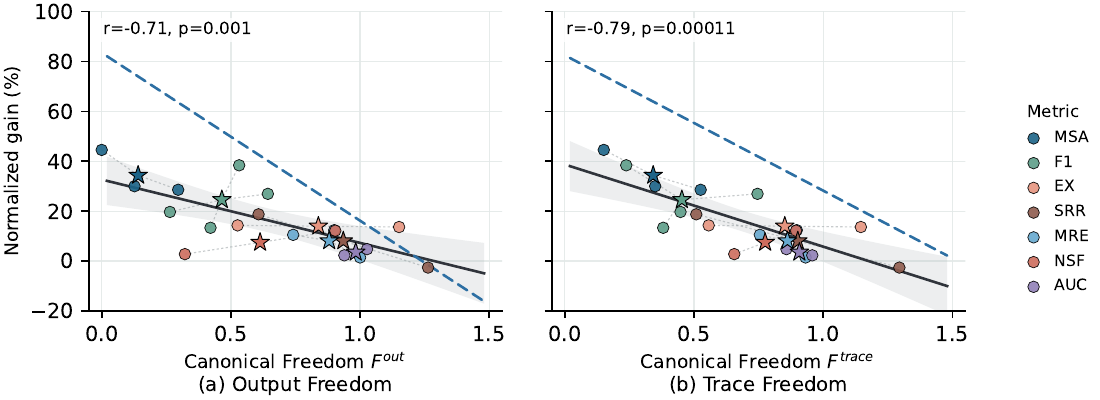}
\caption{\textbf{The pre-skill coupling signal persists under GPT-5.1.} Canonical output-space \textbf{(a)} and response-free trace-space \textbf{(b)} estimates preserve the negative ordering with AdaSkill gain ($r=-0.71$ and $r=-0.79$) across all 18 dataset--metric tuples. Both Freedom and gain are re-estimated under GPT-5.1 using the same evaluation protocol as Figure~\ref{fig:metric-freedom-scatter}. Black solid lines fit GPT-5.1; blue dashed lines transfer the Sonnet~4.6 fits from Figure~\ref{fig:metric-freedom-scatter} without refitting. Circles are individual datasets; stars are metric means.}
\label{fig:metric-freedom-scatter-gpt}
\end{figure}

%% file: figures/fig_tab_gain-intervals.tex
\begin{table}[htbp]
\centering
\caption{\textbf{AdaSkill gains, aggregate robustness, and efficiency.} Positive values favor AdaSkill. Brackets give 95\% intervals; bold intervals exclude zero. MRE gains are baseline minus AdaSkill, with relative reductions in parentheses. All CE comparisons and T2SQL comparisons with Base and Compiler use 200,000 paired question-bootstrap samples. CE Overall MRE uses a dataset-stratified paired bootstrap with equal dataset weights. Ranges in $n$ denote pair-specific complete-case counts across baseline comparisons. Acc@$\tau$ is the fraction of Synthetic questions with relative error at most $\tau$. Efficiency entries are signed savings relative to each baseline: positive values are reductions and negative values are increases. T2SQL EvoSkill entries marked $\dagger$ use Newcombe intervals from final counts.}
\label{tab:gain-intervals}
\scriptsize
\setlength{\tabcolsep}{3.0pt}
\renewcommand{\arraystretch}{0.95}
\begin{adjustbox}{max width=\textwidth}
\begin{tabular}{llrrrr}
\toprule[0.1em]
\textbf{Task} & \textbf{Dataset and metric} & $\boldsymbol{n}$ & \textbf{vs. Base} & \textbf{vs. Compiler} & \textbf{vs. EvoSkill} \\
\midrule[0.05em]
\multicolumn{6}{l}{\textit{Panel A: Dataset-level performance gains (percentage points)}} \\
CE & Textbook MSA & 39 & $\mathbf{28.2\ [15.4,43.6]}$ & $7.7\ [-2.6,18.0]$ & $\mathbf{17.9\ [2.5,33.3]}$ \\
CE & Textbook MRE & 38--39 & $8.1\ [-0.8,18.0]\ (26.0\%)$ & $3.2\ [-4.4,11.7]\ (12.2\%)$ & $2.3\ [-4.6,10.1]\ (9.1\%)$ \\
CE & Synthetic MSA & 45 & $0.0\ [0.0,0.0]$ & $0.0\ [0.0,0.0]$ & $0.0\ [0.0,0.0]$ \\
CE & Synthetic MRE & 45 & $\mathbf{6.3\ [1.1,13.2]}\ (59.4\%)$ & $\mathbf{8.6\ [1.9,17.0]}\ (66.7\%)$ & $2.8\ [-1.0,7.8]\ (39.4\%)$ \\
CE & Real MSA & 31 & $9.7\ [-3.2,22.6]$ & $6.5\ [-6.4,19.4]$ & $3.2\ [-6.5,12.9]$ \\
CE & Real MRE & 31 & $2.4\ [-6.3,11.4]\ (13.1\%)$ & $1.4\ [-12.1,15.9]\ (8.1\%)$ & $5.7\ [-5.2,17.6]\ (26.4\%)$ \\
T2SQL & BIRD-147 EX & 147 & $\mathbf{8.2\ [3.4,13.0]}$ & $0.0\ [-4.1,3.4]$ & $1.4\ [-9.3,12.0]^{\dagger}$ \\
T2SQL & Spider-120 EX & 120 & $6.7\ [-1.6,15.0]$ & $\mathbf{9.2\ [2.5,15.9]}$ & $4.2\ [-7.9,16.0]^{\dagger}$ \\
\midrule[0.05em]
\multicolumn{6}{l}{\textit{Panel B: CE aggregate and threshold robustness}} \\
CE & Overall MRE & 3 datasets & $\mathbf{5.6\ [1.0,10.5]}\ (28.0\%)$ & $4.4\ [-1.4,10.5]\ (23.4\%)$ & $3.6\ [-1.0,8.5]\ (20.0\%)$ \\
CE & Synthetic Acc@10\% & 45 & $\mathbf{13.3\ [4.4,24.4]}$ & $\mathbf{11.1\ [2.2,20.0]}$ & $2.2\ [-6.7,11.1]$ \\
CE & Synthetic Acc@20\% & 45 & $\mathbf{11.1\ [2.2,20.0]}$ & $\mathbf{11.1\ [2.2,20.0]}$ & $6.7\ [-2.2,15.6]$ \\
CE & Synthetic Acc@50\% & 45 & $4.4\ [0.0,11.1]$ & $\mathbf{8.9\ [2.2,17.8]}$ & $4.4\ [0.0,11.1]$ \\
\midrule[0.05em]
\multicolumn{6}{l}{\textit{Panel C: CE efficiency and measured Pareto dominance}} \\
CE & Performance cells W/T/L & 6 & $5/1/0$ & $5/1/0$ & $5/1/0$ \\
CE & Inference-time saving & -- & $-52.9\%$ & $54.8\%$ & $21.2\%$ \\
CE & Inference-cost saving & -- & $30.1\%$ & $66.0\%$ & $23.3\%$ \\
CE & Pareto dominant in performance and cost & -- & $\checkmark$ & $\checkmark$ & $\checkmark$ \\
\bottomrule[0.1em]
\end{tabular}
\end{adjustbox}
\end{table}

%% file: figures/fig_tab_causal-est.tex
\begin{table}[htbp]
\setlength{\aboverulesep}{0.0ex}\setlength{\belowrulesep}{0.0ex}
\centering
\caption{\textbf{Causal Estimation results.} Method choice is tightly coupled to MSA ($F{\approx}0.2$) and weakly coupled to MRE ($F{\approx}0.8$). The pipeline-faithful Workflow Compiler gains strongly on Textbook MSA but regresses on Synthetic MRE, while AdaSkill records the highest MSA and lowest MRE on all three datasets.}
\label{tab:causal-est}
\small
\setlength{\tabcolsep}{3.5pt} 
\begin{adjustbox}{max width=\textwidth}
\begin{tabular}{l cc cc cc cc}
\toprule[0.1em]
\textbf{Dataset$\rightarrow$}& \multicolumn{2}{c}{\textbf{Textbook}} & \multicolumn{2}{c}{\textbf{Synthetic}} & \multicolumn{2}{c}{\textbf{Real}} & \multicolumn{2}{c}{\textbf{Efficiency}} \\
\cmidrule[0.05em](lr){2-3}\cmidrule[0.05em](lr){4-5}\cmidrule[0.05em](lr){6-7}\cmidrule[0.05em](lr){8-9}
\textbf{Method$\downarrow$} & MSA$\uparrow$ & MRE$\downarrow$ & MSA$\uparrow$ & MRE$\downarrow$ & MSA$\uparrow$ & MRE$\downarrow$ & Time (s)$\downarrow$ & Cost (\$)$\downarrow$ \\
\midrule[0.05em]
CAIS (workflow ref.) & 83.3 & 54.0 & 75.9 & 16.2 & 78.3 & 32.0 & 123 & 0.172 \\
\midrule[0.05em]
Base Agent & 61.5 & 31.2 & \textbf{100} & 10.6 & 87.1 & 18.3 & 34 & 0.296 \\
EvoSkill & 71.8 & 25.4 & \textbf{100} & 7.1 & 93.5 & 21.6 & 66 & 0.270 \\
Workflow Compiler & 82.1 & 26.3 & \textbf{100} & 12.9 & 90.3 & 17.3 & 115 & 0.608 \\
\midrule[0.05em]
\textit{Tools Only} & 69.2 & 28.3 & 95.6 & 5.5 & 87.1 & 17.0 & 49 & 0.208 \\
\textit{Knowledge Only} & 79.5 & 25.6 & 91.1 & 6.0 & 90.3 & 18.0 & 55 & 0.305 \\
\textit{Pipeline Only} & 84.6 & 26.5 & \textbf{100} & 5.5 & \textbf{96.8} & 17.8 & 56 & 0.295 \\
\rowcolor{lightgreen}
\textbf{AdaSkill (Ours)} & \textbf{89.7} & \textbf{23.1} & \textbf{100} & \textbf{4.3} & \textbf{96.8} & \textbf{15.9} & 52 & 0.207 \\
\bottomrule[0.1em]
\end{tabular}
\end{adjustbox}
\smallskip
\end{table}

%% file: figures/fig_tab_causal-disc.tex
\begin{table}[htbp]
\centering
\caption{\textbf{Causal Discovery full metrics (mean$\,\pm\,$std, 3 runs).} P = Precision, F1 = F1 Score, FPR = False Positive Rate, SHD = Structural Hamming Distance, NHD = Normalized Hamming Distance.}
\label{tab:causal-disc-full}
\small
\setlength{\tabcolsep}{1pt} 
\begin{adjustbox}{max width=\textwidth}
\begin{tabular}{l ccccccc ccccccc}
\toprule[0.1em]

\textbf{Dataset$\rightarrow$}& \multicolumn{7}{c}{\textbf{AutoMPG}} & \multicolumn{7}{c}{\textbf{DWDClimate}} \\
\cmidrule[0.05em](lr){2-8} \cmidrule[0.05em](lr){9-15}
\textbf{Method$\downarrow$} & P$\uparrow$ & F1$\uparrow$ & FPR$\downarrow$ & SHD$\downarrow$ & NHD$\downarrow$ & Time$\downarrow$ & Cost$\downarrow$ & P$\uparrow$ & F1$\uparrow$ & FPR$\downarrow$ & SHD$\downarrow$ & NHD$\downarrow$ & Time$\downarrow$ & Cost$\downarrow$ \\
\midrule[0.05em]
MATMCD (workflow ref.) & .71{\tiny$\pm$.06} & .74{\tiny$\pm$.02} & .12{\tiny$\pm$.01} & 2.3{\tiny$\pm$1.3} & .10{\tiny$\pm$.01} & 538.0{\tiny$\pm$42.5} & .69{\tiny$\pm$.12} & .60{\tiny$\pm$.02} & .66{\tiny$\pm$.06} & .10{\tiny$\pm$.01} & 5.0{\tiny$\pm$1.0} & .15{\tiny$\pm$.01} & 749.0{\tiny$\pm$65.3} & 1.05{\tiny$\pm$.15} \\
\midrule[0.05em]
Base Agent & .58{\tiny$\pm$.07} & .67{\tiny$\pm$.05} & .20{\tiny$\pm$.05} & 3.0{\tiny$\pm$.8} & .15{\tiny$\pm$.04} & 165.7{\tiny$\pm$35.9} & 2.00{\tiny$\pm$.54} & .55{\tiny$\pm$.05} & .56{\tiny$\pm$.01} & .13{\tiny$\pm$.04} & 5.5{\tiny$\pm$.5} & .18{\tiny$\pm$.02} & 203.0{\tiny$\pm$70.5} & 2.42{\tiny$\pm$.26} \\
EvoSkill & .62{\tiny$\pm$.05} & .70{\tiny$\pm$.03} & .17{\tiny$\pm$.03} & 2.5{\tiny$\pm$.5} & .13{\tiny$\pm$.03} & 193.7{\tiny$\pm$44.2} & 2.03{\tiny$\pm$.63} & .67{\tiny$\pm$.05} & .74{\tiny$\pm$.03} & .10{\tiny$\pm$.02} & 3.5{\tiny$\pm$.5} & .12{\tiny$\pm$.02} & 205.0{\tiny$\pm$33.8} & 2.48{\tiny$\pm$.28} \\
Workflow Compiler & .62{\tiny$\pm$.07} & .70{\tiny$\pm$.03} & .17{\tiny$\pm$.03} & 2.5{\tiny$\pm$.5} & .13{\tiny$\pm$.03} & 52.3{\tiny$\pm$8.0} & .82{\tiny$\pm$.08} & .59{\tiny$\pm$.03} & .64{\tiny$\pm$.08} & .12{\tiny$\pm$.03} & 4.7{\tiny$\pm$.6} & .16{\tiny$\pm$.02} & 96.8{\tiny$\pm$14.7} & 1.08{\tiny$\pm$.02} \\
\midrule[0.05em]
\textit{Tools Only} & .52{\tiny$\pm$.03} & .65{\tiny$\pm$.02} & .24{\tiny$\pm$.03} & 3.7{\tiny$\pm$.5} & .18{\tiny$\pm$.02} & 112.2{\tiny$\pm$20.2} & .43{\tiny$\pm$.09} & .52{\tiny$\pm$.03} & .54{\tiny$\pm$.05} & .13{\tiny$\pm$.01} & 5.7{\tiny$\pm$.5} & .19{\tiny$\pm$.02} & 123.1{\tiny$\pm$64.9} & .45{\tiny$\pm$.11} \\
\textit{Knowledge Only} & .57{\tiny$\pm$.03} & .67{\tiny$\pm$.02} & .20{\tiny$\pm$.02} & 3.0{\tiny$\pm$.4} & .15{\tiny$\pm$.02} & 22.9{\tiny$\pm$1.8} & .47{\tiny$\pm$.08} & .61{\tiny$\pm$.05} & .74{\tiny$\pm$.06} & .15{\tiny$\pm$.02} & 4.0{\tiny$\pm$.8} & .13{\tiny$\pm$.03} & 30.2{\tiny$\pm$2.7} & .49{\tiny$\pm$.09} \\
\textit{Pipeline Only} & .57{\tiny$\pm$.02} & .67{\tiny$\pm$.03} & .20{\tiny$\pm$.01} & 3.0{\tiny$\pm$.5} & .15{\tiny$\pm$.01} & 22.6{\tiny$\pm$1.0} & .44{\tiny$\pm$.07} & .50{\tiny$\pm$.07} & .61{\tiny$\pm$.10} & .19{\tiny$\pm$.02} & 6.0{\tiny$\pm$1.4} & .20{\tiny$\pm$.05} & 32.0{\tiny$\pm$5.7} & .42{\tiny$\pm$.08} \\
\rowcolor{lightgreen}
\textbf{AdaSkill (Ours)} & .56{\tiny$\pm$.08} & .65{\tiny$\pm$.05} & .22{\tiny$\pm$.06} & 3.3{\tiny$\pm$.9} & .17{\tiny$\pm$.05} & 112.0{\tiny$\pm$1.2} & 1.99{\tiny$\pm$.31} & .65{\tiny$\pm$.07} & .70{\tiny$\pm$.05} & .11{\tiny$\pm$.04} & 4.0{\tiny$\pm$.8} & .13{\tiny$\pm$.03} & 169.3{\tiny$\pm$51.8} & 2.62{\tiny$\pm$1.57} \\
\midrule[0.05em]

\textbf{Dataset$\rightarrow$}& \multicolumn{7}{c}{\textbf{Sachs}} & \multicolumn{7}{c}{\textbf{Child}} \\
\cmidrule[0.05em](lr){2-8} \cmidrule[0.05em](lr){9-15}
\textbf{Method$\downarrow$} & P$\uparrow$ & F1$\uparrow$ & FPR$\downarrow$ & SHD$\downarrow$ & NHD$\downarrow$ & Time$\downarrow$ & Cost$\downarrow$ & P$\uparrow$ & F1$\uparrow$ & FPR$\downarrow$ & SHD$\downarrow$ & NHD$\downarrow$ & Time$\downarrow$ & Cost$\downarrow$ \\
\midrule[0.05em]
MATMCD (workflow ref.) & .58{\tiny$\pm$.10} & .56{\tiny$\pm$.11} & .06{\tiny$\pm$.01} & 13.3{\tiny$\pm$10.3} & .11{\tiny$\pm$.01} & 2747.0{\tiny$\pm$185.5} & 3.82{\tiny$\pm$.45} & .65{\tiny$\pm$.06} & .65{\tiny$\pm$.05} & .02{\tiny$\pm$.01} & 14.3{\tiny$\pm$8.3} & .03{\tiny$\pm$.01} & 3774.0{\tiny$\pm$256.2} & 6.49{\tiny$\pm$.72} \\
\midrule[0.05em]
Base Agent & .86{\tiny$\pm$.03} & .87{\tiny$\pm$.05} & .03{\tiny$\pm$.01} & 4.0{\tiny$\pm$1.0} & .04{\tiny$\pm$.01} & 153.6{\tiny$\pm$2.9} & 1.78{\tiny$\pm$.53} & .77{\tiny$\pm$.17} & .80{\tiny$\pm$.17} & .02{\tiny$\pm$.01} & 9.3{\tiny$\pm$7.4} & .03{\tiny$\pm$.02} & 201.8{\tiny$\pm$51.4} & 2.66{\tiny$\pm$.52} \\
EvoSkill & .86{\tiny$\pm$.08} & .87{\tiny$\pm$.08} & .03{\tiny$\pm$.02} & 3.5{\tiny$\pm$1.7} & .03{\tiny$\pm$.02} & 154.0{\tiny$\pm$39.9} & 1.39{\tiny$\pm$.13} & .88{\tiny$\pm$.13} & .90{\tiny$\pm$.10} & .01{\tiny$\pm$.01} & 5.0{\tiny$\pm$5.0} & .01{\tiny$\pm$.01} & 457.2{\tiny$\pm$133.2} & 4.83{\tiny$\pm$.72} \\
Workflow Compiler & .75{\tiny$\pm$.11} & .67{\tiny$\pm$.26} & .03{\tiny$\pm$.01} & 7.7{\tiny$\pm$4.6} & .07{\tiny$\pm$.04} & 198.0{\tiny$\pm$109.1} & 1.95{\tiny$\pm$1.11} & .86{\tiny$\pm$.05} & .86{\tiny$\pm$.05} & .01{\tiny$\pm$.01} & 6.7{\tiny$\pm$3.5} & .02{\tiny$\pm$.01} & 328.9{\tiny$\pm$131.5} & 2.48{\tiny$\pm$.65} \\
\midrule[0.05em]
\textit{Tools Only} & .86{\tiny$\pm$.08} & .89{\tiny$\pm$.08} & .03{\tiny$\pm$.02} & 3.5{\tiny$\pm$1.7} & .03{\tiny$\pm$.02} & 92.9{\tiny$\pm$62.9} & .29{\tiny$\pm$.05} & .90{\tiny$\pm$.05} & .85{\tiny$\pm$.03} & .01{\tiny$\pm$.01} & 5.7{\tiny$\pm$.9} & .01{\tiny$\pm$.01} & 115.6{\tiny$\pm$23.2} & .54{\tiny$\pm$.15} \\
\textit{Knowledge Only} & .88{\tiny$\pm$.05} & .89{\tiny$\pm$.05} & .02{\tiny$\pm$.01} & 3.3{\tiny$\pm$1.3} & .03{\tiny$\pm$.01} & 54.9{\tiny$\pm$9.4} & .41{\tiny$\pm$.07} & .81{\tiny$\pm$.08} & .79{\tiny$\pm$.07} & .01{\tiny$\pm$.01} & 9.3{\tiny$\pm$3.3} & .03{\tiny$\pm$.01} & 161.4{\tiny$\pm$103.2} & .54{\tiny$\pm$.11} \\
\textit{Pipeline Only} & .94{\tiny$\pm$.05} & .95{\tiny$\pm$.05} & .01{\tiny$\pm$.01} & 1.7{\tiny$\pm$1.7} & .02{\tiny$\pm$.02} & 75.9{\tiny$\pm$65.1} & .45{\tiny$\pm$.08} & .79{\tiny$\pm$.05} & .81{\tiny$\pm$.04} & .02{\tiny$\pm$.01} & 9.0{\tiny$\pm$2.2} & .02{\tiny$\pm$.01} & 55.0{\tiny$\pm$18.2} & .77{\tiny$\pm$.14} \\
\rowcolor{lightgreen}
\textbf{AdaSkill (Ours)} & .91{\tiny$\pm$.07} & .92{\tiny$\pm$.06} & .02{\tiny$\pm$.01} & 2.3{\tiny$\pm$1.7} & .02{\tiny$\pm$.02} & 301.8{\tiny$\pm$212.5} & 2.82{\tiny$\pm$1.11} & .96{\tiny$\pm$.06} & .95{\tiny$\pm$.07} & .00{\tiny$\pm$.00} & 1.7{\tiny$\pm$2.4} & .00{\tiny$\pm$.01} & 170.6{\tiny$\pm$12.7} & 2.55{\tiny$\pm$.62} \\
\bottomrule[0.1em]
\end{tabular}
\end{adjustbox}
\end{table}

%% file: figures/fig_tab_fe.tex
\begin{table}[htbp]
\centering
\caption{\textbf{Feature Engineering results (AUC, mean$\,\pm\,$std).} In these high-$F$ settings, AdaSkill achieves the best mean AUC on Taobao and matches the source FELA workflow on Dia at substantially lower reported cost. The FELA workflow cost is estimated from its recorded operating range. High-$F$ feature engineering favors minimal structure.}
\label{tab:fe}
\small
\begin{adjustbox}{max width=\textwidth}
\begin{tabular}{l ccccc ccccc}
\toprule[0.1em]
& \multicolumn{5}{c}{\textbf{Taobao}} & \multicolumn{5}{c}{\textbf{Dia}} \\
\cmidrule[0.05em](lr){2-6}\cmidrule[0.05em](lr){7-11}
\textbf{Method} & AUC & Max & \#F & Time (s) & Cost (\$) & AUC & Max & \#F & Time (s) & Cost (\$) \\
\midrule[0.05em]
FELA (workflow ref.)$^\dagger$ & .653 & .653 & N/A & 32400 & 11.46 & .812 & .812 & N/A & 43200 & 15.80 \\
\midrule[0.05em]
Base Agent & .660{\tiny$\pm$.005} & .667 & 59 & 2235 & 3.75 & .804{\tiny$\pm$.006} & .811 & 76 & 3042 & 6.48 \\
FELA Skill (single agent) & .663{\tiny$\pm$.003} & .666 & 46 & N/A & 8.03 & \textbf{.813}{\tiny$\pm$.009} & \textbf{.820} & 33 & N/A & 7.44{\tiny$\pm$1.52} \\
w/ Workflow Compiler & .658{\tiny$\pm$.008} & .666 & 48 & 2682 & 4.34 & .798{\tiny$\pm$.002} & .800 & 18 & 2466 & 3.58{\tiny$\pm$2.18} \\
\rowcolor{lightgreen}
\textbf{AdaSkill (Ours)} & \textbf{.670}{\tiny$\pm$.002} & \textbf{.672} & 46 & 1972 & 3.48{\tiny$\pm$1.72} & .812{\tiny$\pm$.001} & .813 & 55 & 2301 & 5.05{\tiny$\pm$1.57} \\
\bottomrule[0.1em]
\end{tabular}
\end{adjustbox}
\end{table}

%% file: figures/fig_tab_text2sql.tex
\begin{table}[htbp]
\setlength{\aboverulesep}{0.0ex}\setlength{\belowrulesep}{0.0ex}
\centering
\caption{\textbf{Text-to-SQL results.} SRR, NSR, NSP, and NSF score the explicit \texttt{selected\_columns} declaration; EX independently scores the executed result. Their divergence on Spider records incomplete declarations alongside executable SQL. AdaSkill improves seven of eight declaration metrics over the Base Agent and raises EX on both datasets.}
\label{tab:text2sql-side-by-side}
\begin{small}
\setlength{\tabcolsep}{3pt} 
\begin{adjustbox}{max width=\textwidth}
\begin{tabular}{l ccccc ccccc cc}
\toprule[0.1em]
\textbf{Dataset$\rightarrow$}& \multicolumn{5}{c}{\textbf{BIRD-147}} & \multicolumn{5}{c}{\textbf{Spider-120}} & \multicolumn{2}{c}{\textbf{Efficiency}} \\
\cmidrule[0.05em](lr){2-6} \cmidrule[0.05em](lr){7-11} \cmidrule[0.05em](lr){12-13}
\textbf{Method$\downarrow$} & SRR$\uparrow$ & NSR$\uparrow$ & NSP$\uparrow$ & NSF$\uparrow$ & EX$\uparrow$ & SRR$\uparrow$ & NSR$\uparrow$ & NSP$\uparrow$ & NSF$\uparrow$ & EX$\uparrow$ & Time$\downarrow$ & Cost$\downarrow$ \\
\midrule[0.05em]
APEX (workflow ref.) & \textbf{93.9} & \textbf{98.3} & 47.0 & 60.8 & \textbf{69.4} & \textbf{88.3} & \textbf{97.1} & 36.5 & 47.0 & \textbf{67.5} & 361 & 0.72 \\
\midrule[0.05em]
Base Agent & 71.8 & 91.8 & 90.0 & 90.5 & 59.2 & 0.8 & 8.4 & 21.7 & 11.4 & 60.0 & 75 & 0.39 \\
EvoSkill & 76.2 & 92.5 & 94.7 & 93.0 & 66.0 & 2.5 & 15.2 & 32.5 & 20.8 & 62.5 & 97 & 0.48 \\
Workflow Compiler & 79.6 & 93.6 & 94.8 & 93.7 & 67.4 & 1.7 & 10.3 & 22.6 & 13.4 & 57.5 & 143 & 0.59 \\
\midrule[0.05em]
\textit{Tools Only} & 75.5 & 92.5 & 94.7 & 92.9 & 66.0 & 3.3 & 20.5 & 45.0 & 28.5 & 63.3 & 104 & 0.37 \\
\textit{Knowledge Only} & 74.2 & 92.5 & 94.9 & 93.1 & 66.7 & 5.0 & 36.4 & 68.3 & 45.3 & 64.2 & 93 & 0.51 \\
\textit{Pipeline Only} & 68.7 & 90.7 & 94.2 & 91.7 & 64.6 & 1.7 & 12.0 & 28.3 & 16.2 & 60.8 & 124 & 0.88 \\
\rowcolor{lightgreen}
\textbf{AdaSkill (Ours)} & 80.8 & 94.2 & 94.9 & \textbf{94.1} & 67.4 & 7.5 & 51.3 & \textbf{85.8} & \textbf{60.5} & \textbf{66.7} & 101 & 0.45 \\
\bottomrule[0.1em]
\end{tabular}
\end{adjustbox}

\end{small}
\end{table}

%% file: figures/fig_fig_pareto.tex
\begin{figure}[htbp]
\centering
\begin{tikzpicture}
\begin{groupplot}[
  group style={group size=3 by 2, horizontal sep=9mm, vertical sep=11mm},
  width=0.31\textwidth,
  height=3.8cm,
  grid=major,
  grid style={gray!15},
  tick label style={font=\scriptsize},
  title style={font=\small\bfseries},
  label style={font=\scriptsize},
  xlabel={Cost (\$)},
  ylabel={Performance (\%)},
]

\nextgroupplot[title={T2SQL EX}, xlabel={}, xmin=0.34, xmax=0.77, ymin=57, ymax=71]
  \addplot[black!55, thick] coordinates {(0.39,59.6) (0.45,67.1) (0.72,68.5)};
  \addplot[only marks, mark=triangle*, mark options={fill=cdarkgray}, cdarkgray] coordinates {(0.72,68.5)};
  \addplot[only marks, mark=square*, mark options={fill=cgray}, cgray] coordinates {(0.39,59.6)};
  \addplot[only marks, mark=diamond*, mark options={fill=cblue}, cblue] coordinates {(0.59,62.5)};
  \addplot[only marks, mark=*, mark size=3pt, mark options={fill=cred}, cred] coordinates {(0.45,67.1)};

\nextgroupplot[title={T2SQL NSF}, xlabel={}, ylabel={}, xmin=0.34, xmax=0.77, ymin=47, ymax=80]
  \addplot[black!55, thick] coordinates {(0.39,51.0) (0.45,77.3)};
  \addplot[only marks, mark=triangle*, mark options={fill=cdarkgray}, cdarkgray] coordinates {(0.72,53.9)};
  \addplot[only marks, mark=square*, mark options={fill=cgray}, cgray] coordinates {(0.39,51.0)};
  \addplot[only marks, mark=diamond*, mark options={fill=cblue}, cblue] coordinates {(0.59,53.6)};
  \addplot[only marks, mark=*, mark size=3pt, mark options={fill=cred}, cred] coordinates {(0.45,77.3)};

\nextgroupplot[title={CE MSA}, xlabel={}, ylabel={}, xmin=0.07, xmax=0.65, ymin=76, ymax=99]
  \addplot[black!55, thick] coordinates {(0.11,82.9) (0.21,93.3)};
  \addplot[only marks, mark=triangle*, mark options={fill=cdarkgray}, cdarkgray] coordinates {(0.17,79.2)};
  \addplot[only marks, mark=square*, mark options={fill=cgray}, cgray] coordinates {(0.11,82.9)};
  \addplot[only marks, mark=diamond*, mark options={fill=cblue}, cblue] coordinates {(0.61,90.8)};
  \addplot[only marks, mark=*, mark size=3pt, mark options={fill=cred}, cred] coordinates {(0.21,93.3)};

\nextgroupplot[title={CE $1-$MRE}, xmin=0.07, xmax=0.65, ymin=62, ymax=89]
  \addplot[black!55, thick] coordinates {(0.11,80.0) (0.21,85.6)};
  \addplot[only marks, mark=triangle*, mark options={fill=cdarkgray}, cdarkgray] coordinates {(0.17,65.9)};
  \addplot[only marks, mark=square*, mark options={fill=cgray}, cgray] coordinates {(0.11,80.0)};
  \addplot[only marks, mark=diamond*, mark options={fill=cblue}, cblue] coordinates {(0.61,81.2)};
  \addplot[only marks, mark=*, mark size=3pt, mark options={fill=cred}, cred] coordinates {(0.21,85.6)};

\nextgroupplot[title={CD F1}, ylabel={}, xmin=0.15, xmax=3.2, ymin=62, ymax=84]
  \addplot[black!55, thick] coordinates {(0.30,71.7) (0.42,72.5) (0.47,80.7)};
  \addplot[only marks, mark=triangle*, mark options={fill=cdarkgray}, cdarkgray] coordinates {(3.01,65.3)};
  \addplot[only marks, mark=square*, mark options={fill=cgray}, cgray] coordinates {(0.42,72.5)};
  \addplot[only marks, mark=diamond*, mark options={fill=cblue}, cblue] coordinates {(0.30,71.7)};
  \addplot[only marks, mark=*, mark size=3pt, mark options={fill=cred}, cred] coordinates {(0.47,80.7)};

\nextgroupplot[title={FE AUC}, ylabel={}, xmin=1.5, xmax=14.5, ymin=71.8, ymax=75.0]
  \addplot[black!55, thick] coordinates {(2.78,72.8) (4.26,74.1)};
  \addplot[only marks, mark=triangle*, mark options={fill=cdarkgray}, cdarkgray] coordinates {(13.63,73.3)};
  \addplot[only marks, mark=square*, mark options={fill=cgray}, cgray] coordinates {(5.11,73.2)};
  \addplot[only marks, mark=diamond*, mark options={fill=cblue}, cblue] coordinates {(2.78,72.8)};
  \addplot[only marks, mark=*, mark size=3pt, mark options={fill=cred}, cred] coordinates {(4.26,74.1)};

\end{groupplot}
\node[font=\scriptsize, draw=gray!30, fill=white, inner sep=3pt]
  at ($(group c2r1.north)+(0,1.05cm)$)
  {\textcolor{cdarkgray}{$\blacktriangle$} Source Workflow\quad
   \textcolor{cgray}{$\blacksquare$} Base Agent\quad
   \textcolor{cblue}{$\blacklozenge$} Workflow Compiler\quad
   \textcolor{cred}{$\bullet$} AdaSkill};
\end{tikzpicture}
\caption{Performance and cost operating points for the six task-level aggregates in Table~\ref{tab:main-performance-efficiency}. Solid segments connect non-dominated points. AdaSkill lies on the Pareto frontier in every panel; T2SQL EX retains a trade-off between accuracy and cost relative to the workflow reference.}
\label{fig:pareto}
\end{figure}

%% file: sections/app_case_study.tex
\section{Workflow-to-Skill Conversion Case Study: Causal Estimation}
\label{app:case-study}

We trace the transformation from the original CAIS workflow~\citep{verma2025cais} to the directly deployed AdaSkill, illustrating the component boundary and the metric-conditioned structural decision. The architecture decision is conditioned on CE-MSA.

\subsection*{Adaptive Distillation and Architecture Transformation}

Figure~\ref{fig:distillation-architecture} contrasts the original CAIS four-stage, eight-micro-tool workflow
with the resulting skill-guided tool-using agent.

\begin{figure}[htbp]
\centering
\begin{tikzpicture}[
  mas/.style={
    draw=cred!55, fill=cred!7,
    rounded corners=1.5pt,
    font=\footnotesize\ttfamily,
    text width=3.0cm, align=center,
    minimum height=0.48cm, inner sep=3pt,
  },
  kept/.style={
    draw=cgreen!65, fill=cgreen!10,
    rounded corners=4pt,
    font=\small\bfseries,
    text width=3.6cm, align=center,
    minimum height=0.72cm, inner sep=5pt,
  },
  gone/.style={
    draw=gray!40, fill=gray!5,
    rounded corners=4pt, dashed,
    font=\small, text width=3.6cm, align=center,
    minimum height=0.60cm, inner sep=5pt,
  },
  arr/.style={-{Stealth[length=4pt,width=3pt]}, gray!50, line width=0.6pt},
  bigarr/.style={-{Stealth[length=9pt,width=7pt]}, cMSA, line width=2.5pt},
]

\begin{scope}
  \fill[cred!4, rounded corners=6pt]
    (-1.82, 0.88) rectangle (1.82,-5.00);
  \draw[cred!28, rounded corners=6pt, line width=0.9pt]
    (-1.82, 0.88) rectangle (1.82,-5.00);
\end{scope}

\node[font=\small\bfseries, color=cred!75] at (0, 0.64)
  {Original CAIS workflow};

\draw[cred!18, line width=5pt, line cap=round]
  (0, 0.18) -- (0,-4.55);

\node[mas] (n1) at (0,  0.00) {\texttt{input\_parser}};
\node[mas] (n2) at (0, -0.62) {\texttt{query\_interpreter}};
\node[mas] (n3) at (0, -1.24) {\texttt{dataset\_analyzer}};
\node[mas] (n4) at (0, -1.86) {\texttt{method\_selector}};
\node[mas] (n5) at (0, -2.48) {\texttt{method\_validator}};
\node[mas] (n6) at (0, -3.10) {\texttt{method\_executor}};
\node[mas] (n7) at (0, -3.72) {\texttt{explanation\_gen.}};
\node[mas] (n8) at (0, -4.34) {\texttt{output\_formatter}};
\foreach \a/\b in {n1/n2,n2/n3,n3/n4,n4/n5,n5/n6,n6/n7,n7/n8}
  \draw[arr] (\a) -- (\b);

\node[font=\scriptsize\bfseries, color=cred!65, anchor=west]
  at (-1.78,-4.80) {micro-tools:};
\foreach \i in {1,...,8}{
  \fill[cred!55] ({-1.78 + (\i-1)*0.258},-5.18) circle (4.2pt);
}

\draw[bigarr] (2.15,-2.20) -- (4.05,-2.20)
  node[midway, above, font=\small\bfseries, color=cMSA] {AdaSkill}
  node[midway, below, font=\scriptsize, color=gray!60]
    {$F_{\mathrm{MSA}}\approx 0$};

\begin{scope}
  \fill[cgreen!5, rounded corners=6pt]
    (4.30, 0.88) rectangle (10.45,-5.00);
  \draw[cgreen!28, rounded corners=6pt, line width=0.9pt, dashed]
    (4.30, 0.88) rectangle (10.45,-5.00);
\end{scope}

\node[font=\small\bfseries, color=cgreen!70] at (7.37, 0.64)
  {AdaSkill skill\enspace{\scriptsize\color{cgreen!50}single agent}};

\fill[cgreen!11, rounded corners=4pt]
  (4.52, 0.42) rectangle (10.23,-2.96);
\draw[cgreen!48, rounded corners=4pt, line width=1.0pt]
  (4.52, 0.42) rectangle (10.23,-2.96);
\node[font=\scriptsize\bfseries, color=cgreen!62, anchor=north west]
  at (4.56, 0.42) {\textsc{kept}};

\node[kept] (sk1) at (7.37, 0.14)
  {Callable Tools\enspace{\small 10 functions}};

\node[kept] (sk2) at (7.37,-0.75)
  {Domain Facts\\{\small descriptive only}};

\node[kept] (sk3) at (7.37,-1.63)
  {MSA Policy\\{\small retained at low $F$}};

\draw[-{Stealth[length=3pt,width=2.5pt]}, cgreen!45, line width=0.5pt]
  (7.37,-0.21) -- (7.37,-0.38);
\draw[-{Stealth[length=3pt,width=2.5pt]}, cgreen!45, line width=0.5pt]
  (7.37,-1.10) -- (7.37,-1.27);

\node[font=\scriptsize\itshape, color=cgreen!52, anchor=south east]
  at (10.20,-2.96) {invoked in reasoning-driven order};

\fill[gray!7, rounded corners=4pt]
  (4.52,-3.12) rectangle (10.23,-4.78);
\draw[gray!33, rounded corners=4pt, line width=0.7pt, dashed]
  (4.52,-3.12) rectangle (10.23,-4.78);
\node[gone] (gp) at (7.37,-3.56)
  {Fixed workflow ordering\enspace{\small four stages}};
\node[gone] (gc) at (7.37,-4.42)
  {Shared-state coordination};
\node[font=\scriptsize\bfseries, color=gray!48, fill=gray!7,
      inner sep=1pt, anchor=north west]
  at (4.56,-3.12) {\textsc{discarded}};

\end{tikzpicture}
\caption{\textbf{AdaSkill architecture transformation} for CE ($F_{\mathrm{MSA}} \approx 0$).
\textbf{Left}: the published CAIS workflow comprises four stages and eight shared-state micro-tools; method selection and validation are deterministic, and only a subset of the other tools use LLM assistance.
\textbf{Right}: the adaptive skill retains callable tools and descriptive domain facts, expresses the low-$F$ method-selection policy as a metric-scoped pipeline, and discards the source workflow's global ordering and shared-state orchestration.
The diagram therefore compares workflow structure, not a fixed number of model calls.}
\label{fig:distillation-architecture}
\end{figure}

\paragraph{What happens to the retained pipeline.}
Under the functional boundary in Appendix~\ref{app:operational-protocol}, the original \texttt{method\_selector} is pipeline, not knowledge: it maps observed study properties to an estimator choice. CE-MSA falls on the retain route, so AdaSkill preserves this metric-qualified policy while removing the source workflow's global execution order. The source encodes the policy as hard-coded Python:

\begin{lstlisting}[style=promptstyle, title={Original CAIS: rule\_based\_select\_method (excerpt)}]
if is_rct:
    if instrument_var and instrument_var != treatment:
        return {"selected_method": "instrumental_variable", ...}
    if covariates:
        return {"selected_method": "linear_regression", ...}
    else:
        return {"selected_method": "diff_in_means", ...}
\end{lstlisting}

The adaptive skill converts the same retained pipeline to metric-scoped prose in \texttt{SKILL.md}:

\begin{lstlisting}[style=promptstyle, title={AdaSkill SKILL.md: retained low-$F$ method-selection pipeline}]
### Step 1 - RCT check (do this first)
If description contains: randomized / RCT / randomly assigned -> is_rct = True
- When is_rct = True: matching, PSM, and IPW are FORBIDDEN.
  OLS with covariates is the correct choice.
  Random assignment already balances covariates --
  applying matching on top produces an incorrect selected_method.
- If an instrument for non-compliance exists -> use IV (encouragement design).
\end{lstlisting}

The imperative and prohibitive clauses are deliberately classified as pipeline. Their presence is evidence that the low-$F$ route retained method-selection structure, not that pipeline was hidden inside knowledge. A discard-route module removes the entire block; descriptive facts about randomization and estimator assumptions remain available without prescribing an action.